\documentclass[twoside,11pt]{article}

\usepackage{amsmath}
\usepackage{amssymb}
\usepackage{latexsym}
\usepackage{amsbsy}

\usepackage{epsf}
\usepackage{epsfig}
\usepackage[table]{xcolor}
\usepackage{graphics}
\usepackage{graphicx}

\usepackage{algorithm}
\usepackage{algorithmic}

\usepackage[hidelinks]{hyperref}
\usepackage{xspace}
\usepackage{subfig}
\usepackage{float}
\usepackage{afterpage}

\usepackage{multirow}
\usepackage{booktabs}

\usepackage{mario-def}

\usepackage{jmlr2e}

\renewcommand{\xb}{{x}}

\newcommand{\dsum}{\displaystyle\sum}
\newcommand{\dprod}{\displaystyle\prod}

\newcommand{\etab}{{\boldsymbol\eta}}

\usepackage{mario-def}

\newtheorem{innercustomthm}{PAC-Bound}
\newenvironment{pacbound}[1]
  {\renewcommand\theinnercustomthm{#1}\innercustomthm}
  {\endinnercustomthm}

\newcommand{\nib}{{|\ib|}}

\newcommand{\ex}{{(\xb,y)}}
\newcommand{\KL}{{\rm KL}}
\newcommand{\kl}{{\rm kl}}

\newcommand{\RGQ}{{R_{D'}(G_Q)}}
\newcommand{\RSGQ}{{R_{S}(G_Q)}}
\newcommand{\RDGQ}{{R_{D}(G_Q)}}
\newcommand{\RTGQ}{{R_{T}(G_Q)}}

\newcommand{\RGQS}{{R_{D'}(G_{Q,S})}}
\newcommand{\RSGQS}{{R_{S}(G_{Q,S})}}
\newcommand{\RDGQS}{{R_{D}(G_{Q,S})}}

\newcommand{\BQS}{{B_{Q,S}}}
\newcommand{\RBQS}{{R_{D'}(\BQS)}}

\newcommand{\RDBQS}{{R_{D}(B_{Q,S})}}

\newcommand{\RBQ}{{R_{D'}(B_Q)}}
\newcommand{\RSBQ}{{R_{S}(B_Q)}}
\newcommand{\RDBQ}{{R_{D}(B_Q)}}
\newcommand{\RTBQ}{{R_{T}(B_Q)}}

\newcommand{\esd}[3]{{#1^{_{#2}}_{#3}}}

\newcommand{\eQ}{\esd{e}{D'}{Q}}
\newcommand{\sQ}{\esd{s}{D'}{Q}}
\newcommand{\dQ}{\esd{d}{D'}{Q}}
\newcommand{\eDQ}{\esd{e}{D}{Q}}
\newcommand{\sDQ}{\esd{s}{D}{Q}}
\newcommand{\dDQ}{\esd{d}{D}{Q}}
\newcommand{\eSQ}{\esd{e}{S}{Q}}

\newcommand{\dSQ}{\esd{d}{S}{Q}}
\newcommand{\dTQ}{\esd{d}{\,T}{Q}}

\newcommand{\dQS}{\esd{d}{D'}{Q,S}}
\newcommand{\dDQS}{\esd{d}{D}{Q,S}}
\newcommand{\dSQS}{\esd{d}{S}{Q,S}}

\newcommand{\aQ}{\esd{\alpha}{D'}{Q}}
\newcommand{\aDQ}{\esd{\alpha}{D}{Q}}
\newcommand{\aSQ}{\esd{\alpha}{S}{Q}}
\newcommand{\aij}{\esd{\alpha}{D'}{ij}}
\newcommand{\aSij}{\esd{\alpha}{S}{ij}}
\newcommand{\aDij}{\esd{\alpha}{D}{ij}}
\newcommand{\bQ}{\esd{\beta}{D'}{Q}}
\newcommand{\bDQ}{\esd{\beta}{D}{Q}}
\newcommand{\bSQ}{\esd{\beta}{S}{Q}}
\newcommand{\bij}{\esd{\beta}{D'}{ij}}
\newcommand{\bSij}{\esd{\beta}{S}{ij}}
\newcommand{\bDij}{\esd{\beta}{D}{ij}}

\newcommand{\Cbound}{\mbox{$\Ccal$-bound}\xspace}

\newcommand{\C}{\Ccal}
\newcommand{\CQ}{\Ccal_Q^{D'}}
\newcommand{\CDQ}{\Ccal_Q^{D}}
\newcommand{\CSQ}{\Ccal_Q^{S}}
\newcommand{\CTQ}{\Ccal_Q^{T}}
\newcommand{\CDQprime}{\Ccal_{Q'}^D}
\newcommand{\CSQprime}{\Ccal_{Q'}^S}

\newcommand{\Rdel}{\mathcal{R}_{Q,S}^{\,\delta}}

\newcommand{\Ddel}{\mathcal{D}_{Q,S}^{\,\delta}}
\newcommand{\Adel}{\mathcal{A}_{Q,S}^{\,\delta}}
\newcommand{\AdelEd}{\mathcal{A}_{Q,S}^{\,\delta/2}}
\newcommand{\Adelrestreint}{\widetilde{\mathcal{A}}_{Q,S}^{\,\delta}}

\newcommand{\AdelrestreintEd}{\widehat{\mathcal{A}}_{Q,S}^{\,\delta/2}}
\newcommand{\Rdeld}{\mathcal{R}_{Q,S}^{\delta/2}}
\newcommand{\Edeld}{\mathcal{E}_{Q,S}^{\delta/2}}

\newcommand{\Ddeld}{\mathcal{D}_{Q,S}^{\delta/2}}

\newcommand{\Ddelpd}{\mathcal{D}_{Q,S_{\cal U}}^{\,\delta/2}}

\newcommand{\Mb}{{\mathbf M}}

\newcommand{\loss}{\mathcal{L}}
\newcommand{\linloss}{\mathcal{L}_\ell}
\newcommand{\zoloss}{\mathcal{L}_{_{01}}}

\newcommand{\Eaggloss}[2]{\mathbb{E}_{#2}^{\loss_{#1}}}
\newcommand{\Eloss}[1]{\mathbb{E}_{#1}^{\loss}}
\newcommand{\Elinloss}[1]{\mathbb{E}_{#1}^{\linloss}}
\newcommand{\Ezoloss}[1]{\mathbb{E}_{#1}^{\zoloss}}
\newcommand{\ElossS}{\Eloss{S}}
\newcommand{\ElossD}{\Eloss{D}}
\newcommand{\ElossDprim}{\Eloss{D'}}
\newcommand{\ElinlossDprim}{\Elinloss{D'}}

\newcommand{\ElinlossD}{\Elinloss{D}}
\newcommand{\ElinlossS}{\Elinloss{S}}

\newcommand{\ibc}{{\ib^c}}
\newcommand{\ElinlossSib}{\Elinloss{S_\ib}}
\newcommand{\ElinlossSibc}{\Elinloss{S_\ibc}}

\newcommand{\fpaired}{{f_{ij}}}

\newcommand{\tuple}[1]{\left\langle\,#1\,\right\rangle}

\newcommand{\cov}[1]{
    \underset{#1}{{\rm\bf Cov}}\
}

\newcommand{\MQ}[1]{{M_Q^{\mbox{\tiny$#1$}}}}
\newcommand{\MQS}[1]{{M_{Q,S}^{\mbox{\tiny$#1$}}}}
\newcommand{\MQprime}[1]{{M_{Q'}^{\mbox{\tiny$#1$}}}}

\newcommand{\momentone}{{\mu_{\hspace{0.2mm}1}}}
\newcommand{\momenttwo}{{\mu_{\hspace{0.2mm}2}}}
\newcommand{\variance}{{\text{\rm Var}}}
\newcommand{\covariance}[1]{{\text{\rm Cov}^{_{#1}}}}

\newcommand{\vMQ}{{\variance(\MQ{D'})}}

\newcommand{\vTMQ}{{\variance(\MQ{T})}}

\newcommand{\Yover}{{\overline{\Ycal}}}

\begin{document}

\title{Risk Bounds for the Majority Vote: From a PAC-Bayesian Analysis to a Learning Algorithm}

\jmlrheading{16}{2015}{787-860}{5/13; Revised 9/14}{4/15}{Pascal Germain, Fran\c{c}ois Laviolette, Alexandre Lacasse, Mario Marchand and Jean-Francis Roy}
\firstpageno{787}

\ShortHeadings{Risk Bounds for the Majority Vote}{Germain, Lacasse, Laviolette, Marchand and Roy} 

\author{\name Pascal Germain   \email Pascal.Germain@ift.ulaval.ca \\
       \name Alexandre Lacasse    \email Alexandre.Lacasse@ift.ulaval.ca \\
       \name Fran\c{c}ois Laviolette    \email Francois.Laviolette@ift.ulaval.ca \\
       \name Mario Marchand    \email Mario.Marchand@ift.ulaval.ca \\
       \name Jean-Francis Roy    \email Jean-Francis.Roy@ift.ulaval.ca \\
       \addr D{\'e}partement d'informatique et de g\'enie logiciel \\
                Universit\'e Laval  \\
                Qu\'ebec, Canada,
                G1V 0A6 \\[.6mm]
                \hfill{\rm $^\star$\,All authors contributed equally to this work.}
}

\editor{Koby Crammer}

\maketitle

\begin{abstract}We propose an extensive analysis of the behavior of majority votes in binary classification. In particular, we introduce a risk bound for majority votes, called the \Cbound, that takes into account the average quality of the voters and their average disagreement. We also propose an extensive PAC-Bayesian analysis that shows how the \Cbound can be estimated from various observations contained in the training data. The analysis intends to be self-contained and can be used as introductory material to PAC-Bayesian statistical learning theory.  It starts from a general PAC-Bayesian perspective and ends with uncommon PAC-Bayesian bounds. Some of these bounds contain no Kullback-Leibler divergence and others allow kernel functions to be used as voters (via the sample compression setting).
Finally, out of the analysis, we propose the MinCq learning algorithm that basically minimizes the \Cbound. MinCq reduces to a simple quadratic program. Aside from being theoretically grounded, MinCq achieves state-of-the-art performance, as shown in our extensive empirical comparison with both AdaBoost and the Support Vector Machine.
\end{abstract}

\section{Previous Work and Implementation} 

This paper can be considered as an extended version of \cite{nips06-mv} and \cite{lmr-11}, and also contains ideas from~\cite{lm-05,lm-07} and \cite{gllm-09,gllms-11}. We unify this previous work, revise the mathematical approach, add new results and extend empirical experiments. 

The source code to compute the various PAC-Bayesian bounds presented in this paper and the implementation of the MinCq learning algorithm is available at:

\noindent \url{http://graal.ift.ulaval.ca/majorityvote/} 

\section{Introduction}

In binary classification, many state-of-the-art algorithms output prediction functions that can be seen as a majority vote of ``simple'' classifiers. Firstly, ensemble methods such as Bagging \citep{b-96}, Boosting \citep{schapire99} and Random Forests \citep{b-01} are well-known examples
of  learning algorithms that output majority votes. Secondly, majority votes are also central in the Bayesian approach (see \citealp{gelman2004bayesian}, for an introductory text);   in this setting, the majority vote is generally called the \emph{Bayes Classifier}.  Thirdly, it is interesting to point out that classifiers produced by kernel methods, such as the Support Vector Machine (SVM) \citep{DBLP:journals/ml/CortesV95}, can also be viewed as majority votes. Indeed, to classify an example~$x$, the SVM classifier computes
\begin{equation} \label{eq:svm_intro}
\sgn \bigg(\mbox{\footnotesize$\dsum_{i=1}^{|S|}$} \alpha_i\, y_i\, k(\xb_i,\xb)\bigg),
\end{equation}
where $k(\cdot,\cdot)$ is a kernel function, and the input-output pairs $(x_i,y_i)$ represent the examples from the training set $S$. Thus, one can interpret each \,$y_i\,k(\xb_i,\cdot)$\, as a voter that chooses (with confidence level $|k(\xb_i,\xb)|$) between two alternatives (``positive'' or ``negative''), and~$\alpha_i$ as the respective weight of this voter in the majority vote. 
Then, if the total confidence-multiplied weight of each voter that votes positive is greater than the total confidence-multiplied weight of each voter that votes negative, the classifier outputs a $+1$ label (and a $-1$ label in the opposite case). 
Similarly, each \emph{neuron} of the last layer of an artificial neural network can be interpreted as a majority vote, since it outputs a real value given by $K(\sum_i w_i g_i(\xb))$ for some \emph{activation function} $K$.\footnote{In this case, each voter $g_i$ has incoming weights which are also learned (often by back propagation of errors) together with the weights~$w_i$. The analysis presented in this paper considers fixed voters. Thus, the PAC-Bayesian theory for artificial neural networks remains to be done. Note however that the recent work by~\cite{mcallester-13} provides a first step in that direction.}

\smallskip
 In practice, it is well known that the classifier output by each of these learning algorithms performs much better than any of its voters individually. Indeed, voting can dramatically improve performance when the ``community'' of classifiers tends to
compensate for individual errors. In particular, this phenomenon explains the success of Boosting algorithms \citep[\eg,][]{sfbl-98}.
The first aim of this paper is to explore how bounds on the generalized risk of the majority vote are not only able to theoretically justify learning algorithms but also to detect when the voted combination provably outperforms the
average of its voters.
We  expect that this study of the behavior of a majority vote should improve the understanding of existing learning algorithms and even lead to new ones. We indeed present a learning algorithm based on these ideas at the end of the paper.

 \smallskip
The PAC-Bayesian theory is  a well-suited approach to analyze majority votes. Initiated by~\citet{m-99}, this theory aims to provide
 Probably Approximately Correct  guarantees (PAC guarantees) to ``Bayesian-like'' learning algorithms. Within this
 approach, one considers a {\em prior\footnote{Priors have been used
 for many years in statistics. The priors in this paper have only
 indirect links with the \emph{Bayesian priors}. We nevertheless
 use this language, since it comes from previous work.}
 distribution} $P$ over a space of classifiers that characterizes
 its prior belief about good classifiers (before the observation of
 the data) and a {\em posterior distribution}~$Q$ (over the same
 space of classifiers) that takes into account the additional
information provided by the training data. The classical PAC-Bayesian approach indirectly bounds the risk of a $Q$-weighted majority vote by bounding the risk of an associate (stochastic) classifier, called the {\em Gibbs classifier}. A remarkable result,
 known as the ``PAC-Bayesian Theorem'', provides a risk bound for the ``true'' risk of the Gibbs classifier, 
 by considering the empirical risk of this Gibbs classifier on the training data and the Kullback-Leibler divergence between a posterior distribution $Q$ and a prior distribution $P$. 
 It is well known \citep{ls-03,m-03b,gllm-09} that the risk of the (deterministic) majority vote classifier is upper-bounded by twice the risk of the associated (stochastic) Gibbs classifier. Unfortunately,  and especially if the involved voters are weak, this indirect bound on the majority vote classifier is far from being tight, even if the PAC-Bayesian bound itself generally gives  a tight bound on the risk of the Gibbs classifier. 
In practice, as stated before,  the ``community'' of classifiers can act in such a way as to compensate for individual errors. 
When such compensation occurs, the risk of the majority vote is then much lower than the Gibbs risk itself and, a fortiori, much lower than twice the Gibbs risk. By limiting the analysis to Gibbs risk only, the commonly used PAC-Bayesian framework is unable to evaluate whether or not
 this compensation occurs. Consequently, this framework cannot help in producing highly accurate voted combinations of classifiers when these classifiers are individually weak.
 
 \smallskip
In this paper, we tackle this problem by studying the margin of the majority vote as a random variable. The first and second moments of this random variable are respectively linked with the risk of the Gibbs classifier and the expected disagreement between the voters of the majority vote. As we will show, the well-known factor of two used to bound the risk of the majority vote is recovered by applying Markov's inequality to the first moment of the margin. Based on this observation, we show that a tighter bound, that we call the \Cbound, is obtained by considering the first two moments of the margin, together with Chebyshev's inequality.

 Section~\ref{section:Majority_Vote_Bounds} presents, in a more detailed way, the work on the \Cbound originally presented in~\cite{nips06-mv}. We then present both theoretical and empirical studies that show that the \Cbound is an accurate indicator of the risk of the majority vote. We also show that the \Cbound can be
smaller than the risk of the Gibbs classifier and can even be arbitrarily close to zero even if the risk of the Gibbs classifier is close to~1/2. This indicates that the \Cbound can effectively capture the compensation of the individual errors made by the voters.

 \smallskip
 We then develop PAC-Bayesian guarantees on the \Cbound in order to obtain an upper bound on the risk of the majority vote based on empirical observations. Section~\ref{section:PAC-Bayes} presents a general approach of the PAC-Bayesian theory by which we recover the most commonly used forms of the bounds of~\cite{m-99,m-03} and~\cite{ls-01-techreport,s-02, l-05}. Thereafter, we extend the theory to obtain upper bounds on the \Cbound in two different ways. 
 The first method is to separately bound the risk of the Gibbs classifier and the expected disagreement---which are the two fundamental ingredients that are present in the \Cbound. Since the expected disagreement does not rely on labels, this strategy is well-suited for the semi-supervised learning framework.
 The second method directly bounds the \Cbound and empirically improves the achievable bounds in the supervised learning framework. 
 
 \smallskip
 Sections~\ref{section:Further-PAC_Bayes} and~\ref{sec:sample-compression} bring together relatively new PAC-Bayesian ideas that allow us, for one part, to derive a PAC-Bayesian bound that does not rely on the Kullback-Leibler divergence between the prior and posterior distributions \citep[as in][]{c-07,gllms-11,lmr-11} and, for the other part, to extend the bound to the case where the voters are defined using elements of the training data, \eg, voters defined by kernel functions $y_ik(\xb_i,\cdot)$. This second approach is based on the sample compression theory \citep{fw-95,lm-07,gllms-11}. In  PAC-Bayesian theory, the sample compression approach is \emph{a priori\/} problematic, since a PAC-Bayesian bound makes use of a prior distribution on the set of all voters  that has to be defined before observing the data. If the voters themselves are defined using a part of the data, there is an apparent contradiction that has to be overcome.

 \smallskip
 Based on the foregoing, a learning algorithm, that we call MinCq, is presented in Section~\ref{sect:mincq}. The algorithm basically minimizes the \Cbound, but in a particular way that is, inter alia, justified by the PAC-Bayesian analysis of Sections~\ref{section:Further-PAC_Bayes} and~\ref{sec:sample-compression}. This algorithm was originally presented in~\cite{lmr-11}. Given a set of voters (either classifiers or kernel functions), MinCq builds a majority vote classifier by finding the posterior distribution~$Q$ on the set of voters that minimizes the \Cbound. Hence, MinCq takes into account not only the overall quality of the voters, but also their respective disagreements. In this way, MinCq builds a  ``community'' of voters that can compensate for their individual errors. Even though the \Cbound consists of a relatively complex quotient, the MinCq learning algorithm reduces to a simple quadratic program. 
 Moreover, extensive empirical experiments confirm that MinCq is very competitive when compared with AdaBoost \citep{schapire99} and the Support Vector Machine \citep{DBLP:journals/ml/CortesV95}. 

 \smallskip
In Section~\ref{section:Finis-tu, never ending paper?}, we conclude by pointing out recent work that uses the PAC-Bayesian theory to tackle more sophisticated machine learning problems.

\section{Basic Definitions}

We consider classification problems where the input space
$\Xcal$ is an arbitrary set and the output
space is a discrete set denoted $\Ycal$. An \emph{example} $(\xb,y)$ is an
input-output pair where $\xb\in\Xcal$ and $y\in\Ycal$.
A \emph{voter} is a  function $\Xcal \rightarrow \Yover$ for some output space $\Yover$ related to $\Ycal$.
Unless otherwise specified, we consider the binary classification problem where $\Ycal = \{-1,1\}$ and then we either consider $\Yover$ as $\Ycal$ itself, or its convex hull $[-1,+1]$.
In this paper, we also use the following convention: $f$ denotes a real-valued voter (\ie, $\Yover = [-1,1]$), and $h$ denotes a binary-valued voter (\ie, $\Yover = \{-1,1\}$).
Note that this notion of voters is quite general, since any uniformly bounded real-valued set of functions can be viewed as a set of voters when properly normalized.

\smallskip
We consider learning algorithms that construct majority votes based on a (finite) set $\Hcal$ of voters.
Given any $\xb\in\Xcal$, the output $B_Q(\xb)$ of a $Q$-weighted majority vote classifier $B_Q$ (sometimes called the \emph{Bayes classifier})
is given by
\begin{equation}\label{eq:bayes}
B_Q(\xb)\ \eqdef\ \sgn\LB \esp{f\sim Q} f(\xb) \RB,
\end{equation}
where $\sgn(a)=1$ if $a>0$, \, $\sgn(a)=-1$ if $a<0$, and $\sgn(0)=0$. 

Thus, in case of a tie in the majority vote -- \ie, $\Eb_{f\sim Q}f(\xb)\!=\!0$ --, we consider that the majority vote classifier abstains -- \ie, $B_Q(x)=0$. There are other possible ways to handle this particular case. In this paper, we choose to define $\sgn(0)\!=\!0$ because it simplifies the forthcoming analysis.

We adopt the PAC setting where each example
$(\xb, y)$ is drawn i.i.d.\ according to a fixed, but unknown, probability
distribution $D$ on $\Xcal\!\times\!\Ycal$. The \emph{training set} of $m$ examples is denoted by   $S=\tuple{(\xb_1,y_1), \dots,(\xb_{m},y_{m})}\sim D^m$. Throughout the paper, $D'$ generically represents either the true (and unknown) distribution $D$, or its empirical counterpart $\mathrm{U}_S$ (\ie, the uniform distribution over the training set $S$). Moreover, for notational simplicity, we often replace $\mathrm{U}_S$ by $S$.  

In order to quantify the accuracy of a voter, we use a \emph{loss function} \mbox{$\loss\,:\, \Yover\!\times\!\Ycal\rightarrow [0,1]\,.$}
The PAC-Bayesian theory traditionally considers majority votes of binary voters 
of the form $h:\Xcal\rightarrow\{-1,1\}$,
and the \emph{zero-one loss} 
$
\zoloss\big(  h(\xb), y \big) \eqdef  I \big( h(\xb) \neq y \big)\,,
$
where $I(a)=1$ if predicate $a$ is true and $0$ otherwise. 

The extension of the zero-one loss to  real-valued voters 
(of the form $f:\Xcal\rightarrow[-1,1]$) is given by the following definition.
\begin{definition}\rm
\label{conv:Risque_de_B_Q}\rm
In the (more general) case where voters are functions $f:\Xcal\rightarrow [-1,1]$, 
the zero-one loss $\zoloss$ is defined by
\begin{eqnarray*}
\quad\quad \zoloss \big(  f(\xb), y \big) \ \eqdef \ I \big( y \cdot f(\xb) \leq 0\big).
\end{eqnarray*}
\end{definition}
Hence, 
a voter abstention -- \ie, when $f(x)$ outputs exactly~$0$ -- results in a loss of $1$. 
Clearly, other choices are possible for this particular case.\footnote{As an example, when $f(x)$ outputs $0$, the loss may be $1/2$. However, we choose for this unlikely event the worst loss value -- \ie, $\zoloss (0, y)=1$ -- because it simplifies the majority vote analysis.}

\medskip

 \noindent  In this paper, we also consider the \emph{linear loss}~$\linloss$ defined as follows.
 
 \begin{definition}\rm
 \label{def:linearloss}\rm
Given a voter $f:\Xcal\rightarrow [-1,1]$, the linear loss $\linloss$ is defined by
\begin{eqnarray*} %
  \linloss\big( f(\xb), y \big) \ \eqdef \ \dfrac{1}{2} \Big( 1- y \cdot f(\xb) \Big).
\end{eqnarray*}
 \end{definition}
Note that the linear loss is equal to the zero-one loss  when the output space is binary.  That is, for any $(h(\xb), y) \in \{-1,1\}^2$, we always have   
\begin{eqnarray} \label{eq:losses_equivalents_sur_binaire}
\linloss\big( h(\xb), y \big) & = & \zoloss\big( h(\xb), y \big)\,,
\end{eqnarray}  
because 
$\linloss\big( h(\xb), y \big) = 1$ if $h(\xb)\ne y$, and 
$\linloss\big( h(\xb), y \big) = 0$ if $h(\xb)= y$.
Hence, we generalize all definitions implying classifiers to voters using the equality of Equation~\eqref{eq:losses_equivalents_sur_binaire} as an inspiration.  
Figure~\ref{fig:losses} illustrates the difference between the zero-one loss and the linear loss for real-valued voters. Remember that in the case $y\,f(x)=0$ , the loss is $1$ (see Definition~\ref{conv:Risque_de_B_Q}).

\begin{figure}[t]
\centering{
\begingroup
  \makeatletter
  \providecommand\color[2][]{%
    \GenericError{(gnuplot) \space\space\space\@spaces}{%
      Package color not loaded in conjunction with
      terminal option `colourtext'%
    }{See the gnuplot documentation for explanation.%
    }{Either use 'blacktext' in gnuplot or load the package
      color.sty in LaTeX.}%
    \renewcommand\color[2][]{}%
  }%
  \providecommand\includegraphics[2][]{%
    \GenericError{(gnuplot) \space\space\space\@spaces}{%
      Package graphicx or graphics not loaded%
    }{See the gnuplot documentation for explanation.%
    }{The gnuplot epslatex terminal needs graphicx.sty or graphics.sty.}%
    \renewcommand\includegraphics[2][]{}%
  }%
  \providecommand\rotatebox[2]{#2}%
  \@ifundefined{ifGPcolor}{%
    \newif\ifGPcolor
    \GPcolorfalse
  }{}%
  \@ifundefined{ifGPblacktext}{%
    \newif\ifGPblacktext
    \GPblacktexttrue
  }{}%
  \let\gplgaddtomacro\g@addto@macro
  \gdef\gplbacktext{}%
  \gdef\gplfronttext{}%
  \makeatother
  \ifGPblacktext
    \def\colorrgb#1{}%
    \def\colorgray#1{}%
  \else
    \ifGPcolor
      \def\colorrgb#1{\color[rgb]{#1}}%
      \def\colorgray#1{\color[gray]{#1}}%
      \expandafter\def\csname LTw\endcsname{\color{white}}%
      \expandafter\def\csname LTb\endcsname{\color{black}}%
      \expandafter\def\csname LTa\endcsname{\color{black}}%
      \expandafter\def\csname LT0\endcsname{\color[rgb]{1,0,0}}%
      \expandafter\def\csname LT1\endcsname{\color[rgb]{0,1,0}}%
      \expandafter\def\csname LT2\endcsname{\color[rgb]{0,0,1}}%
      \expandafter\def\csname LT3\endcsname{\color[rgb]{1,0,1}}%
      \expandafter\def\csname LT4\endcsname{\color[rgb]{0,1,1}}%
      \expandafter\def\csname LT5\endcsname{\color[rgb]{1,1,0}}%
      \expandafter\def\csname LT6\endcsname{\color[rgb]{0,0,0}}%
      \expandafter\def\csname LT7\endcsname{\color[rgb]{1,0.3,0}}%
      \expandafter\def\csname LT8\endcsname{\color[rgb]{0.5,0.5,0.5}}%
    \else
      \def\colorrgb#1{\color{black}}%
      \def\colorgray#1{\color[gray]{#1}}%
      \expandafter\def\csname LTw\endcsname{\color{white}}%
      \expandafter\def\csname LTb\endcsname{\color{black}}%
      \expandafter\def\csname LTa\endcsname{\color{black}}%
      \expandafter\def\csname LT0\endcsname{\color{black}}%
      \expandafter\def\csname LT1\endcsname{\color{black}}%
      \expandafter\def\csname LT2\endcsname{\color{black}}%
      \expandafter\def\csname LT3\endcsname{\color{black}}%
      \expandafter\def\csname LT4\endcsname{\color{black}}%
      \expandafter\def\csname LT5\endcsname{\color{black}}%
      \expandafter\def\csname LT6\endcsname{\color{black}}%
      \expandafter\def\csname LT7\endcsname{\color{black}}%
      \expandafter\def\csname LT8\endcsname{\color{black}}%
    \fi
  \fi
  \setlength{\unitlength}{0.0500bp}%
  \begin{picture}(7200.00,2520.00)%
    \gplgaddtomacro\gplbacktext{%
      \csname LTb\endcsname%
      \put(726,465){\makebox(0,0)[r]{\strut{} 0}}%
      \put(726,1279){\makebox(0,0)[r]{\strut{} 0.5}}%
      \put(726,2093){\makebox(0,0)[r]{\strut{} 1}}%
      \put(883,220){\makebox(0,0){\strut{}-1}}%
      \put(2139,220){\makebox(0,0){\strut{}-0.5}}%
      \put(3394,220){\makebox(0,0){\strut{} 0}}%
      \put(4650,220){\makebox(0,0){\strut{} 0.5}}%
      \put(5906,220){\makebox(0,0){\strut{} 1}}%
      \put(6408,220){\makebox(0,0){\strut{}$\ \ y f(\xb)$}}%
    }%
    \gplgaddtomacro\gplfronttext{%
      \csname LTb\endcsname%
      \put(4454,1928){\makebox(0,0)[l]{\strut{}$\zoloss\big(f(\xb),y\big)$}}%
      \csname LTb\endcsname%
      \put(4454,1598){\makebox(0,0)[l]{\strut{}$\linloss\big(f(\xb),y)$}}%
    }%
    \gplbacktext
    \put(0,0){\includegraphics{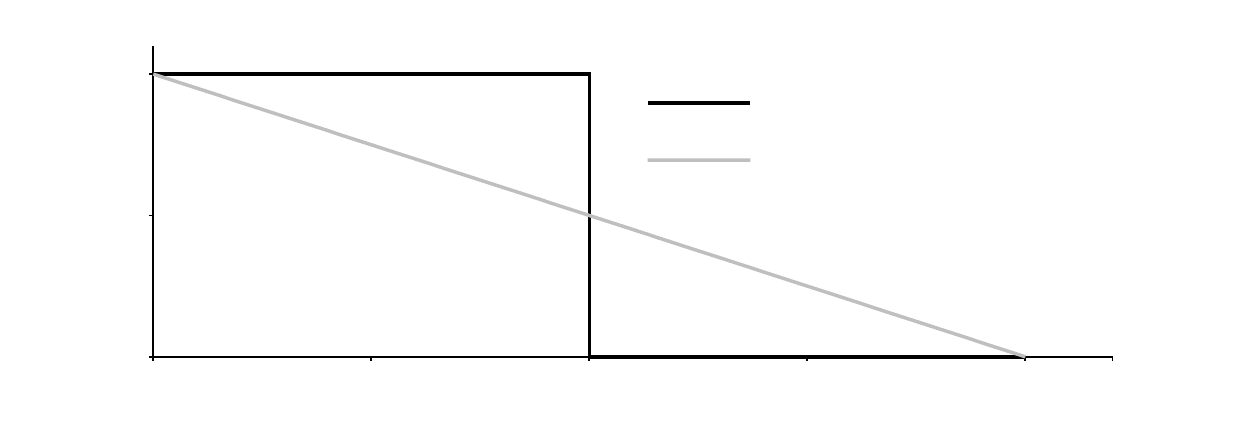}}%
    \gplfronttext
  \end{picture}%
\endgroup
 }
 \caption{The zero-one loss $\zoloss$ and the linear loss $\linloss$ as a function of $yf(\xb)$.}
 \label{fig:losses}
\end{figure}

\begin{definition}\rm \label{def:expected_loss}
Given a loss function~$\loss$ and a voter $f$, the \emph{expected loss} $\ElossDprim( f )$ of $f$ relative to distribution $D'$ is defined as 
\begin{equation*}
\ElossDprim( f ) \ \eqdef  \,\esp{\ex\sim D'} \loss \big( f(\xb), y \big)\,.
\end{equation*}
In particular, the \emph{empirical expected loss} on a training set $S$ is given by
\begin{equation*}
\ElossS( f )\ = \ \frac{1}{m} \sum_{i=1}^m \loss \big( f(\xb_i), y_i \big)\,.
\end{equation*}
\end{definition}
We therefore define the risk of the majority vote $R_{D'}(B_Q)$ as follows.
\begin{definition}\rm \label{def:bayesrisk}
 For any probability distribution $Q$ on a set of voters, the \emph{Bayes risk} $R_{D'}(B_Q)$, also called \emph{risk of the majority vote}, is defined as the expected zero-one loss of the majority vote classifier $B_Q$ relative to $D'$. Hence,
\begin{equation*}
 R_{D'}(B_Q)\ \eqdef\  
 \Ezoloss{D'}(B_Q) \ =\ 
 \esp{\ex\sim D'} I\Big( B_Q(\xb) \neq y \Big) \ =\ 
  \esp{\ex\sim D'} I\Big( \esp{f\sim Q} y \cdot f(\xb) \leq 0 \Big) \, .
\end{equation*}
\end{definition}
 Remember from the definition of $B_Q$ (Equation~\ref{eq:bayes}) that the majority vote classifier abstains in the case of a tie on an example~$\ex$. Therefore, the above Definition~\ref{def:bayesrisk} implies that the Bayes risk is 1 in this case, as $R_{\langle(\xb,y)\rangle}(B_Q) \!=\! \zoloss(0,y)\!=\!1$. In practice, a tie in the vote is a rare event, especially if there are many voters.

The output of the deterministic majority vote classifier
$B_Q$ is closely related to the output of a stochastic
classifier called the \emph{Gibbs classifier}. 
To classify an input
example~$\xb$, the Gibbs classifier $G_Q$ randomly chooses a
voter $f$ according to $Q$ and returns $f(\xb)$. Note the stochasticity of the Gibbs classifier: it can output different values when given the same input $x$ twice. We will see later how the link between $B_Q$ and $G_Q$ is used in the PAC-Bayesian theory.

In the case of binary voters, the Gibbs risk %
corresponds to the probability that $G_Q$ misclassifies an example of distribution $D'$. Hence,
\begin{equation*} %
R_{D'}(G_Q) 
\ = \, \prob{\substack{(\xb,y)\sim D'\\ h\sim Q}} \big(h(\xb)\ne y\big)
\ = \, \esp{h\sim Q} \Ezoloss{D'}(h)
\ = \, \esp{(\xb,y)\sim D'}\esp{h\sim Q}I\big(h(\xb)\ne y\big) 
\,.
\end{equation*}
In order to handle real-valued voters, we generalize the Gibbs risk as follows. 
\begin{definition}\rm \label{def:gibbsrisk}
   For any probability distribution $Q$ on a set of voters, the \emph{Gibbs risk} $R_{D'}(G_Q)$ is defined as the expected linear loss of the Gibbs classifier $G_Q$ relative to $D'$. Hence,
\begin{equation*} %
R_{D'}(G_Q) \ \eqdef \, 
 \esp{f\sim Q} \ElinlossDprim(f)
\ = \  \frac{1}{2} \left(1-\esp{(\xb,y)\sim D'}\esp{f\sim Q}y \cdot f(\xb) \right)\,.
\end{equation*}
\end{definition}

\begin{remark} \label{remark:2gibbs} \rm
It is well known in the PAC-Bayesian literature \citep[\eg,][]{ls-03,m-03b,gllm-09} that the Bayes risk $\RBQ$ is  bounded by twice the Gibbs risk $\RGQ$. This statement extends to our more general definition  of the Gibbs risk (Definition~\ref{def:gibbsrisk}).
\end{remark}
\begin{proof}
Let $(\xb,y)\in\Xcal\times\{-1,1\}$ be any example. We claim that
\begin{equation}\label{eq:risque_sur_un_exemple2}
R_{\langle(\xb,y)\rangle}(B_Q)\ \le\ 2\,R_{\langle(\xb,y)\rangle}(G_Q)\,.
\end{equation}
Notice that $R_{\langle(\xb,y)\rangle}(B_Q)$ is either $0$ or $1$ depending of the fact that $B_Q$ errs or not on $(\xb,y)$. In the case where $R_{\langle(\xb,y)\rangle}(B_Q) =0$, Equation~\eqref{eq:risque_sur_un_exemple2} is trivially true. If $R_{\langle(\xb,y)\rangle}(B_Q) =1$, we know by the last equality of Definition~\ref{def:bayesrisk} that $\esp{f\sim Q}y\cdot f(\xb) \leq 0$. 
Therefore, Definition~\ref{def:gibbsrisk} gives
\begin{equation*}
2\cdot R_{\langle(\xb,y)\rangle}(G_Q) \ = \  2\cdot \frac{1}{2} \left(1-\esp{f\sim Q}y \cdot f(\xb) \right) \ \geq \  1 = R_{\langle(\xb,y)\rangle}(B_Q) \,,
\end{equation*}
which proves the claim.
  
Now, by taking the expectation according to $(\xb,y)\sim D'$ on each side of Equation~\eqref{eq:risque_sur_un_exemple2}, we obtain
\begin{equation*}
R_{D'}(B_Q) \,= \esp{(\xb,y)\sim D'} R_{\langle(\xb,y)\rangle}(B_Q)\ \le\ \esp{(\xb,y)\sim D'}  2\,R_{\langle(\xb,y)\rangle}(G_Q)\,= 2\,R_{D'}(G_Q) \,,
\end{equation*}
as wanted.
\end{proof}

Thus, PAC-Bayesian bounds on the risk of the majority vote are usually bounds on the Gibbs risk, multiplied by a factor of two.  Even if this type of bound can be tight in some situations, the factor two can also be misleading. \citet{ls-03} have shown that under some circumstances, the factor of two can be reduced to $(1+\epsilon)$. Nevertheless, distributions $Q$ on voters giving $\RGQ \gg \RBQ$ are common.
The extreme case happens when the expected linear loss on each example is just below one half -- \ie, for all~$\ex$, $\Eb_{f\sim Q}\ y \, f(\xb) = \frac{1}{2}{-}\epsilon$ --, leading to a perfect majority vote classifier but an almost inaccurate Gibbs classifier.
Indeed, we have $\RGQ=\frac{1}{2}{-}\epsilon$ and $\RBQ=0$. Therefore, in this circumstance, the bound $\RBQ\leq 1{-}2\epsilon$, given by Remark~\ref{remark:2gibbs}, fails to represent the perfect accuracy of the majority vote. This problem is due to the fact that the Gibbs risk only considers the loss of the average output of the population of voters.
Hence, the bound of Remark~\ref{remark:2gibbs} states that the majority vote is weak whenever every individual voter is weak. 
The bound cannot capture the fact that it might happen that the ``community'' of voters 
compensates for individual errors. To overcome this lacuna, we need a bound that compares the output of voters between them, not only the average quality of each voter taken individually.

We can compare the output of binary voters by considering the probability of disagreement between them:
\begin{eqnarray*}
 \prob{\substack{\xb\sim D_\Xcal'\\ h_1, h_2\sim Q}} \Big( h_1(\xb) \neq h_2(\xb) \Big) 
 & = &
 \esp{\xb\sim D_\Xcal'} \esp{h_1\sim Q} \esp{h_2\sim Q}  I \Big( h_1(\xb) \neq h_2(\xb) \Big)\\[-2mm]
  & =&
 \esp{\xb\sim D_\Xcal'} \esp{h_1\sim Q} \esp{h_2\sim Q}  I \Big( h_1(\xb)\cdot h_2(\xb)\ne 1 \Big) \\[1mm]
  & =&
  \esp{\xb\sim D_\Xcal'} \esp{h_1\sim Q} \esp{h_2\sim Q}  \zoloss \big( \,h_1(\xb)\!\cdot\! h_2(\xb)\,, \, 1\, \big) \, ,
\end{eqnarray*}
where $D_\Xcal'$ denotes the marginal on $\Xcal$ of distribution $D'$. 
Definition~\ref{def:disagreement} extends this notion of disagreement to real-valued voters.
\begin{definition}\rm
\label{def:disagreement}
 For any probability distribution $Q$ on a set of voters, the \emph{expected disagreement}~$\dQ$ relative to $D'$ is defined as
\begin{eqnarray}
\nonumber
  \dQ &\eqdef& \esp{\xb\sim D_\Xcal'} \esp{f_1\sim Q} \esp{f_2\sim Q}  \linloss \,\big(\, f_1(\xb)\!\cdot\! f_2(\xb)\,,\, 1 \,\big)\\
  \nonumber
  &=& \frac{1}{2} \left( 1- \esp{\xb\sim D_\Xcal'} \esp{f_1\sim Q} \esp{f_2\sim Q}  1\cdot f_1(\xb) \cdot f_2(\xb) \right) \\
   \label{eq:desagrement_espere}
   \nonumber
   &=& \frac{1}{2} \bigg( 1- \esp{\xb\sim D_\Xcal'} \left[ \esp{f\sim Q}   f(\xb) \right]^2 \bigg)\,.
\end{eqnarray}
\end{definition}
Notice that the value of $\dQ$ does not depend on the labels $y$ of the examples $\ex\sim D'$.  Therefore, we can estimate the expected disagreement with unlabeled data.

\section{Bounds on the Risk of the Majority Vote}
\label{section:Majority_Vote_Bounds}

The aim of this section is to introduce the $\Cbound$, which upper-bounds the risk of the majority vote (Definition~\ref{def:bayesrisk}) based on the Gibbs risk (Definition~\ref{def:gibbsrisk}) and the expected disagreement (Definition~\ref{def:disagreement}). We start by studying the margin of a majority vote as a random variable (Section~\ref{section:margin_and_moments}). From the first moment of the margin, we easily recover the well-known bound of twice the Gibbs risk presented by Remark~\ref{remark:2gibbs} (Section~\ref{section:2RG_rediscover}). We therefore suggest extending this analysis to the second moment of the margin to obtain the \Cbound (Section~\ref{section:Cbound_thm}). Finally, we present some statistical properties of the \Cbound (Section~\ref{section:Cbound_statistical_analysis}) and an empirical study of its predictive power (Section~\ref{section:Cbound_empirical_study}).  

\subsection{The Margin of the Majority Vote and its Moments}
\label{section:margin_and_moments}

The bounds on the risk of a majority vote classifier proposed in this section result from the study of the weighted margin of the majority vote as a random variable.

\begin{definition}\rm \label{def:margin}
Let $\MQ{D'}$ be the random variable that, given any example $\ex$ drawn according to $D'$, outputs the \emph{margin} of the majority vote $B_Q$ on that example, which is
\begin{equation*}
 M_Q(\xb,y) \ \eqdef \ \esp{f\sim Q} y \cdot f(\xb)\,.
\end{equation*}
\end{definition}

\noindent
From Definitions~\ref{def:bayesrisk} and~\ref{def:margin}, 
we have the following nice property:\footnote{Note that for another choice of the zero-one loss definition (Definition~\ref{conv:Risque_de_B_Q}), the tie in the majority vote -- \ie, when $M_Q(x,y)=0$ -- would have been more complicated to handle, and the statement should  have been relaxed to
\begin{equation*}%
  \prob{(\xb,y)\sim D'}\Big(M_Q(\xb,y)< 0 \Big)
  \ \leq\  R_{D'}(B_Q) \ \leq \
\prob{(\xb,y)\sim D'} \Big(M_Q(\xb,y)\leq 0 \Big)\,.
 \end{equation*}}\,
\begin{equation}\label{eq:convention}
    R_{D'}(B_Q)\ =\, 
\prob{(\xb,y)\sim D'}\Big(M_Q(\xb,y)\leq 0\Big)\,.
\end{equation}

\bigskip

The margin is not only related to the risk of the
majority vote, but also to Gibbs risk.  For that purpose, let us consider the first moment $\momentone(\MQ{D'})$ of the random variable $\MQ{D'}$ which is defined as
\begin{equation} \label{eq:margin_moment_one}
 \momentone(\MQ{D'}) \ \eqdef \, \esp{(\xb,y)\sim D'} M_Q(\xb,y) \,.
\end{equation} 
We can now rewrite the Gibbs risk (Definition~\ref{def:gibbsrisk}) as a function of $\momentone(\MQ{D'})$, since
 \begin{eqnarray} \nonumber
 R_{D'}(G_Q) & = &   \esp{f\sim Q} \ElinlossDprim(f)
 \ = \ \frac{1}{2} \left(1-\esp{(\xb,y)\sim D'}\esp{f\sim Q}y \cdot f(\xb) \right) \\ 
&=& \frac{1}{2}\Big(1-\!\!\esp{(\xb,y)\sim D'}\!\!M_Q(\xb,y)\Big) \nonumber \\
&=& \frac{1}{2}\Big(1-\momentone(\MQ{D'})\Big)\,. \label{eq:RGibbsMq}
 \end{eqnarray}
Similarly, we can rewrite the expected disagreement as a function of the second moment of the margin. 
We use $\momenttwo(\MQ{D'})$ to denote the second moment. Since $y\in\{-1,1\}$ and, therefore, $y^2=1$, the second moment of the margin does not rely on labels. Indeed, we have
\begin{eqnarray} \label{eq:margin_moment_two}
  \momenttwo(\MQ{D'}) & \eqdef & \esp{(\xb,y)\sim D'} \Big[ M_Q\ex \Big]^2
  \\ \nonumber
   &=&\ \esp{\ex\sim D'} y^2 \cdot \Big[ \esp{f\sim Q}  f(\xb) \Big]^2 
   \\ \nonumber
    &=&\ \esp{\xb\sim D'_\Xcal}\Big[  \esp{f\sim Q}  f(\xb) \Big]^2 \,.
\end{eqnarray}
Hence, from the last equality and Definition~\ref{def:disagreement}, the expected disagreement can be expressed as 
\begin{eqnarray}
  \nonumber
  \dQ &=& \frac{1}{2} \bigg( 1- \esp{\xb\sim D_\Xcal'} \left[ \esp{f\sim Q}   f(\xb) \right]^2 \bigg)
\\ \label{eq:dQ_Mq}
  &=& \frac{1}{2} \Big( 1- \momenttwo(\MQ{D'}) \Big) \,.
 \end{eqnarray}

\bigskip
Equation~\eqref{eq:dQ_Mq} shows that $0 \, \leq\, \dQ\, \leq \, 1/2$, since $0 \, \leq\, \momenttwo(\MQ{D'})\, \leq \, 1$.  Furthermore, we can upper-bound the disagreement more tightly than simply saying it is at most 1/2 by making use of the value of the Gibbs risk. To do so, let us write the variance of the margin as 
\begin{eqnarray} \nonumber
 \vMQ &\eqdef& \var{\ex\sim D'}\big(M_Q\ex\big)\\
 &=& \momenttwo(\MQ{D'}) \,-\, \left( \momentone(\MQ{D'}) \right)^2\,.  \label{eq:VarMq:mu2-sqr(mu1)}
\end{eqnarray}
Therefore, as the variance cannot be negative, it follows that
\begin{eqnarray*}
  \momenttwo(\MQ{D'}) &\geq &\left( \momentone(\MQ{D'}) \right)^2\,,
 \end{eqnarray*}
which implies that
\begin{eqnarray} \label{eq:relation_dq_rdgq}
  1 - 2\cdot \dQ  & \geq & \left( 1-2\cdot R_{D'}(G_Q) \right)^2 \,.
\end{eqnarray}
Easy calculation then gives the desired bound of $\dQ$ (that is based on the Gibbs risk):
\begin{eqnarray}
 \label{eq:VarMq:dQbound}
    \dQ & \leq & 2 \cdot R_{D'}(G_Q) \cdot \big( 1-R_{D'}(G_Q) \big) \,.
\end{eqnarray}
We therefore have the following proposition.
\begin{proposition} \label{prop:dQ_bound}
  For any distribution $Q$ on a set of voters and any distribution $D'$ on \mbox{$\Xcal\!\times\!\{-1,1\}$}, we have
  \begin{equation*}
   \dQ \ \leq \ 2 \cdot R_{D'}(G_Q) \cdot \big( 1-R_{D'}(G_Q) \big) \ \leq \ \frac{1}{2}\,.
  \end{equation*}
  Moreover, if $\dQ=\frac{1}{2}$ \ then \ $R_{D'}(G_Q)=\frac{1}{2}$\,.
\end{proposition}
\begin{proof}
Equation~\eqref{eq:VarMq:dQbound} gives the first inequality. The rest of the proposition directly follows from the fact that $f(x)=2 x(1-x)$ is a parabola whose (unique) maximum is at the point~$( \frac{1}{2},\frac{1}{2})$.
\end{proof}

\subsection{Rediscovering the bound \  $R_{D'} (B_Q) \leq 2\cdot R_{D'} (G_Q)$}
\label{section:2RG_rediscover}

The well-known factor of two with which one can transform a bound on the Gibbs risk $\RGQ$ into a bound on the risk $\RBQ$ of the majority vote is usually justified by an argument similar to the one given in Remark~\ref{remark:2gibbs}.  However, as shown by the proof of Proposition~\ref{thm:2RG_rediscover}, the result can also be obtained by considering that the risk of the majority vote is the probability that the margin $\MQ{D'}$ is lesser than or equal to zero (Equation~\ref{eq:convention}) and by simply applying Markov's inequality (Lemma~\ref{lem:markov}, provided in Appendix~\ref{section:appendix_auxmath}).

\begin{proposition}
\label{thm:2RG_rediscover}
 For any distribution $Q$ on a set of voters and any distribution $D'$ on \mbox{$\Xcal\!\times\!\{-1,1\}$}, we have 
\begin{equation*}
 R_{D'}(B_Q) \ \leq \ 2\cdot R_{D'} (G_Q)\,.
\end{equation*}
\end{proposition}
\begin{proof}
 Starting from Equation~\eqref{eq:convention} and using Markov's inequality (Lemma~\ref{lem:markov}), we have
 {\small
 \begin{eqnarray*} 
 \nonumber
 R_{D'}(B_Q)&=&  \prob{(\xb,y)\sim D'}\big(M_Q(\xb,y)\leq 0\big)\\
 &=&\prob{(\xb,y)\sim D'}\big(1-M_Q(\xb,y)\geq 1\big)\\
 \nonumber
&\leq&  \esp{(\xb,y)\sim D'} \big(1-M_Q(\xb,y)\big)
 \hspace{26mm} \mbox{(Markov's inequality)}\hspace{-26mm}\\
&=&  1-\esp{(\xb,y)\sim D'} M_Q(\xb,y) \\
&=&  1- \momentone(\MQ{D'})\\[2mm]
&=& 2\cdot R_{D'}(G_Q)\,.
 \end{eqnarray*}
 }The last equality is directly obtained from Equation~\eqref{eq:RGibbsMq}.
\end{proof}

This proof highlights that we can upper-bound $\RBQ$ by considering solely the first moment of the margin $\momentone(\MQ{D'})$. Once we realize this fact, it becomes natural to extend this result to higher moments. We do so in the following subsection where we make use of Chebyshev's inequality (instead of Markov's inequality), which uses not only the first, but also the second moment of the margin. %
This gives rise to the $\Cbound$ of Theorem~\ref{thm:C-bound}.

\subsection{The~\Cbound:~a~Bound~on~$\RBQ$~That~Can~Be~Much~Smaller~Than~$\RGQ$}
\label{section:Cbound_thm}

Here is the bound on which most of the results of this paper are based. We refer to it as the \Cbound. It was first introduced (but in a different form) in \cite{nips06-mv}.\footnote{We present the form used by~\cite{nips06-mv} in Remark~\ref{remark:original_CQ} at the end of the present subsection.}
We give here three different (but equivalent) forms of the \Cbound. Each one highlights a different property or behavior of the bound. Figure~\ref{fig:CQ_3forms} illustrates these behaviors.
 \begin{figure}[t]
\subfloat[First form.]{\includegraphics[height=4.6cm]{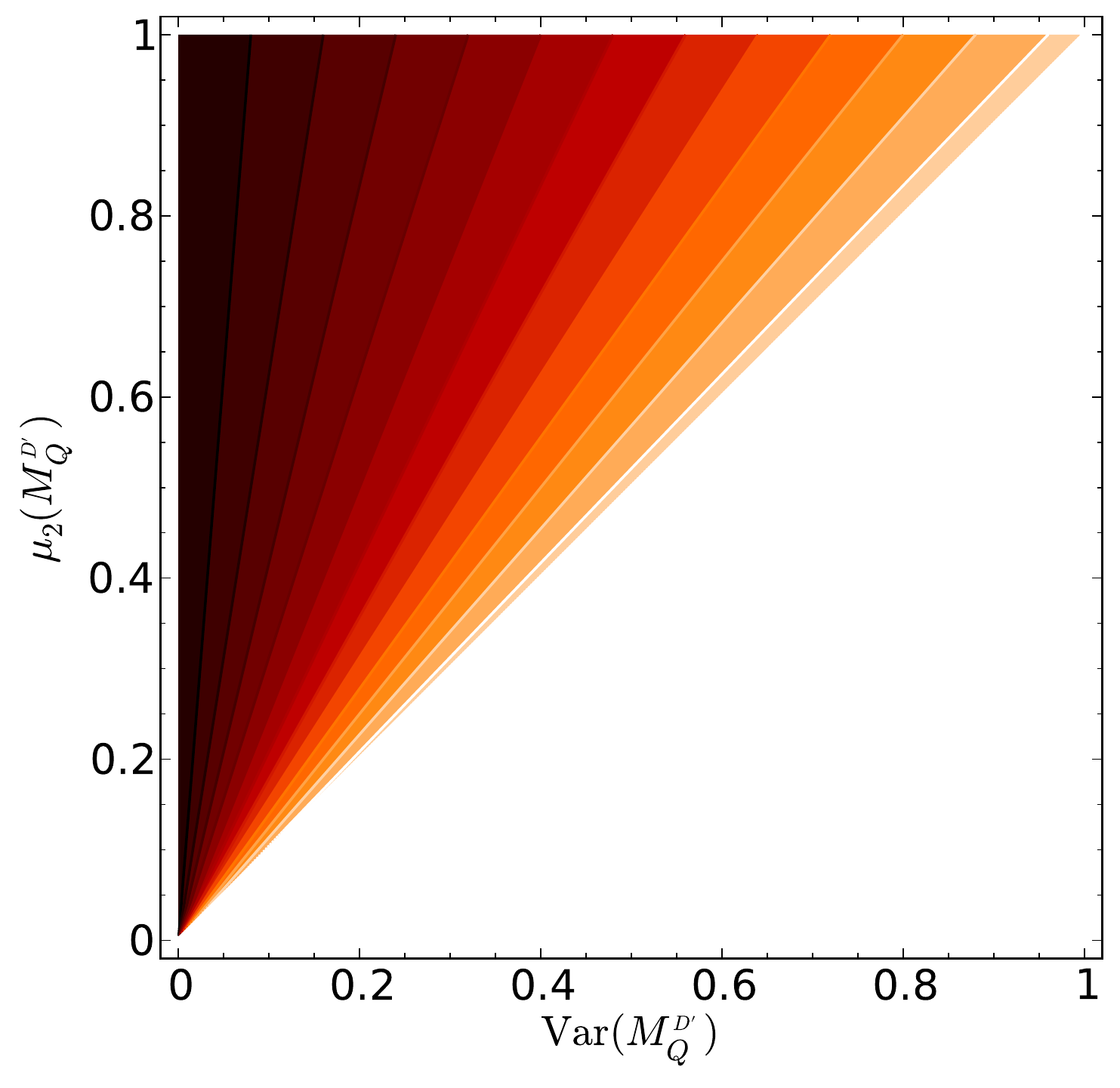} \label{fig:CQ_form1} } 
\subfloat[Second form.]{\includegraphics[height=4.6cm]{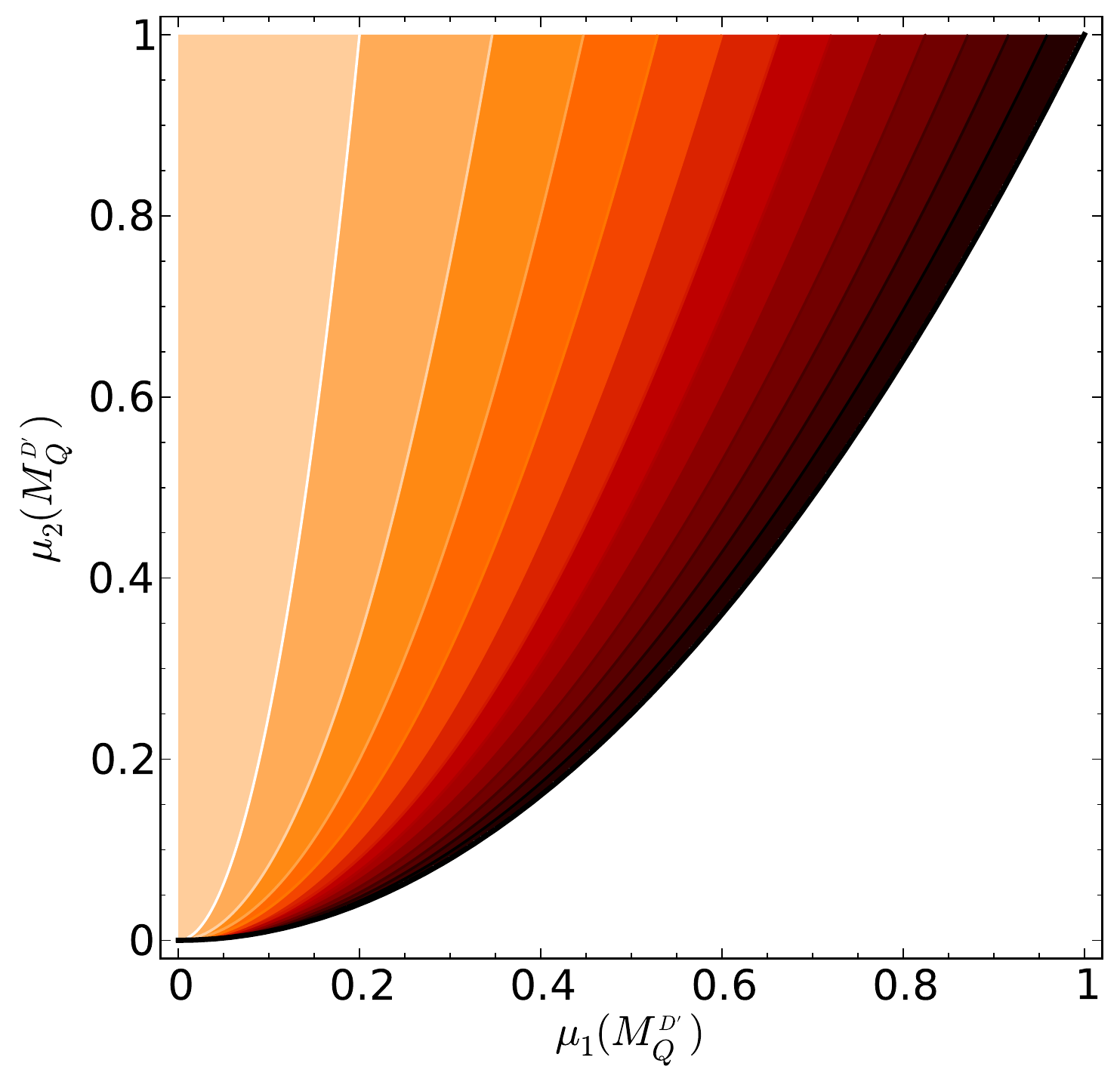} \label{fig:CQ_form2}} 
\subfloat[Third form.]{\includegraphics[height=4.6cm]{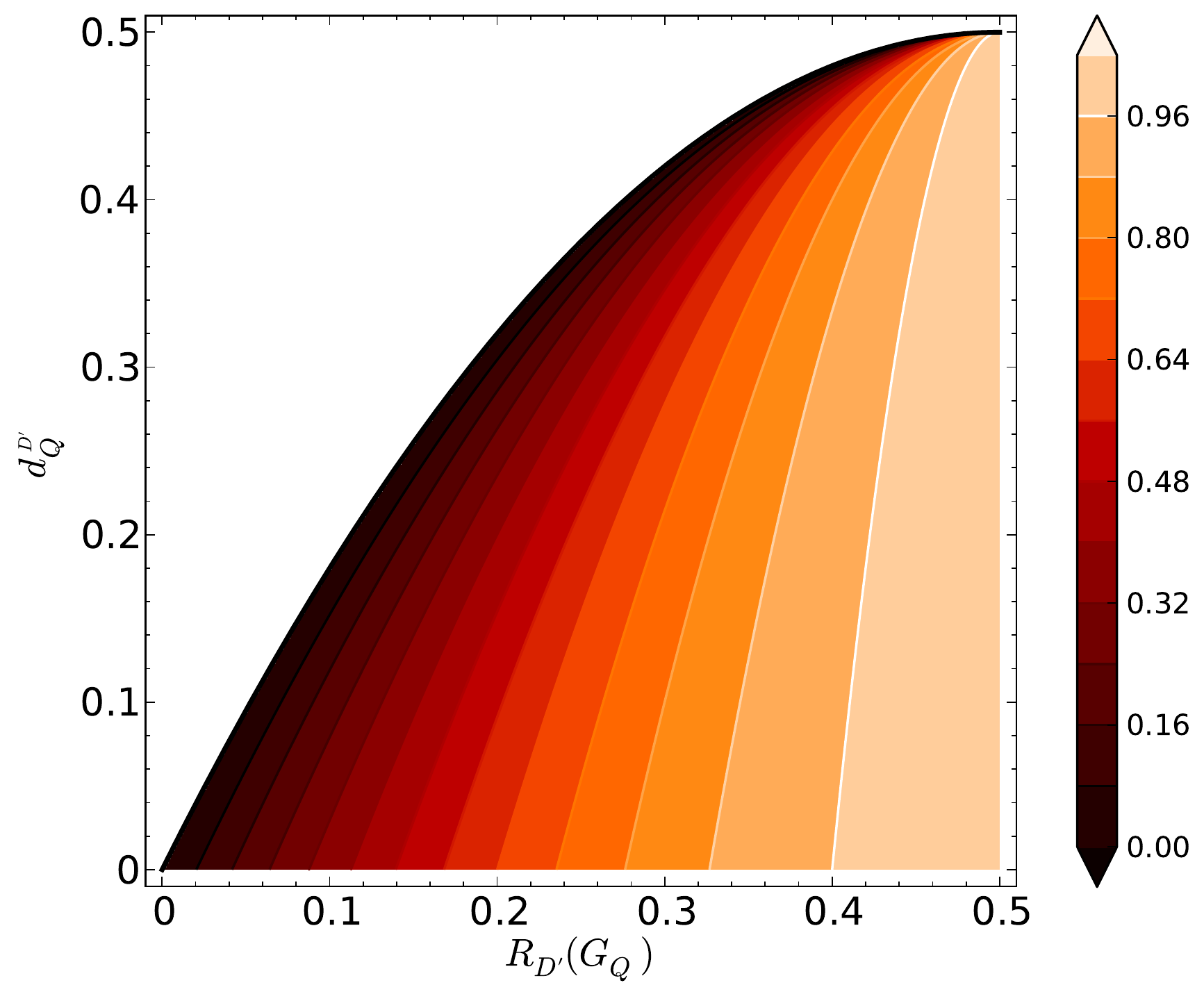} \label{fig:CQ_form3}} 
\caption{Contour plots of the \Cbound.}
\label{fig:CQ_3forms}
\end{figure}

It is interesting to note that the proof of Theorem~\ref{thm:C-bound} below has the same starting point as the proof of Proposition~\ref{thm:2RG_rediscover}, but uses Chebyshev's inequality instead of Markov's inequality (respectively Lemmas~\ref{th:chebychev} and~\ref{lem:markov}, both provided in Appendix~\ref{section:appendix_auxmath}). Therefore, Theorem~\ref{thm:C-bound} is based on the variance of the margin in addition of its mean.
\begin{theorem}[The \Cbound]
\label{thm:C-bound}
 For any distribution $Q$ on a set of voters and any distribution $D'$ on $\Xcal\!\times\!\{-1,1\}$,\ \  if\, \mbox{$\momentone(\MQ{D'}) > 0$} (i.e., $R_{D'}(G_Q) < 1/2$), we have 
\begin{equation*}
 R_{D'}(B_Q) \ \leq \ \Ccal_Q^{D'}\,,\\[-2mm]
\end{equation*}
where\\[-2mm]
\begin{equation*}
\Ccal_Q^{D'} 
\ \eqdef \ 
\underbrace{\begin{array}{c}
\dfrac{\phantom{\Big(}\vMQ\phantom{\Big.^2}}  { \phantom{\Big(} \momenttwo(\MQ{D'})\phantom{\Big)}}
\\[-3mm]\phantom{.}\end{array}}_{{\mbox{First form}}} 
\  = \ 
\underbrace{\begin{array}{c}
1-\dfrac{\Big(\momentone(\MQ{D'})\Big)^2}{\phantom{\Big(} \momenttwo(\MQ{D'})\phantom{\Big)} }
\\[-3mm]\phantom{.}\end{array}}_{{\mbox{Second form}}} 
\ = \ 
\underbrace{\begin{array}{c}
1-\dfrac{\Big( 1- 2\cdot R_{D'}(G_Q) \Big)^2}{ \phantom{\Big(}1 - 2\cdot \dQ \phantom{\Big)} }
\\[-3mm]\phantom{.}\end{array}}_{{\mbox{Third form}}} 
\,.
\end{equation*}
\end{theorem}

\begin{proof}
 Starting from Equation~\eqref{eq:convention} and using the one-sided Chebyshev inequality (Lem\-ma~\ref{th:chebychev}), with  $X\!=\!-M_Q(\xb,y)$,\ \ $\mu=\esp{(\xb,y)\sim D'}\!\!\!\big(-\!M_Q(\xb,y)\big)$ and $a=\esp{(\xb,y)\sim D'}\!\!\!M_Q(\xb,y)$, we~obtain
 {\small
 \begin{eqnarray} 
 \nonumber
 R_{D'}(B_Q)&=& 
 \prob{(\xb,y)\sim D'}\Big(M_Q(\xb,y)\leq 0\Big)
 \\ \nonumber
&=&  
\prob{(\xb,y)\sim D'}\left(-M_Q(\xb,y) + \esp{(\xb,y)\sim D'}\!\!\!M_Q(\xb,y) \geq \esp{(\xb,y)\sim D'}\!\!\!M_Q(\xb,y)\right) 
\\ \nonumber%
&\leq& 
\frac{\var{(\xb,y)\sim D'}(M_Q(\xb,y))}{\var{(\xb,y)\sim D'}(M_Q(\xb,y))+ \left(\esp{(\xb,y)\sim D'}M_Q(\xb,y)\right)^2} 
 \hspace{11mm} \mbox{(Chebyshev's inequality)}\hspace{-11mm}\\
\label{eq:cq-cheby-form1}
&=& \frac{\vMQ}
{ \momenttwo(\MQ{D'})\phantom{\Big)}-\Big(\momentone(\MQ{D'})\Big)^2+ \Big(\momentone(\MQ{D'})\Big)^2}
\quad = \quad 
\frac{\vMQ}{ \momenttwo(\MQ{D'})\phantom{\Big)}}
\\ 
\nonumber
&=& \frac{ \momenttwo(\MQ{D'})\phantom{\Big)}-\Big(\momentone(\MQ{D'})\Big)^2}{\phantom{\Big(} \momenttwo(\MQ{D'})\phantom{\Big)} }\\
\label{eq:cq-cheby-form2}
& = &
1-\frac{\Big(\momentone(\MQ{D'})\Big)^2}{\phantom{\Big(} \momenttwo(\MQ{D'})\phantom{\Big)} }
\\ \label{eq:cq-cheby-form3}
&=& 
1-\frac{\Big( 1- 2\cdot R_{D'}(G_Q) \Big)^2}{ \phantom{\Big(}1 - 2\cdot \dQ \phantom{\Big)} }
  \,.
 \end{eqnarray}
 }Lines~\eqref{eq:cq-cheby-form1} and~\eqref{eq:cq-cheby-form2} respectively present the first and the second forms of $\Ccal_Q^{D'}$, and follow from the definitions of
  $\momentone(\MQ{D'})$,\,  $\momenttwo(\MQ{D'})$,\, and $\vMQ$ (see Equations~\ref{eq:margin_moment_one}, \ref{eq:margin_moment_two} and~\ref{eq:VarMq:mu2-sqr(mu1)}).
  The third form of $\Ccal_Q^{D'}$ is obtained at Line~\eqref{eq:cq-cheby-form3} using $\momentone(\MQ{D'}) =1- 2\cdot R_{D'}(G_Q)$ and $\momenttwo(\MQ{D'})=1- 2\cdot \dQ$, which can be derived directly from Equations~\eqref{eq:RGibbsMq} and~\eqref{eq:dQ_Mq}.
\end{proof}
The third form  of the \Cbound shows that the bound decreases when the Gibbs risk $R_{D'}(G_Q)$ decreases or when the disagreement $\dQ$ increases. This new bound therefore suggests that a majority vote should perform a trade-off between the Gibbs risk and the disagreement in order to achieve a low Bayes risk. This is more informative than the usual bound of Proposition~\ref{thm:2RG_rediscover}, which focuses solely on the minimization of the Gibbs risk.

The first form of the \Cbound highlights that its value is always positive (since the variance and the second moment of the margin are positive), whereas the second form of the \Cbound highlights that it cannot exceed one. Finally, the fact that $\dQ=\frac{1}{2} \Rightarrow R_{D'}(G_Q)=\frac{1}{2}$ \ (Proposition~\ref{prop:dQ_bound}) implies that the bound is always defined, since $R_{D'}(G_Q)$ is here assumed to be strictly less than $\frac{1}{2}$.

\begin{remark}\rm
\label{remark:original_CQ}
As explained before, the \Cbound was originally stated in \citet{nips06-mv}, but in a different form. It was presented as a function of $W_Q(\xb,y)$, the $Q$-weight  of voters making an error on example $(\xb, y)$. More precisely, the \Cbound was presented as follows:
\begin{equation*}
 \Ccal_Q^D 
  \ =\
 \frac{ \var{\ex\sim D'}\big(W_Q(\xb,y)\big)}{ \var{\ex\sim D'}\big(W_Q(\xb,y)\big)+ \left(1/2-\RGQ\right)^2}\,.
\end{equation*}
It is easy to show that this form is equivalent to the three forms stated in  Theorem~\ref{thm:C-bound}, and that
$ W_Q(\xb,y) $ and $M_Q(\xb,y)$ are related by
\begin{equation*}
W_Q(\xb,y)
  \ \eqdef \ 
 \esp{f\sim Q} \linloss\big( f(\xb), y \big) 
 \ = \
 \frac{1}{2} \left(1-y \cdot \esp{f\sim Q}  f(\xb) \right)
 \ = \ \frac{1}{2} \Big(1- M_Q(\xb,y) \Big)
 \,.
\end{equation*}
However, we do not discuss further this form of the \Cbound here, since we now consider that the margin~$M_Q(\xb,y)$ is a more natural notion than $W_Q(\xb,y)$.
\end{remark}

\subsection{Statistical Analysis of the \Cbound's Behavior}
\label{section:Cbound_statistical_analysis}

This section presents some properties of the $\Cbound$.  	
In the first place, we discuss the conditions under which the \Cbound is optimal, in the sense that if the only information that one has about a majority vote is the first two moments of its margin distribution, it is possible that the value given by the \Cbound \emph{is} the Bayes risk, \ie, $\CQ = \RBQ$.\footnote{In other words, the \emph{optimality of the \Cbound} means here that there exists a random variable with the same first moments as the margin distribution, such that Chebyshev's inequality of Lemma~\ref{th:chebychev} is reached.}
In the second place, we show that the \Cbound can be arbitrarily small, especially in the presence of ``non-correlated'' voters, even if the Gibbs risk is large, \ie, $\CQ \ll \RGQ$.

\subsubsection{Conditions of Optimality}
 For the sake of simplicity, let us focus on a random variable~$M$ that represents a margin distribution (here, we ignore underlying distributions $Q$ on $\Hcal$ and~$D'$ on $\Xcal{\times}\{-1,1\}$) of first moment $\momentone(M)$ and second moment $\momenttwo(M)$. By Equation~\eqref{eq:convention}, we have
 \begin{eqnarray} \label{eq:RBM}
  R(B_{M}) &\eqdef& \prob{}\, (M \leq 0)\,.
 \end{eqnarray}
Moreover, $R(B_{M})$ is upper-bounded by $\C_M$, the \Cbound given by the second form of Theorem~\ref{thm:C-bound},
 \begin{eqnarray} \label{eq:CM}
\C_{M} &\eqdef& 1-\dfrac{\big(\,\momentone(M)\,\big)^2}{\momenttwo(M) }\,.
 \end{eqnarray}
The next proposition shows when the \Cbound can be achieved.
 
\newcommand{\Mtilde}{{\widetilde{M}}}
\begin{proposition}[Optimality of the \Cbound] 
 \label{prop:CQ_optimal}
  Let $M$ be any random variable that represents the margin of a majority vote. Then
  there exists a random variable $\Mtilde$ such that 
  \begin{eqnarray} \label{eq:optimality:2}
 \momentone(\Mtilde) = \momentone(M)\,, \quad
 \momenttwo(\Mtilde) = \momenttwo(M)\,, \quad\mbox{and}\quad
 \C_{\Mtilde} = \C_{M} = R(B_{\Mtilde})\,
   \end{eqnarray}
 if and only if
  \begin{equation} \label{eq:optimality:1}
  0<\momenttwo(M) \leq \momentone(M)\,.
  \end{equation}
\end{proposition}
\begin{proof} 
   First, let us show that \eqref{eq:optimality:1} implies \eqref{eq:optimality:2}. Given $0<\momenttwo(M) \leq \momentone(M)$, we consider a distribution $\Mtilde$ concentrated in two points defined as
   \begin{equation*}
 \Mtilde \ =\ \begin{cases}
              0 &\mbox{with probability $\C_{M} = 1-\dfrac{\big( \momentone(M) \big)^2}{\momenttwo(M)}$}\,,\\[3mm]
              \dfrac{\momenttwo(M)}{\momentone(M)} & \mbox{with probability $1-\C_{M} = \dfrac{\big( \momentone(M) \big)^2}{\momenttwo(M)}$}\,.
             \end{cases}
   \end{equation*}
     This distribution has the required moments, as 
     {\small
     \begin{eqnarray*}
    \momentone(\Mtilde) =\dfrac{\big( \momentone(M) \big)^2}{\momenttwo(M)} \!\left[\frac{\momenttwo(M)}{\momentone(M)}\right] \! =  \momentone(M)\,,
    \ \mbox{and}\ \ 
    \momenttwo(\Mtilde) =\dfrac{\big( \momentone(M) \big)^2}{\momenttwo(M)} \! \left[  \frac{\momenttwo(M)}{\momentone(M)}\right]^2 \! =  \momenttwo(M)\,.
     \end{eqnarray*}
    }It follows directly from Equation~\eqref{eq:CM} that $\C_{\Mtilde}\! = \C_{M}$. 
    Moreover, by Equation~\eqref{eq:RBM} and because $\frac{\momenttwo(M)}{\momentone(M)} > 0$, we obtain as desired
   \begin{equation*}
    R(B_{\Mtilde}) \ =\  \prob{}\, (\Mtilde \leq 0) \ =\ \C_{M}\,.
   \end{equation*}

   Now, let us show that \eqref{eq:optimality:2} implies \eqref{eq:optimality:1}. Consider a distribution $\Mtilde$ such that the equalities of Line~\eqref{eq:optimality:2} are satisfied. By Proposition~\ref{thm:2RG_rediscover} and Equation~\eqref{eq:RGibbsMq}, we obtain the inequality
   \begin{equation*}
   \C_{M}\ =\ R(B_{\Mtilde}) \ \leq \ 1-\momentone(\Mtilde) \ =\ 1-\momentone(M)\,.
   \end{equation*}
   Hence, by the definition of $\C_{M}$, we have
   \begin{equation*}
   1-\dfrac{\big(\,\momentone(M)\,\big)^2}{\momenttwo(M) } \ \leq \ 1-\momentone(M)\,,
   \end{equation*}   
   which, by straightforward calculations, implies
   $0 \ < \ \momenttwo(M) \ \leq \ \momentone(M)\,,$
   and we are done.
\end{proof}

We discussed in Section~\ref{section:margin_and_moments} the multiple connections between the moments of the margin, the Gibbs risk and the expected disagreement of a majority vote. In the next proposition, we exploit these connections to derive expressions equivalent to Line~\eqref{eq:optimality:1} of Proposition~\ref{prop:CQ_optimal}. Thus, this shows three (equivalent) necessary conditions under which the \Cbound is optimal.
  
\begin{proposition}
\label{prop:CQ_optimal_suite}
For any distribution $Q$ on a set of voters and any distribution $D'$ on $\Xcal\!\times\!\{-1,1\}$, if \mbox{$\momentone(\MQ{D'}) > 0$} (i.e., $R_{D'}(G_Q) < 1/2$), then the three following statements are equivalent:
 \begin{enumerate}
  \item[(i)] $\momenttwo(\MQ{D'}) \ \leq \ \momentone(\MQ{D'})$ ;
  \item[(ii)] $\RGQ \ \leq \ \dQ$ ;
  \item[(iii)] $\CQ \ \leq \ 2\,\RGQ$ .
 \end{enumerate}
\end{proposition}
 \begin{proof} The truth of  $(i) \Leftrightarrow (ii)$ is a direct consequence of Equations~\eqref{eq:RGibbsMq} and~\eqref{eq:dQ_Mq}. 
 To prove $(ii) \Leftrightarrow (iii)$, we express $\CQ$ in its third form. Straightforward calculations give 
  \begin{equation*}
   \CQ \, =\, 1 - \frac{\left(1-2\,\RGQ\right)^2}{1-2\,\dQ} \ \leq \ 2\RGQ 
   \quad\Longleftrightarrow\quad
   \RGQ \ \leq \ \dQ\,.\\[-6mm]
  \end{equation*}
 \end{proof}
Propositions~\ref{prop:CQ_optimal} and~\ref{prop:CQ_optimal_suite} illustrate an interesting result: the \Cbound is optimal if and only if its value is lower than twice the Gibbs risk, the classical bound on the risk of the majority vote (see Proposition~\ref{thm:2RG_rediscover}).

\subsubsection{The \Cbound Can Be Arbitrarily Small, Even for Large Gibbs Risks}
The next result shows that, when the number of voters tends to
infinity (and the weight of each voter tends to zero), the variance
of $M_Q$ will tend to $0$ provided that the average of the
covariance of the outputs of all pairs of distinct voters is $\leq 0$.
In particular, the variance will always tend to $0$ if the risk of
the voters is pairwise independent. To quantify the independence between voters, we use the concept of covariance of a pair of voters $(f_1, f_2)$:
\begin{eqnarray*}
 \covariance{D'}(f_1, f_2) & \eqdef & \cov{\ex\sim D'}\Big(y\cdot f_1(\xb),\, y\cdot f_2(\xb)\Big) \\[-1mm]
 &=& 
 \ \esp{(\xb,y)\sim D'} f_1(\xb)f_2(\xb)-\left(\esp{(\xb,y)\sim D'} f_1(\xb) \right)\left(\esp{(\xb,y)\sim D'} f_2(\xb) \right)\,.
\end{eqnarray*}
Note that the covariance $\covariance{D'}(f_1, f_2)$ is zero when $f_1$ and $f_2$ are independent \mbox{(uncorrelated)}.
\vspace{-3mm}
\begin{proposition}\label{prop:convergence}
For any countable set of voters $\Hcal$, any distribution
$Q$ on $\Hcal$, and any distribution $D'$ on \mbox{$\Xcal\!\times\!\{-1,1\}$}, we have
\begin{eqnarray*}
   \vMQ
    &\leq&
   \sum_{f\in\Hcal }Q^2(f) \,+\, \sum_{f_1\in \Hcal}\sum_{\underset{f_2\ne f_1}{f_2\in
\Hcal:}}Q(f_1)Q(f_2) \cdot \covariance{D'}(f_1, f_2) \,.
\end{eqnarray*}
\end{proposition}
\begin{proof}
By the definition of the margin (Definition~\ref{def:margin}), we rewrite $M_Q\ex$ as a sum of random variables:
{\small
\begin{eqnarray*}
  \var{(\xb,y)\sim D'}\Big(M_Q\ex\Big) \hspace{-2.7cm} & & \\[-1mm]
  &=& \var{(\xb,y)\sim D'}\Bigg( \sum_{f\in\Hcal} Q(f) \cdot y \cdot f(\xb) \Bigg) 
  \\
  &=& \sum_{f\in\Hcal} Q^2(f) \var{(\xb,y)\sim D'}\Big( y \cdot f(\xb) \Big) 
  \,+\, \sum_{f_1\in \Hcal}\sum_{\underset{f_2\ne f_1}{f_2\in\Hcal:}}Q(f_1)Q(f_2)\cov{\ex\sim D'}\Big(y\cdot f_1(\xb), y\cdot f_2(\xb)\Big)\,. \\[-4mm]
\end{eqnarray*}
}The inequality is a consequence of the  fact that
 $ \forall f\in\Hcal : \var{(\xb,y)\sim D'}\Big( y \cdot f(\xb) \Big)  \leq 1$.
\end{proof}
The key observation that comes out of this result is that
\,$\sum_{f\in\Hcal}Q^2(f)$\, is usually much smaller than one.
Consider, for example, the case where $Q$ is uniform on $\Hcal$ with
$|\Hcal| = n$. Then\, $\sum_{f\in\Hcal }Q^2(f) = 1/n$. Moreover,
if $\covariance{D'}(f_1, f_2)\leq 0$\, 
for each pair of distinct classifiers
in $\Hcal$, then \,$\vMQ  \le 1/n$. Hence, in these cases, we
have that \,$\CQ \in {\cal O}(1/n)$\,\, whenever\,\,
$1\!-\!2\,R_{D'}(G_Q)$ and $1\!-\!2\,\dQ$ are larger than some positive constants independent
of $n$. Thus, even when $R_{D'}(G_Q)$ is large, we see that the \Cbound can
be arbitrarily close to $0$ as we increase the number of classifiers
having non-positive pairwise covariance of their risk. More
precisely, we have

\begin{corollary}\label{prop:powerindependent}
 Given $n$ independent voters under a uniform distribution $Q$, we have
 \begin{equation*}
  R_{D'}(B_Q) \ \leq \ \CQ 
  \ \leq \ \dfrac{1}{n\!\cdot\!\Big(1\!-\!2\,\dQ\Big)}
  \ \leq \ \dfrac{1}{n\!\cdot\!\Big(1\!-\!2\,R_{D'}(G_Q)\Big)^2}\,.
 \end{equation*}
\end{corollary}
\begin{proof}
The first inequality directly comes from the \Cbound (Theorem~\ref{thm:C-bound}).  The second inequality is a consequence of Proposition \ref{prop:convergence}, considering that in the case of a uniform distribution of independent voters, we have 
$\covariance{D'}(f_1, f_2) = 0$, and then $\vMQ  \le 1/n$.
Applying this to the first form of the \Cbound, we obtain

\begin{equation*}
	\CQ \ =\
	\frac{\vMQ}  { \momenttwo(\MQ{D'})} 
	\ = \  \frac{\vMQ}  {1\!-\!2\,\dQ}
	\ \leq\ \frac{\frac{1}{n}}  {1\!-\!2\,\dQ}
	\ =\ \frac{1}{n\!\cdot\!\Big(1\!-\!2\,\dQ\Big)}\,.
\end{equation*}
To obtain the third inequality, we simply apply Equation~\eqref{eq:relation_dq_rdgq}, and we are done.
\end{proof}

\subsection{Empirical Study of The Predictive Power of the \Cbound}
\label{section:Cbound_empirical_study}

To further motivate the use of the \Cbound, we investigate how its empirical value relates to the risk of the majority vote by conducting two experiments. The first experiment shows that the \Cbound clearly outperforms the individual capacity of the other quantities of Theorem~\ref{thm:C-bound} in the task of predicting the risk of the majority vote. The second experiment shows that the $\Cbound$ is a great stopping criterion for Boosting algorithms.

\subsubsection{Comparison with Other Indicators}
We study how $\RGQ$, $\vMQ$, $\dQ$
and $\CQ$ are respectively related to $\RBQ$. Note that these four quantities appear in the first form or the third form of the \Cbound (Theorem~\ref{thm:C-bound}).  We omit here the moments $\momentone(\MQ{D'})$ and $\momenttwo(\MQ{D'})$ required by the second form of the \Cbound, as there is a linear relation between $\momentone(\MQ{D'})$ and $\RGQ$, as well as between $\momenttwo(\MQ{D'})$ and $\dQ$.

 \begin{figure}[t]
\centering
\subfloat[Gibbs risk.]{\includegraphics[height=6cm]{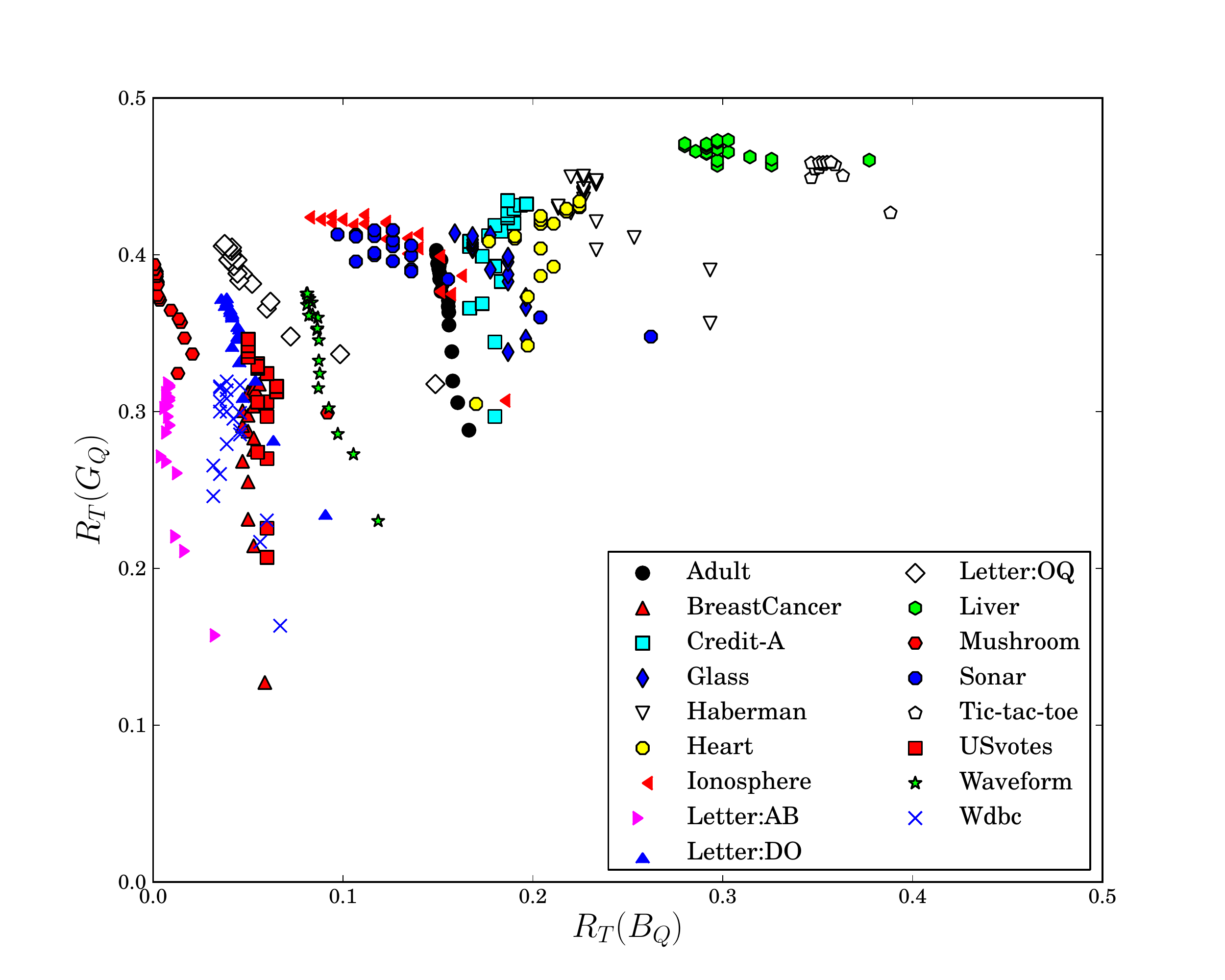} \label{fig:CQbon_RTGQ} } 
\subfloat[Variance of the margin.]{\includegraphics[height=6cm]{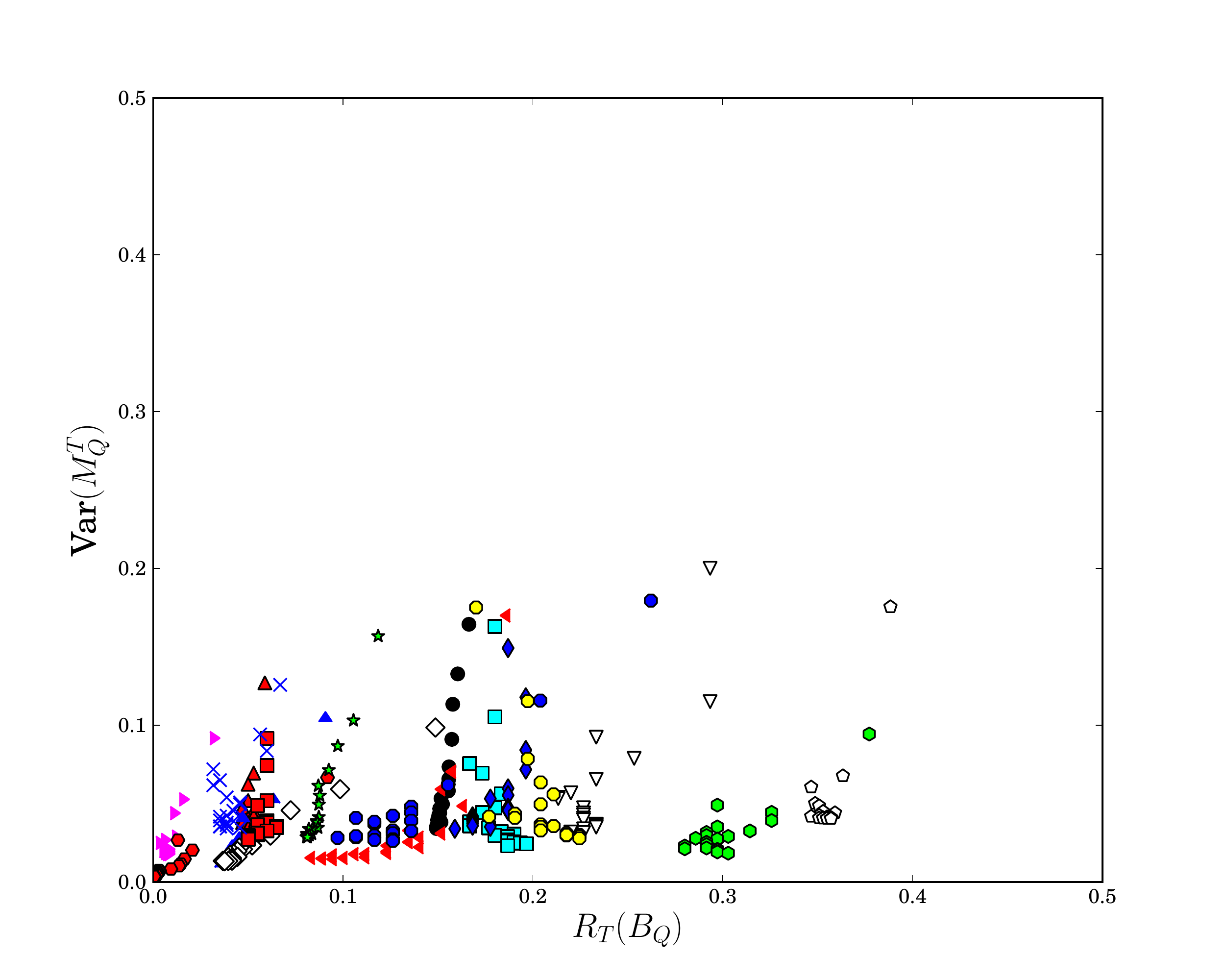} \label{fig:CQbon_vTMQ}} \\
\subfloat[Expected disagreement.]{\includegraphics[height=6cm]{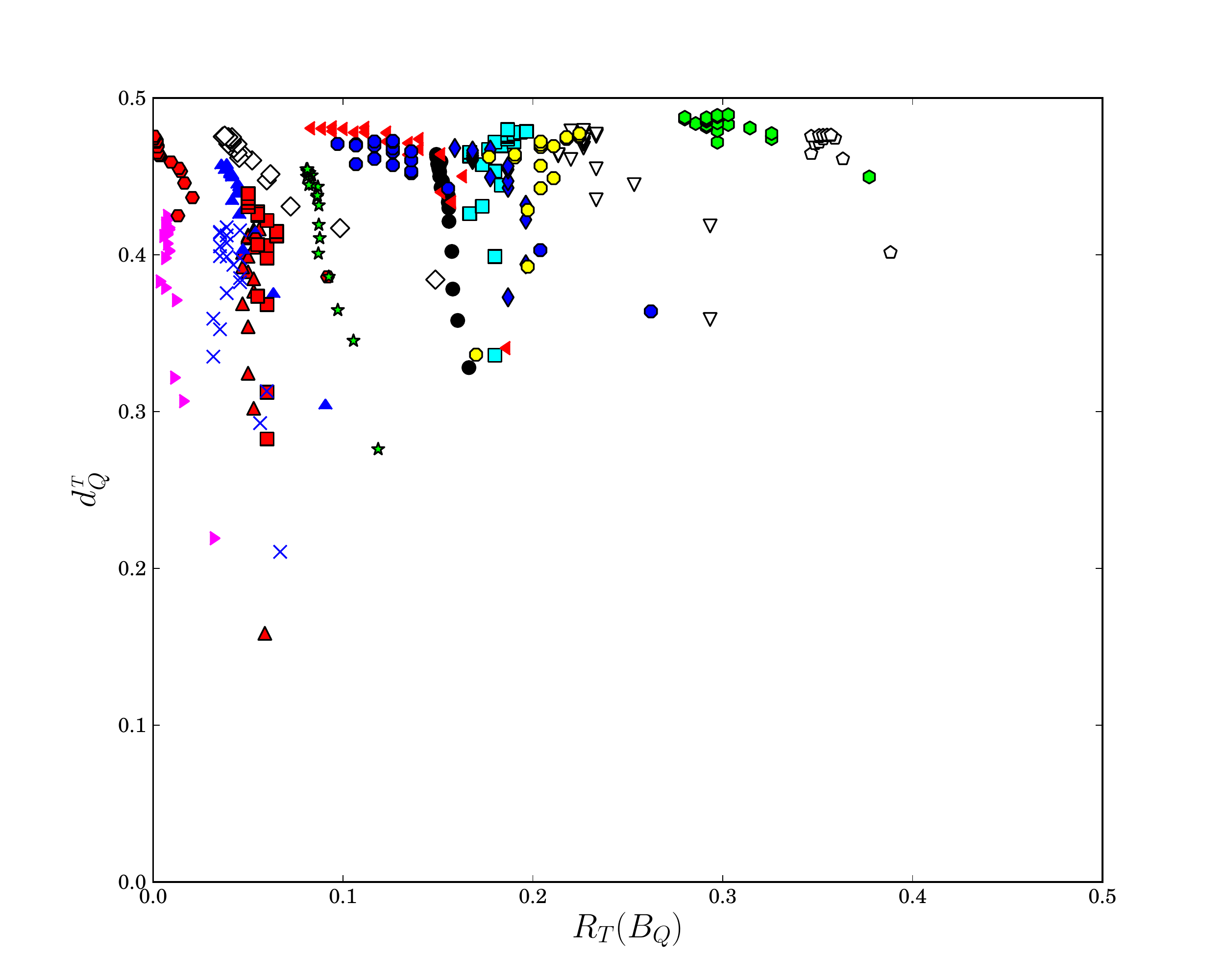} \label{fig:CQbon_dTQ} } 
\subfloat[\Cbound.]{\includegraphics[height=6cm]{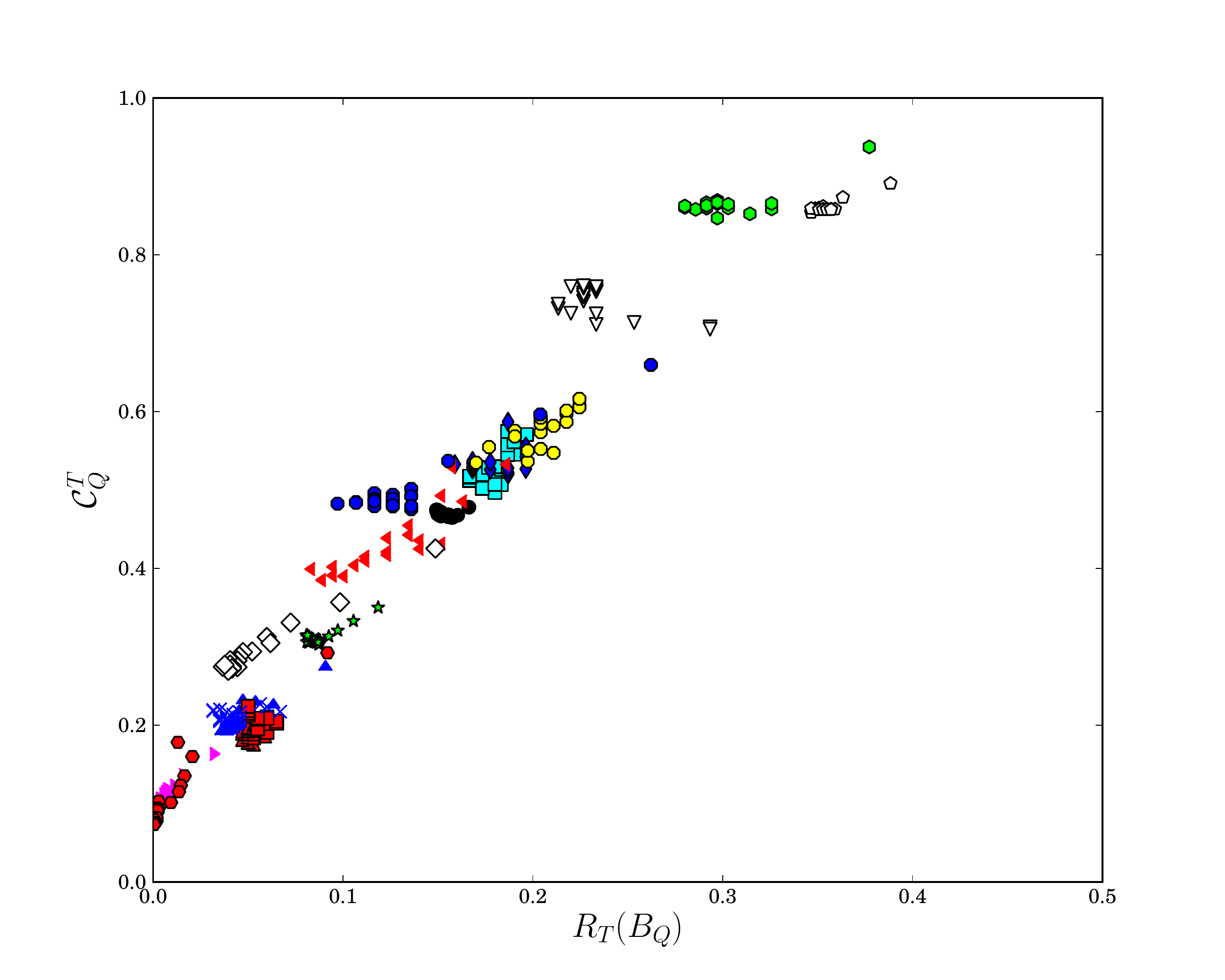} \label{fig:CQbon_CTQ} } 
\caption{$\RTBQ$ versus $\RTGQ$, $\vTMQ$, $\dTQ$ and $\CTQ$ respectively. }
\label{fig:C_Q est bon}
\end{figure}
The results of Figure~\ref{fig:C_Q est bon} are obtained with the AdaBoost
 algorithm of \cite{schapire99}, used with
``decision stumps'' as weak learners, on
several UCI binary classification data sets~\citep{uci-98}.
Each data set is split into two halves: a training set $S$ and a testing set $T$. We run AdaBoost on set $S$ for 100 rounds and compute the quantities $\RTGQ$, $\vTMQ$, $\dTQ$
and $\CTQ$ on set $T$ at every 5 rounds of boosting. That is, we study 20 different majority vote classifiers per data set. 

In Figure~\ref{fig:CQbon_RTGQ}, we see that we almost always have $\RTBQ  < \RTGQ$.
There is, however, no clear correlation between $\RTBQ$ and
$\RTGQ$. 
We also see no clear correlation between $\RTBQ$ and
$\vTMQ$ or between $\RTBQ$ and
$\dTQ$ in Figures~\ref{fig:CQbon_vTMQ} and~\ref{fig:CQbon_dTQ} respectively, except that generally $\RTBQ >
\vTMQ$ and $\RTBQ <
\dTQ$. In contrast, Figure~\ref{fig:CQbon_CTQ} shows a  strong
 correlation between $\CTQ$ and $\RTBQ$. Indeed, it is almost a linear relation! Therefore, the $\Cbound$ seems well-suited to characterize the behavior of the Bayes risk, whereas each of the individual quantities contained in the \Cbound is insufficient to do so.

\begin{table}[t]
\begin{center}
\begin{scriptsize}
\rowcolors{3}{black!10}{}
\begin{tabular}{lccccccccccc}
\toprule
\multicolumn{3}{c}{Data Set Information}  & \multicolumn{9}{c}{Risk $\RTBQ$ by Stopping Criterion {\it(and number of rounds performed)}} \\
\cmidrule(l r){1-3} \cmidrule(l r){4-12}
Name  & $|S|$ & $|T|$ &  \multicolumn{2}{c}{\Cbound $\CSQ$} & \multicolumn{2}{c}{Risk $\RSBQ$} & \multicolumn{2}{c}{Validation Set}& \multicolumn{2}{c}{Cross-Validation} & 1000 rounds \\
\cmidrule(l r){1-1} \cmidrule(l r){2-2} \cmidrule(l r){3-3} \cmidrule(l r){4-5} \cmidrule(l r){6-7} \cmidrule(l r){8-9} \cmidrule(l r){10-11} \cmidrule(l r){12-12}
\!Adult\!&\!400\!&\!11409\!&\!0.166\!&\!\textit{(149)}\!&\!0.169\!&\!\textit{(314)}\!&\!\textbf{0.165}\!&\!\textit{(13)}\!&\!0.166\!&\!\textit{(97)}\!&\!0.172\! \\
\!BreastCancer\!&\!341\!&\!342\!&\!0.050\!&\!\textit{(127)}\!&\!0.047\!&\!\textit{(48)}\!&\!\textbf{0.041}\!&\!\textit{(57)}\!&\!0.047\!&\!\textit{(108)}\!&\!0.058\! \\
\!Credit-A\!&\!326\!&\!327\!&\!0.187\!&\!\textit{(346)}\!&\!0.199\!&\!\textit{(854)}\!&\!\textbf{0.156}\!&\!\textit{(9)}\!&\!0.174\!&\!\textit{(47)}\!&\!0.199\! \\
\!Glass\!&\!107\!&\!107\!&\!0.252\!&\!\textit{(72)}\!&\!\textbf{0.196}\!&\!\textit{(299)}\!&\!0.346\!&\!\textit{(6)}\!&\!0.290\!&\!\textit{(35)}\!&\!\textbf{0.196}\! \\
\!Haberman\!&\!147\!&\!147\!&\!0.320\!&\!\textit{(27)}\!&\!0.320\!&\!\textit{(45)}\!&\!\textbf{0.279}\!&\!\textit{(1)}\!&\!0.320\!&\!\textit{(38)}\!&\!0.340\! \\
\!Heart\!&\!148\!&\!149\!&\!0.215\!&\!\textit{(124)}\!&\!0.289\!&\!\textit{(950)}\!&\!\textbf{0.181}\!&\!\textit{(31)}\!&\!0.195\!&\!\textit{(14)}\!&\!0.289\! \\
\!Ionosphere\!&\!175\!&\!176\!&\!\textbf{0.085}\!&\!\textit{(210)}\!&\!0.120\!&\!\textit{(56)}\!&\!0.142\!&\!\textit{(2)}\!&\!0.114\!&\!\textit{(67)}\!&\!\textbf{0.085}\! \\
\!Letter:AB\!&\!400\!&\!1155\!&\!\textbf{0.005}\!&\!\textit{(42)}\!&\!0.014\!&\!\textit{(17)}\!&\!0.061\!&\!\textit{(2)}\!&\!\textbf{0.005}\!&\!\textit{(60)}\!&\!0.010\! \\
\!Letter:DO\!&\!400\!&\!1158\!&\!\textbf{0.041}\!&\!\textit{(179)}\!&\!\textbf{0.041}\!&\!\textit{(44)}\!&\!0.143\!&\!\textit{(1)}\!&\!0.044\!&\!\textit{(83)}\!&\!0.043\! \\
\!Letter:OQ\!&\!400\!&\!1136\!&\!0.050\!&\!\textit{(65)}\!&\!0.050\!&\!\textit{(138)}\!&\!0.063\!&\!\textit{(26)}\!&\!\textbf{0.044}\!&\!\textit{(118)}\!&\!0.049\! \\
\!Liver\!&\!172\!&\!173\!&\!\textbf{0.289}\!&\!\textit{(541)}\!&\!\textbf{0.289}\!&\!\textit{(743)}\!&\!0.335\!&\!\textit{(5)}\!&\!\textbf{0.289}\!&\!\textit{(603)}\!&\!0.295\! \\
\!Mushroom\!&\!400\!&\!7724\!&\!\textbf{0.010}\!&\!\textit{(612)}\!&\!0.024\!&\!\textit{(38)}\!&\!0.079\!&\!\textit{(6)}\!&\!0.024\!&\!\textit{(51)}\!&\!\textbf{0.010}\! \\
\!Sonar\!&\!104\!&\!104\!&\!0.192\!&\!\textit{(688)}\!&\!0.250\!&\!\textit{(20)}\!&\!0.317\!&\!\textit{(2)}\!&\!\textbf{0.163}\!&\!\textit{(34)}\!&\!0.202\! \\
\!Tic-tac-toe\!&\!400\!&\!558\!&\!0.389\!&\!\textit{(59)}\!&\!0.364\!&\!\textit{(2)}\!&\!\textbf{0.358}\!&\!\textit{(5)}\!&\!0.403\!&\!\textit{(9)}\!&\!0.389\! \\
\!USvotes\!&\!217\!&\!218\!&\!0.032\!&\!\textit{(11)}\!&\!0.041\!&\!\textit{(598)}\!&\!0.032\!&\!\textit{(16)}\!&\!\textbf{0.028}\!&\!\textit{(1)}\!&\!0.046\! \\
\!Waveform\!&\!400\!&\!7600\!&\!\textbf{0.101}\!&\!\textit{(145)}\!&\!0.102\!&\!\textit{(178)}\!&\!0.106\!&\!\textit{(13)}\!&\!0.103\!&\!\textit{(22)}\!&\!0.115\! \\
\!Wdbc\!&\!284\!&\!285\!&\!0.049\!&\!\textit{(40)}\!&\!0.060\!&\!\textit{(19)}\!&\!0.091\!&\!\textit{(2)}\!&\!\textbf{0.046}\!&\!\textit{(10)}\!&\!0.060\! \\
\bottomrule
\end{tabular}

\vspace{5mm}

\rowcolors{3}{}{black!10}
\begin{tabular}{lcccc}
\toprule
\multicolumn{5}{c}{Statistical Comparison Tests} \\
\cmidrule(l r){1-5}
&  $\CSQ$ vs $R_S(B_Q)$ &  $\CSQ$ vs Validation Set &  $\CSQ$ vs Cross-Validation &  $\CSQ$ vs 1000 rounds \\
\cmidrule(l r){2-2} \cmidrule(l r){3-3} \cmidrule(l r){4-4} \cmidrule(l r){5-5}
Poisson binomial test  & 91\% & 86\% & 57\% & 90\%  \\
Sign test ($p$-value) & 0.05 & 0.23 & 0.60 & 0.02  \\
\bottomrule
\hline

\end{tabular}
\end{scriptsize}
\end{center}
\caption{Comparison of various stopping criteria over 1000 rounds of boosting. The Poisson binomial test gives the probability that $\CSQ$ is a better stopping criterion than every other approach. The sign test gives a $p$-value representing the probability that the null hypothesis is true (\ie, the $\CSQ$ stopping criterion has the same performance as every other approach).}
\label{tab:stopcriteria} 
\end{table}

\subsubsection{The $\Cbound$ as a Stopping Criterion for Boosting}

We now evaluate the accuracy of the empirical value of the $\Cbound$ as a model selection tool. More specifically, we compare its ability to act as a stopping criterion for the AdaBoost algorithm.

We use the same version of the algorithm and the same data sets as in the previous experiment. However, for this experiment, each data set is split into a training set $S$ of at most $400$ examples and a testing set $T$ containing the remaining examples. We run AdaBoost on set $S$ for 1000 rounds. At each round, we compute the empirical $\Cbound$ $\CSQ$ (on the training set). Afterwards, we select the majority vote classifier with the lowest value of $\CSQ$ and compute its Bayes risk $\RTBQ$ (on the test set). We compare this stopping criterion with three other methods. For the first method, we compute the empirical Bayes risk $\RSBQ$ at each round of boosting and, after that, we select the one having the lowest such risk.\footnote{When several iterations have the same value of $\RSBQ$, we select the earlier one.} The second method consists in performing 5-fold cross-validation and selecting the number of boosting rounds having the lowest cross-validation risk. Finally, the third method is to reserve 10\% of $S$ as a validation set, train AdaBoost on the remaining 90\%, and keep the majority vote with the lowest Bayes risk on the validation set. Note that this last method differs from the others because AdaBoost sees 10\% fewer examples during the learning process, but this is the price to pay for using a validation set.

Table~\ref{tab:stopcriteria} compares the Bayes risks on the test set $\RTBQ$ of the majority vote classifiers selected by the different stopping criteria.
We compute the probability of \Cbound being a better stopping criteria than every other methods with two statistical tests: the Poisson binomial test \citep{lacoste-2012} and the sign test~\citep{mendenhall1983nonparametric}.
Both statistical tests suggest that the empirical $\Cbound$ is a better model selection tool than the empirical Bayes risk
(as usual in machine learning tasks, this method is prone to overfitting)
and the validation set
(although this method performs very well sometimes, it suffers from the small quantity of training examples on several tasks). 
The empirical \Cbound and the cross-validation methods obtain a similar accuracy. However, the cross-validation procedure needs more running time.
We conclude that the empirical $\Cbound$ is a surprisingly good stopping criterion for Boosting.

\section{A PAC-Bayesian Story: From Zero to a PAC-Bayesian  \Cbound}
\label{section:PAC-Bayes}

In this section, we present a PAC-Bayesian theory that allows one to estimate the \Cbound value $\CDQ$ from its empirical estimate $\CSQ$. From there, we derive bounds on the risk of the majority vote $\RDBQ$ based on empirical observations.
  We first recall the classical PAC-Bayesian bound (here called the PAC-Bound~\ref{bound:classic}) that bounds the true Gibbs risk by its empirical counterpart. We then present two different PAC-Bayesian bounds on the majority vote classifier (respectively called PAC-Bounds~\ref{bound:variancebinouille} and~\ref{bound:trinouille}). A third bound, PAC-Bound~\ref{bound:variancebinouille-qu}, will be presented in Section~\ref{section:Further-PAC_Bayes}. 
  This analysis intends to be self-contained, and can act as an introduction to PAC-Bayesian theory.\footnote{We also recommend the ``practical prediction tutorial'' of~\citet{l-05}, that contains an insightful PAC-Bayesian introduction.}

\smallskip
The first PAC-Bayesian theorem was proposed by~\citet{m-99}.
Given a set of voters~$\Hcal$, a \emph{prior} distribution $P$ on $\Hcal$ chosen before observing the data, and a \emph{posterior} distribution~$Q$ on $\Hcal$ chosen after observing a training set $S\!\sim\! D^m$ ($Q$ is typically chosen by running a learning algorithm on~$S$), PAC-Bayesian theorems give tight risk bounds for the Gibbs classifier $G_Q$. These bounds on $\RDGQ$ usually rely on two quantities: 
\begin{enumerate}
\item[a)] The empirical Gibbs risk $\RSGQ$, that is computed on the $m$ examples of $S$,
\begin{equation*}
\RSGQ \ = \ \frac{1}{m}\sum_{i=1}^m \esp{f\sim Q} \linloss (f(x_i), y_i) \,.
\end{equation*}
\item[b)] The Kullback-Leibler divergence between distributions $Q$ and $P$, that measures ``how far'' the chosen posterior $Q$ is from the prior $P$,
\begin{equation} \label{eq:KLQP}
\KL(Q\|P) \ \eqdef \ \esp{f\sim Q}\ln\frac{Q(f)}{P(f)}\,.
\end{equation}
\end{enumerate}
 Note that the obtained PAC-Bayesian bounds are uniformly valid for all possible posteriors~$Q$.

\medskip
 In the following, we present a very general PAC-Bayesian theorem (Section~\ref{section:PB_real_losses}), and we specialize it to obtain a bound on the Gibbs risk $\RDGQ$ that is converted in a bound on the risk of the majority vote  $\RDBQ$ by the factor 2 of Proposition~\ref{thm:2RG_rediscover} (Section~\ref{section:PBzero}). Then, we define new losses that rely on a pair of voters (Section~\ref{section:generalization_to_esd}). These new losses allow us to extend the PAC-Bayesian theory to directly bound $\RDBQ$ through the $\Cbound$ (Sections~\ref{section:PBC1} and~\ref{section:PBC2}). For each proposed bound, we explain the algorithmic procedure required to compute its value.

\subsection{General PAC-Bayesian Theory for Real-Valued Losses}
\label{section:PB_real_losses}

A key step of most PAC-Bayesian proofs is summarized by the following \emph{Change of measure inequality} (Lemma~\ref{lem:change-measure}). 

We present here the same proof as in~\citet{seldin-tishby-10} and~\citet{mcallester-13}.
Note that the same result
is derived from Fenchel's inequality in~\citet{banerjee-06} and Donsker-Varadhan's variational formula for relative entropy in~\citet{seldin-12,ilya-13}.

\begin{lemma}[Change of measure inequality] \label{lem:change-measure}
For any set $\Hcal$, for any distributions $P$ and $Q$ on $\Hcal$, and for any measurable function $\phi:\Hcal­\to \Reals$,  we have
\begin{equation*}
\esp{f\sim Q} \phi(f) \ \leq \ \KL(Q\|P) + \ln \left( \esp{f\sim P} e^{\phi(f)} \right) \,.
\end{equation*}
\end{lemma}
\begin{proof}
The result is obtained by simple calculations, exploiting the definition of the KL-divergence given by Equation~\eqref{eq:KLQP}, and then Jensen's inequality (Lemma~\ref{lem:jensen}, in Appendix~\ref{section:appendix_auxmath}) on concave function~$\ln(\cdot)$ :
\begin{eqnarray*}
\esp{f\sim Q} \phi(f) 
& = & \esp{f\sim Q} \ln e^{\phi(f)}  
\ = \ \esp{f\sim Q} \ln \left(\frac{Q(f)}{P(f)} \cdot \frac{P(f)}{Q(f)} \cdot e^{\phi(f)} \right) \\
& = & \KL(Q\|P)  \, +\,  \esp{f\sim Q}  \ln \left(\frac{P(f)}{Q(f)} \cdot e^{\phi(f)} \right)
\\
&\leq & \KL(Q\|P)\,+\, \ln \left(\esp{f\sim Q} \frac{P(f)}{Q(f)} \cdot e^{\phi(f)} \right) \hspace{11mm} \mbox{(Jensen's inequality)}\hspace{-11mm}\\
& \leq& \KL(Q\|P) + \ln \left( \esp{f\sim P} e^{\phi(f)} \right). %
\end{eqnarray*}
Note that the last inequality becomes an equality if $Q$ and $P$ share the same support.
\end{proof}

Let us now present a general PAC-Bayesian theorem which bounds the expectation of any real-valued loss function $\loss: \Yover\times\Ycal\rightarrow[0,1]$. 
This theorem is slightly more general than the PAC-Bayesian theorem of~\citet[Theorem~2.1]{gllm-09}, that is specialized to the expected linear loss, and therefore gives rise to a bound of the ``generalized'' Gibbs risk of Definition~\ref{def:gibbsrisk}. A similar result is presented in~\citet[Lemma~1]{ilya-13}.

\begin{theorem}[General PAC-Bayesian theorem for real-valued losses] \label{thm:gen-pac-Bayes} 
For any distribution $D$ on $\Xcal\times\Ycal$, for any set $\Hcal$ of voters $\Xcal \rightarrow \Yover$, for any loss \mbox{$\loss : \Yover \times \Ycal  \rightarrow [0,1]$}, 
for
any prior distribution $P$ on $\Hcal$, for any $\dt\!\in\!
(0,1]$, for any $m'>0$, and for any convex function 
\mbox{$\Dcal : [0,1]\!\times\! [0,1]\rightarrow \Reals$}, 
we have
\begin{small}
\begin{equation*}
\prob{S\sim D^m}\!\!\LP \!\!
\begin{array}{l}
  \mbox{\small For all posteriors $Q$ on $\Hcal$}: \\
  \Dcal(\esp{f\sim Q} \!\ElossS(f), \esp{f\sim Q}\! \ElossD(f)) \le 
  \dfrac{1}{m'}\!\LB \KL(Q\|P) \!+\!
  \ln\!\LP\dfrac{1}{\dt}\esp{S\sim D^m}\!\esp{f\sim P}\!e^{\,m'\cdot\Dcal(\ElossS(f),\ElossD(f))}\!\RP\!\RB 
\end{array}
\!\!\!\! \RP \! \ge 1 -\dt\,,
\end{equation*}
\end{small}where $\KL(Q\|P)$ is the Kullback-Leibler divergence between $Q$ and $P$ of Equation~\eqref{eq:KLQP}.
\end{theorem}
Most of the time, this theorem is used with $m'=m$, the size of the training set. However, as pointed out by \citet{lls-10}, $m'$ does not have to be so. One can easily show that different values of $m'$ affect the relative weighting between the terms $\KL(Q\|P)$ and
$\ln \big(\frac{1}{\dt}\Eb_{S\sim D^m}\Eb_{f\sim P}e^{\,m'\cdot\Dcal(\ElossS(f),\ElossD(f))} \big)$ in the bound. Hence, especially in situations where these two terms have very different values, a ``good'' choice for the value of $m'$ can tighten the~bound. %
\begin{proof}
Note that $\esp{f\sim P} e^{m'\cdot\Dcal(\ElossS(f),\ElossD(f))}$ is a non-negative random variable. By Markov's inequality (Lemma~\ref{lem:markov}, in Appendix~\ref{section:appendix_auxmath}), we have
\begin{equation*}
 \prob{S\sim D^m}\LP
  \esp{f\sim P}e^{m'\cdot\Dcal(\ElossS(f),\ElossD(f))}
  \,\le\, 
 \frac{1}{\dt}\esp{S\sim D^m}\esp{f\sim P}e^{m'\cdot\Dcal(\ElossS(f),\ElossD(f))} \RP 
 \ge  1-\dt \,.
\end{equation*}
Hence, by taking the logarithm on each side of the
innermost inequality, we obtain
\begin{equation*}
 \prob{S\sim D^m}\LP
 \ln\LB\esp{f\sim P}e^{m'\cdot\Dcal(\ElossS(f),\ElossD(f))} \RB 
 \, \le \,
 \ln\LB\frac{1}{\dt}\esp{S\sim D^m}\esp{f\sim P}e^{m'\cdot\Dcal(\ElossS(f),\ElossD(f))}\RB \RP 
 \ge 1-\dt  \,.
\end{equation*}
We apply the change of measure inequality (Lemma~\ref{lem:change-measure}) on the left side of innermost inequality, with $\phi(f) = m'\cdot\Dcal(\ElossS(f),\ElossD(f))$.
We then use Jensen's inequality (Lemma~\ref{lem:jensen}, in Appendix~\ref{section:appendix_auxmath}), exploiting the convexity of~$\Dcal$ :
\begin{eqnarray*}
\forall \,Q \mbox{ on } \Hcal :\quad  \ln\LB\esp{f\sim P}e^{m'\cdot\Dcal(\ElossS(f),\ElossD(f))} \RB 
 &\ge&   m'\cdot\esp{f\sim Q}\Dcal(\ElossS(f),\ElossD(f)) -\KL(Q\|P) \\
 &\ge&  m'\!\cdot\Dcal(\esp{f\sim Q} \!\!\ElossS(f), \esp{f\sim Q}\!\! \ElossD(f)) -\KL(Q\|P)\,.
\end{eqnarray*}
 We therefore have
 {\small
\begin{equation*}
 \prob{S\sim D^m} \! \!\LP \! 
\begin{array}{l}
\mbox{For all posteriors} \ Q: \\
m'\!\cdot\Dcal(\esp{f\sim Q}\!\!\ElossS(f), \esp{f\sim Q}\!\! \ElossD(f))  -\KL(Q\|P)
 \! \le
 \ln \! \LB\frac{1}{\dt}\esp{S\sim D^m}\esp{f\sim P}e^{m'\cdot\Dcal(\ElossS(f),\ElossD(f))}\RB
\end{array}
 \RP 
 \! \ge \!  1-\dt  \, .
\end{equation*}
}The result then follows from easy calculations.
\end{proof}
As shown in~\citet{gllm-09}, the general PAC-Bayesian theorem can be used to recover many common variants of the PAC-Bayesian theorem, simply by selecting a well-suited function~$\Dcal$.  Among these, we obtain a similar bound as the one proposed by~\cite{ls-01-techreport,s-02,l-05} by using the Kullback-Leibler divergence between the
Bernoulli distributions with probability of success $q$ and
probability of success~$p$: 
\begin{equation} \label{eq:small_kl}
\kl\big(q\,\|\,p) \ \eqdef \ 
q \ln\frac{q}{p} + (1-q)\ln\frac{1-q}{1-p}\,.
\end{equation}
Note that $\kl\big(q\,\|\,p)$ is a shorthand notation for $\KL(Q\|P)$ of Equation~\eqref{eq:KLQP}, with $Q=(q,1\!-\!q)$ and $P = (p,1\!-\!p)$.
Corollary~\ref{cor:bintrin} (in Appendix~\ref{section:appendix_auxmath}) shows that $\kl\big(q\,\|\,p)$ is a convex function.
\medskip

\noindent
 In order to apply Theorem~\ref{thm:gen-pac-Bayes} with $\Dcal(q,p) = \kl(q\|p)$ and $m'=m$, we need the next lemma.
\begin{lemma} \label{lem:xi}
 For any distribution $D$ on $\Xcal\!\times\!\Ycal$, for any voter $f:\Xcal\rightarrow\Yover$, for any loss $\loss:\Yover\!\times\!\Ycal\rightarrow[0,1]$, and any positive integer $m$,  we have
 \begin{eqnarray*}
\esp{S\sim D^m} \exp\bigg[{m\cdot\kl\Big(\ElossS(f)\, \|\, \ElossD(f)\Big)} \bigg]
& \leq &  \xi(m)\,,
\end{eqnarray*}
where
 \begin{equation} \label{eq:xi}
 \xi(m) \ \, \eqdef \ \ \sum_{k=0}^m \binom{m}{k} \LP \frac{k}{m} \RP^{k} \LP 1 - \frac{k}{m} \RP^{m-k}.
\end{equation}
 Moreover, $\sqrt{m} \, \leq \, \xi(m) \, \leq \, 2\sqrt{m}$\,. %
  \end{lemma}

 \begin{proof}
 Let us introduce a random variable $X_f$ that follows a binomial distribution of $m$ trials with a probability of success $\ElossD(f)$. Hence, $X_f\sim B(m,\ElossD(f))$\,. \\

As $e^{m\cdot\kl\big(\cdot\, \|\, \ElossD(f)\big)}$ is a convex function, Lemma~\ref{lem:maurer} \citep[due to][and provided in Appendix~\ref{section:appendix_auxmath}]{m-04}, shows that
\begin{equation*}
\esp{S\sim D^m} \exp\bigg[{m\cdot\kl\Big(\ElossS(f)\, \|\, \ElossD(f)\Big)} \bigg]
\ \leq \ 
\esp{X_f\sim B(m,\ElossD(f))} \exp\bigg[{m\cdot\kl\Big(\tfrac{1}{m}  X_f \, \|\, \ElossD(f)\Big)} \bigg]\,.
\end{equation*}
We then have
\begin{eqnarray*}
  & & \hspace{-2cm}  
  \esp{X_f\sim B(m,\ElossD(f))}e^{m\kl(\frac{1}{m}  X_f \|\ElossD(f))} \\
 &=& \esp{X_f\sim B(m,\ElossD(f))}\LP\frac{\frac{1}{m} X_f}{\ElossD(f)}\RP^{X_f}\LP\frac{1-\frac{1}{m} X_f}{1-\ElossD(f)}\RP^{m- X_f}\\
  &=& \sum_{k=0}^m \prob{X_f\sim B(m,\ElossD(f))}\!\!\Big( X_f = k \Big) \cdot \LP\frac{\frac{k}{m}}{\ElossD(f)}\RP^{k} \LP\frac{1-\frac{k}{m}}{1-\ElossD(f)}\RP^{m- k}\\
  &=& \sum_{k=0}^m {m\choose k} \Big(\ElossD(f)\Big)^k \Big(1-\ElossD(f)\Big)^{m-k} \cdot \LP\frac{\frac{k}{m}}{\ElossD(f)}\RP^{k} \LP\frac{1-\frac{k}{m}}{1-\ElossD(f)}\RP^{m- k}\\
  &=&  
  \sum_{k=0}^m {m\choose k} \LP\frac{k}{m}\RP^k \LP 1-\frac{k}{m}\RP^{m-k}
   \ = \ \xi(m)\,.
\end{eqnarray*}
\cite{m-04} shows that $\xi(m)\leq 2\sqrt{m} $ for $m\geq 8$, and  $\xi(m) \geq \sqrt{m}$ for $m\geq 2$. However, the cases for $m\in\{1,2,3,4,5,6,7\}$ are easy to verify computationally.
\end{proof}
Theorem~\ref{thm:pac-bayes-kl-g} below specializes the general PAC-Bayesian theorem to $\Dcal(q,p) =
\kl(q\|p)$, but still applies to any real-valued loss functions.  
This theorem can be seen as an intermediate step to obtain Corollary~\ref{cor:pac-bayes-classic} of the next section, which uses the linear loss to bound the Gibbs risk.  However, Theorem~\ref{thm:pac-bayes-kl-g} below is reused afterwards in Section~\ref{section:generalization_to_esd} to derive PAC-Bayesian theorems for other loss functions.

\begin{theorem} \label{thm:pac-bayes-kl-g}
For any distribution $D$ on $\Xcal\!\times\!\Ycal$, for any set $\Hcal$ of voters $\Xcal \rightarrow \Yover$, for any loss $\loss : \Yover\times \Ycal \rightarrow [0,1]$, for
any prior distribution $P$ on $\Hcal$, for any $\dt\in
(0,1]$,  we have
\begin{equation*}
\prob{S\sim D^m}\left(\!\!
\begin{array}{l}
  \mbox{For all posteriors $Q$ on $\Hcal$}: \\[1mm]
   \kl\Big(\esp{f\sim Q}\!\ElossS(f) \,\Big\|\, \esp{f\sim Q}\! \ElossD(f)\Big) \,\le\,
    \dfrac{1}{m}\,\LB\KL(Q\|P) +
\ln\dfrac{\xi(m)}{\dt}\RB
\end{array}\!\!
\right) \ge \, 1 -\dt\,.
\end{equation*}
\end{theorem}
\begin{proof}
By Theorem~\ref{thm:gen-pac-Bayes}, with $\Dcal(q,p) = \kl(q\|p)$ and $m'=m$,  we have
{\small
\begin{equation*}
\prob{S\sim D^m}\!\!\LP \!\!\!\!
\begin{array}{l}
  \ \forall\,Q\,\,\text{on}\,\,\Hcal\colon \\
  \kl(\esp{f\sim Q} \!\ElossS(f)\,\|\, \esp{f\sim Q}\! \ElossD(f))\! \le \!
  \dfrac{1}{m}\!\LB \KL(Q\|P) \!+\!
  \ln\!\LP\dfrac{1}{\dt}\esp{S\sim D^m}\!\esp{f\sim P}\!\!\!e^{m\cdot\kl(\ElossS(f)\,\|\,\ElossD(f))}\!\RP\!\RB 
\end{array}
\!\!\! \RP \! \ge 1 -\dt\,.
\end{equation*}
}As the prior $P$ is independent of $S$, we can swap the two expectations in  
$\esp{S\sim D^m}\!\esp{f\sim P}\!\!\!e^{m\cdot\kl(\cdot\|\cdot)}$.  
This observation, together with Lemma~\ref{lem:xi}, gives
 \begin{equation*}
  \esp{S\sim D^m}\esp{f\sim P}\!\!\!e^{m\cdot \kl(\ElossS(f)\,\|\, \ElossD(f))} 
  \ = \ 
  \esp{f\sim P} \esp{S\sim D^m} \!\!\!e^{m\cdot \kl(\ElossS(f)\,\|\, \ElossD(f))} 
  \ \leq \  \esp{f\sim P}  \xi(m) \ = \ \xi(m)\,.
 \end{equation*}
\end{proof}

\subsection{PAC-Bayesian Theory for the Gibbs Classifier}
\label{section:PBzero}

This section presents two classical PAC-Bayesian results that bound the risk of the Gibbs classifier. One of these bounds is used to express a first PAC-Bayesian bound on the risk of the majority vote classifier. Then, we explain how to compute the empirical value of this bound by a root-finding method. 

\subsubsection{PAC-Bayesian Theorems for the Gibbs Risk}
We interpret the two following results as straightforward corollaries of Theorem~\ref{thm:pac-bayes-kl-g}.
Indeed, from Definition~\ref{def:gibbsrisk}, the expected linear loss of a Gibbs classifier $G_Q$ on a distribution $D'$ \emph{is} $\RGQ$. These two Corollaries are very similar to well-known PAC-Bayesian theorems.
At first, Corollary~\ref{cor:pac-bayes-classic} is similar to the PAC-Bayesian theorem of~\cite{ls-01-techreport,s-02, l-05}, with the exception that $\ln\frac{m+1}{\dt}$ is replaced by $\ln\frac{\xi(m)}{\dt}$.  Since $\xi(m) \leq 2\sqrt{m} \leq m+1$, this result gives slightly better bounds. 
Similarly, Corollary~\ref{cor:pac-bayes-McAllester} provides a slight improvement of the PAC-Bayesian bound of~\citet{m-99,m-03}.
\begin{corollary} {\rm \citep{ls-01-techreport,s-02, l-05}} \label{cor:pac-bayes-classic}
For any distribution $D$ on \mbox{$\Xcal\!\times\!\{-1,1\}$}, for any set $\Hcal$ of voters $\Xcal \rightarrow [-1,1]$, for any prior distribution $P$ on $\Hcal$, and any $\dt\in(0,1]$, we have
\begin{equation*}
\prob{S\sim D^m}\left(
\begin{array}{l}\!\!
  \mbox{For all posteriors $Q$ on $\Hcal$}: \\ [.5mm]
   \kl\big(R_S(G_Q) \big\| R_D(G_Q)\big) \, \le\, 
    \dfrac{1}{m}\LB\KL(Q\|P) +
\ln\dfrac{\xi(m)}{\dt}\RB
\end{array}\!\!
\right) \ge \, 1 -\dt\,.
\end{equation*}
\end{corollary} 
\begin{proof}
 The result is directly obtained from Theorem~\ref{thm:pac-bayes-kl-g} using the linear loss $\loss = \linloss$ to recover the Gibbs risk of Definition~\ref{def:gibbsrisk}.
 \end{proof}

\begin{corollary}{\rm \citep{m-99,m-03}} \label{cor:pac-bayes-McAllester}
For any distribution $D$ on \mbox{$\Xcal\!\times\!\{-1,1\}$}, for any set $\Hcal$ of voters $\Xcal \rightarrow [-1,1]$, for any prior distribution $P$ on $\Hcal$, and any $\dt\in(0,1]$, we have
\begin{equation*}
\prob{S\sim D^m}\left(\!\!
\begin{array}{l}
  \mbox{For all posteriors $Q$ on $\Hcal$}: \\[1mm]
    R_D(G_Q) \,\le\,
    R_S(G_Q) + \sqrt{ \dfrac{1}{2m}\LB\KL(Q\|P) +
\ln\dfrac{\xi(m)}{\dt}\RB }
\end{array}\!\!
\right) \ge \,1 -\dt\,.
\end{equation*}
\end{corollary}
\begin{proof}
 The result is obtained from Corollary~\ref{cor:pac-bayes-classic} together with Pinsker's inequality 
 \begin{equation*}
 2(q-p)^2 \ \leq \ \kl(q\|p)\,.\\[-3mm]
 \end{equation*}
 We then have
 {\small
 \begin{equation*}
 \prob{S\sim D^m}\left(
 \begin{array}{l}\!\!
   \mbox{For all posteriors $Q$ on $\Hcal$}: \\ [.5mm]
    2\!\cdot\!\Big(R_S(G_Q) - R_D(G_Q)\Big)^2 \, \le\, 
     \dfrac{1}{m}\LB\KL(Q\|P) +
 \ln\dfrac{\xi(m)}{\dt}\RB
 \end{array}\!\!
 \right) \ge \, 1 -\dt\,.
 \end{equation*}
 }The result is obtained by isolating $\RDGQ$ in the inequality, omitting the lower bound of $\RDGQ$. Recall that the probability is ``\,$\geq 1\!-\!\delta$\,'', hence if we omit an event, the probability may just increase, continuing to be greater than $1\!-\!\delta$.
\end{proof}

\subsubsection{A First Bound for the Risk of the Majority Vote}
\label{section:classic-bound-computation}

Let assume that the Gibbs risk $\RDGQ$ of a classifier is lower than or equal to $\frac{1}{2}$.
Given an empirical Gibbs risk $\RSGQ$ computed on a training set of $m$ examples, the Kullback-Leibler divergence $\KL(Q\|P)$, and a confidence parameter $\delta$, Corollary~\ref{cor:pac-bayes-classic} says that the Gibbs risk $\RDGQ$ is included (with confidence $1\!-\!\delta$) in the continuous set $\Rdel$ defined as
\begin{equation} \label{eq:setGibbsRisk}
  \Rdel \ \eqdef \  \LC r\,\,
\colon \,\, \kl\big(R_S(G_Q)\,\big\|\,r\big)\,\le\,
    \frac{1}{m}\,\Big[\,  \KL(Q\|P) +
\ln\frac{\xi(m)}{\dt}\Big] \ \ \ \mbox{and} \ \ \ r\le\tfrac{1}{2}\RC\,.
\end{equation}
Thus, an upper bound on $\RDGQ$ is obtained by seeking the maximum value of $\Rdel$. As explained by Proposition~\ref{thm:2RG_rediscover}, we need to multiply the obtained value by a factor 2 to have an upper bound on $\RDBQ$. This methodology is summarized by PAC-Bound~\ref{bound:classic}.

Note that PAC-Bound~\ref{bound:classic} is also valid when $\RDGQ$ is greater than $\frac{1}{2}$, because in this case, $2 \cdot \sup\Rdel = 1$ (with confidence at least $1\!-\!\delta$), which is a trivial upper bound of~$\RDBQ$.

\begin{pacbound}{0}\label{bound:classic}
For any distribution $D$ on \mbox{$\Xcal\!\times\!\{-1,1\}$}, for any set $\Hcal$ of voters \mbox{$\Xcal \rightarrow [-1,1]$}, for any prior distribution $P$ on $\Hcal$, and any $\dt\in (0,1]$, we have

\begin{equation*}
 \prob{S\sim D^m}\biggl(
\forall\,Q\mbox{ on }\Hcal\,:\ 
  \RDBQ \ \leq \ 2 \cdot \sup\Rdel
\biggl)\ \ge\ 1 -\dt\,. 
 \end{equation*}
\end{pacbound}
\begin{proof}
 If $ \sup\Rdel=\frac{1}{2}$, the bound is trivially valid because $\RDBQ\leq 1$. Otherwise, the bound is a direct consequence of Proposition~\ref{thm:2RG_rediscover} and Corollary~\ref{cor:pac-bayes-classic}.
\end{proof}
 As we see, the proposed bound cannot be obtained by a closed-form expression.  Thus, we need to use a strategy as the one suggested in the following.
 
\subsubsection{Computation of PAC-Bound~\ref{bound:classic}}

One can compute the value $r\!=\!\sup\Rdel$ of PAC-Bound~\ref{bound:classic} by solving
$$\kl\big(R_S(G_Q)\,\big\|\,r\big) \ =\ \tfrac{1}{m}\big[\KL(Q\|P) + \ln\tfrac{\xi(m)}{\dt}\big]\,, \qquad\mbox{with $\RSGQ\leq r\leq\frac{1}{2}$\,,}$$
by a root-finding method.  This turns out to be an easy task since the left-hand side of the equality is a convex function of $r$ and the right-hand side is a constant value.  Note that solving the same equation with the constraint $r\leq\RSGQ$ gives a lower bound of $\RDGQ$, but not a lower bound on $\RDBQ$. Figure~\ref{figGraphPacBayes} shows an application example of PAC-Bound~\ref{bound:classic}.

\begin{figure}[t]
\centering{
\begingroup
  \makeatletter
  \providecommand\color[2][]{%
    \GenericError{(gnuplot) \space\space\space\@spaces}{%
      Package color not loaded in conjunction with
      terminal option `colourtext'%
    }{See the gnuplot documentation for explanation.%
    }{Either use 'blacktext' in gnuplot or load the package
      color.sty in LaTeX.}%
    \renewcommand\color[2][]{}%
  }%
  \providecommand\includegraphics[2][]{%
    \GenericError{(gnuplot) \space\space\space\@spaces}{%
      Package graphicx or graphics not loaded%
    }{See the gnuplot documentation for explanation.%
    }{The gnuplot epslatex terminal needs graphicx.sty or graphics.sty.}%
    \renewcommand\includegraphics[2][]{}%
  }%
  \providecommand\rotatebox[2]{#2}%
  \@ifundefined{ifGPcolor}{%
    \newif\ifGPcolor
    \GPcolorfalse
  }{}%
  \@ifundefined{ifGPblacktext}{%
    \newif\ifGPblacktext
    \GPblacktexttrue
  }{}%
  \let\gplgaddtomacro\g@addto@macro
  \gdef\gplbacktext{}%
  \gdef\gplfronttext{}%
  \makeatother
  \ifGPblacktext
    \def\colorrgb#1{}%
    \def\colorgray#1{}%
  \else
    \ifGPcolor
      \def\colorrgb#1{\color[rgb]{#1}}%
      \def\colorgray#1{\color[gray]{#1}}%
      \expandafter\def\csname LTw\endcsname{\color{white}}%
      \expandafter\def\csname LTb\endcsname{\color{black}}%
      \expandafter\def\csname LTa\endcsname{\color{black}}%
      \expandafter\def\csname LT0\endcsname{\color[rgb]{1,0,0}}%
      \expandafter\def\csname LT1\endcsname{\color[rgb]{0,1,0}}%
      \expandafter\def\csname LT2\endcsname{\color[rgb]{0,0,1}}%
      \expandafter\def\csname LT3\endcsname{\color[rgb]{1,0,1}}%
      \expandafter\def\csname LT4\endcsname{\color[rgb]{0,1,1}}%
      \expandafter\def\csname LT5\endcsname{\color[rgb]{1,1,0}}%
      \expandafter\def\csname LT6\endcsname{\color[rgb]{0,0,0}}%
      \expandafter\def\csname LT7\endcsname{\color[rgb]{1,0.3,0}}%
      \expandafter\def\csname LT8\endcsname{\color[rgb]{0.5,0.5,0.5}}%
    \else
      \def\colorrgb#1{\color{black}}%
      \def\colorgray#1{\color[gray]{#1}}%
      \expandafter\def\csname LTw\endcsname{\color{white}}%
      \expandafter\def\csname LTb\endcsname{\color{black}}%
      \expandafter\def\csname LTa\endcsname{\color{black}}%
      \expandafter\def\csname LT0\endcsname{\color{black}}%
      \expandafter\def\csname LT1\endcsname{\color{black}}%
      \expandafter\def\csname LT2\endcsname{\color{black}}%
      \expandafter\def\csname LT3\endcsname{\color{black}}%
      \expandafter\def\csname LT4\endcsname{\color{black}}%
      \expandafter\def\csname LT5\endcsname{\color{black}}%
      \expandafter\def\csname LT6\endcsname{\color{black}}%
      \expandafter\def\csname LT7\endcsname{\color{black}}%
      \expandafter\def\csname LT8\endcsname{\color{black}}%
    \fi
  \fi
  \setlength{\unitlength}{0.0500bp}%
  \begin{picture}(8640.00,4032.00)%
    \gplgaddtomacro\gplbacktext{%
      \csname LTb\endcsname%
      \put(726,465){\makebox(0,0)[r]{\strut{} 0}}%
      \put(726,1125){\makebox(0,0)[r]{\strut{} 0.1}}%
      \put(726,1786){\makebox(0,0)[r]{\strut{} 0.2}}%
      \put(726,2446){\makebox(0,0)[r]{\strut{} 0.3}}%
      \put(726,3107){\makebox(0,0)[r]{\strut{} 0.4}}%
      \put(726,3767){\makebox(0,0)[r]{\strut{} 0.5}}%
      \put(883,220){\makebox(0,0){\strut{} 0}}%
      \put(1701,220){\makebox(0,0){\strut{} 0.1}}%
      \put(2518,220){\makebox(0,0){\strut{} 0.2}}%
      \put(3336,220){\makebox(0,0){\strut{}$R_S(G_Q)$}}%
      \put(4154,220){\makebox(0,0){\strut{} 0.4}}%
      \put(4971,220){\makebox(0,0){\strut{} 0.5}}%
      \put(5789,220){\makebox(0,0){\strut{} 0.6}}%
      \put(6607,220){\makebox(0,0){\strut{} 0.7}}%
      \put(7424,220){\makebox(0,0){\strut{} 0.8}}%
      \put(8242,220){\makebox(0,0){\strut{}$r$}}%
    }%
    \gplgaddtomacro\gplfronttext{%
      \csname LTb\endcsname%
      \put(2991,3602){\makebox(0,0)[l]{\strut{}$\kl(R_S(G_Q)\|r)$}}%
      \csname LTb\endcsname%
      \put(2991,3272){\makebox(0,0)[l]{\strut{}$\frac{1}{m}\,\big[\,\KL(Q\|P) + \ln\frac{\xi(m)}{\dt}\big] \approx 0.0117$}}%
    }%
    \gplbacktext
    \put(0,0){\includegraphics{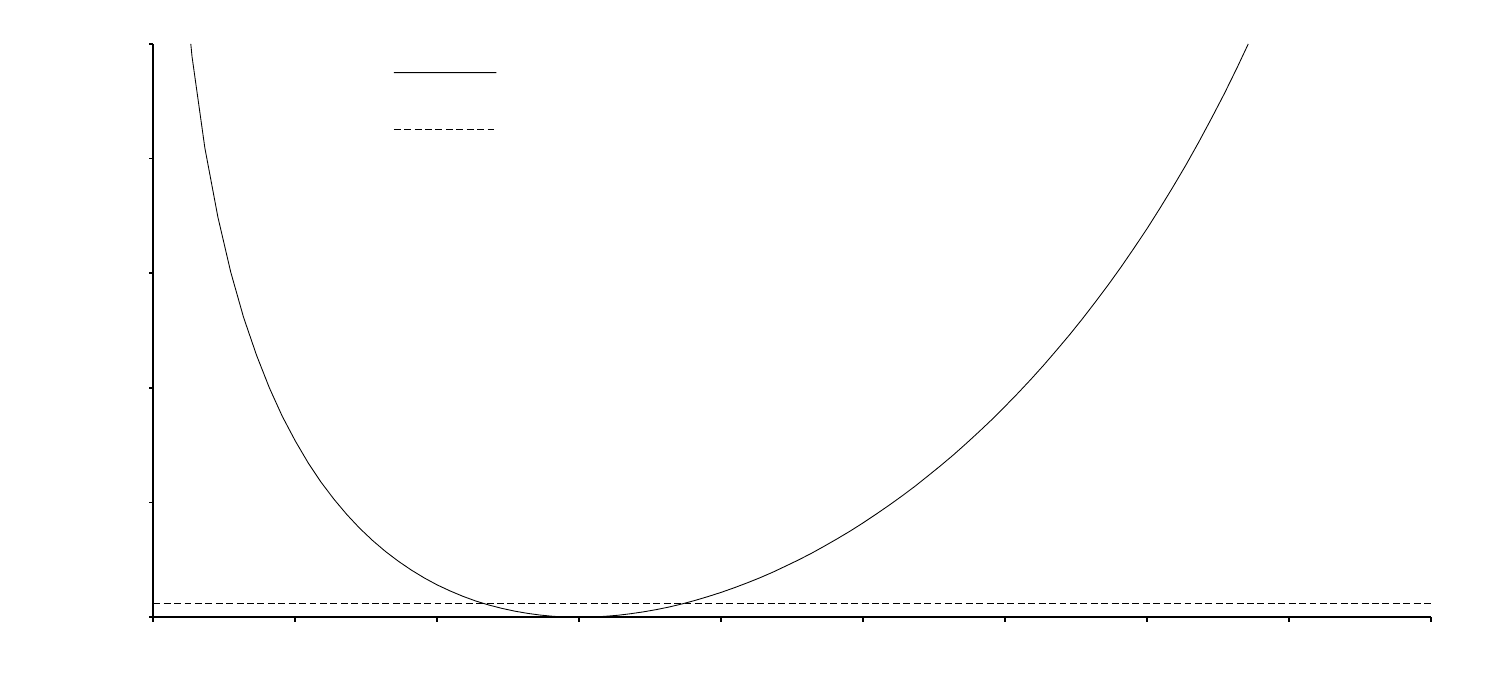}}%
    \gplfronttext
  \end{picture}%
\endgroup
 }
 \caption{Example of application of PAC-Bound~\ref{bound:classic}. We suppose that $\KL(Q\|P)=5$, $m=1000$ and $\delta=0.05$. 
 If we observe an empirical Gibbs risk $\RSGQ = 0.30$, then $\RDGQ\in\Rdel \approx [0.233, 0.373]$ with a confidence of $95\%$. On the figure, the intersections between the two curves correspond to the limits of the interval $\Rdel$. 
 Then, with these values, PAC-bound~\ref{bound:classic} gives $\RDBQ\lesssim 2\cdot 0.373 = 0.746$.}
 \label{figGraphPacBayes}
\end{figure}

\subsection{Joint Error, Joint Success, and Paired-voters}
\label{section:generalization_to_esd}

We now introduce a few notions that are necessary to obtain new PAC-Bayesian theorems for the \Cbound in Sections~\ref{section:PBC1} and~\ref{section:PBC2}.

\subsubsection{The Joint Error and the Joint Success}
We have already defined the expected disagreement $\dQ$ of a distribution $Q$ of voters (Definition~\ref{def:disagreement}). In the case of binary voters, the expected disagreement corresponds to
 \begin{eqnarray*}
   \dQ &=& \esp{h_1\sim Q}\esp{h_2\sim Q} \Big( \esp{(\xb,y)\sim{D'}}I(h_1(\xb)\ne h_2(\xb)) \Big)\,.
 \end{eqnarray*}
 Let us now define two closely related notions, the expected joint success $\sQ$ and the expected joint error $\eQ$.  In the case of binary voters, these two concepts are expressed naturally by
 \begin{eqnarray*}
   \eQ &=& \esp{h_1\sim Q}\esp{h_2\sim Q} \Big( \esp{(\xb,y)\sim{D'}}I(h_1(\xb)\ne y) I(h_2(\xb)\ne y) \Big)\,,\\
   \sQ &=& \esp{h_1\sim Q}\esp{h_2\sim Q} \Big( \esp{(\xb,y)\sim{D'}}I(h_1(\xb)= y) I(h_2(\xb)= y) \Big)\,.
 \end{eqnarray*}
Let us now extend in the usual way these equations to the case of real-valued voters.
 \begin{definition}\label{def:esd}\rm
 For any probability distribution $Q$ on a set of voters, we  define the \emph{expected joint error}~$\eQ$ relative to $D'$ and
 the \emph{expected joint success}~$\sQ$ relative to $D'$ as
 \begin{eqnarray*}
   \eQ &\eqdef& \esp{f_1\sim Q}\esp{f_2\sim Q} \Big( \esp{(\xb,y)\sim{D'}} \linloss(f_1(\xb),y) \cdot \linloss(f_2(\xb),y) \Big)\,,\\
   \sQ &\eqdef& \esp{f_1\sim Q}\esp{f_2\sim Q} \Big( \esp{(\xb,y)\sim{D'}} \Big[1-\linloss(f_1(\xb),y) \Big] \cdot \Big[1-\linloss(f_2(\xb),y) \Big]\Big)\,.
 \end{eqnarray*}
\end{definition}

From the definitions of the linear loss (Definition~\ref{def:linearloss}) and the margin (Definition~\ref{def:margin}), we can easily see that
\begin{eqnarray*}\label{eq:esdWQ}
    \eQ &=& \esp{(\xb,y)\sim {D'}} \left(\frac{1-M_Q(\xb,y)}{2}\right)^2
    \ = \ \, \frac{1}{4}\Big(1-2\cdot\momentone(\MQ{D'}) + \momenttwo(\MQ{D'}) \Big)\,,\\
    \sQ &=& \esp{(\xb,y)\sim {D'}} \left(\frac{1+M_Q(\xb,y)}{2}\right)^2
    \ = \ \, \frac{1}{4}\Big(1+2\cdot\momentone(\MQ{D'}) + \momenttwo(\MQ{D'}) \Big)\,.%
\end{eqnarray*}
Remembering from Equation~\eqref{eq:dQ_Mq} that $\dQ = \frac{1}{2}\left( 1- \momenttwo(\MQ{D'}) \right)$, we can conclude that $\eQ$, $\sQ$ and $\dQ$ always sum to one:\footnote{This is fairly intuitive in the case of binary voters.
Indeed, given any example $(\xb,y)$ and any two binary voters $h_1, h_2$, we have either: both voters misclassify the example -- \ie, $h_1(\xb)=h_2(\xb)\neq y$ --, both voters correctly classify the example -- \ie, $h_1(\xb)=h_2(\xb)=y$ --, or both voters disagree -- \ie,  $h_1(\xb)\neq h_2(\xb)$.
}
\begin{equation*}
 \eQ \,+\, \sQ \,+\, \dQ \ =\ 1\,.
\end{equation*}
We can now rewrite the first moment of the margin and the Gibbs risk as
\begin{eqnarray}
\nonumber
\label{eq:mu_vs_e_d}
\momentone(\MQ{D'}) &=& \sQ - \eQ \ = \ 1 - (2\eQ+\dQ)\,,\\[1mm]
\label{eq:Gibbs_vs_e_d}
 R_{D'}(G_Q) &=& \tfrac{1}{2} \,(1-\sQ+\eQ) \ = \ \tfrac{1}{2}\,(2\eQ + \dQ)\,. 
 \end{eqnarray}
Therefore, the third form of \Cbound of Theorem~\ref{thm:C-bound} can be rewritten as
\begin{eqnarray}
 \CQ &=& 1- \frac{\left(1 - (2\eQ+\dQ)\right)^2}{1-2\dQ}\,. \label{eq:Cq_ed}
\end{eqnarray}

\subsubsection{Paired-Voters and Their Losses}

This first generalization of the PAC-Bayesian theorem allows us to bound \emph{separately} either\, $\dDQ$, $\eDQ$ or $\sDQ$, and therefore to bound $\CDQ$. 
To prove this result, we need to define a new kind of voter that we call a paired-voter.

\begin{definition}\label{def:paired_voters}\rm
Given two voters $f_i : \Xcal\rightarrow[-1,1]$ and $f_j : \Xcal\rightarrow[-1,1]$, the \emph{paired-voter} $\fpaired : \Xcal \rightarrow [-1,1]^2$ outputs a tuple:
\begin{equation*}
\fpaired  (\xb) \ \eqdef \ \tuple{f_i(\xb),f_j(\xb)}  %
\,.
\end{equation*}
Given a set of voters $\Hcal$ weighted by a distribution $Q$ on $\Hcal$, we define a set of paired-voters~$\Hcal^2$ weighted by a distribution $Q^2$ as 
\begin{equation}\label{eq:H2Q2}
\Hcal^2 \ \eqdef\ \{\fpaired \, :\, f_i, f_j \in \Hcal\}\,,
\quad\mbox{and }\quad
Q^{2}(\fpaired)\ \eqdef\ Q(f_i)\cdot Q(f_j)\,.
\end{equation} 
\end{definition}

We now present three losses for paired-voters.  
Remember that a loss function has the form $\Yover\!\times\!\Ycal \rightarrow [0,1]$, where $\Yover$ is the voter's output space.
As a paired-voter output is a tuple, our new loss functions map $[-1,1]^2\times\{-1,1\}$ to $[0,1]$. Thus,
\begin{eqnarray}
  \nonumber \loss_e \big( \fpaired(\xb), \,y\big) &\eqdef&  \linloss(f_i(\xb),y) \cdot \linloss(f_j(\xb),y)\,, \\
  \nonumber \loss_s \big(\fpaired(\xb), \,y\big) &\eqdef& \Big[1-\linloss(f_i(\xb),y) \Big] \cdot \Big[1-\linloss(f_j(\xb),y) \Big]\,,\\
      \loss_d \big( \fpaired(\xb), \,y\big) &\eqdef& \linloss(f_i(\xb)\!\cdot\!f_j(\xb)\,,\,1)\,. \label{eqn:lossd}
\end{eqnarray}

The key observation to understand the next theorems is that the expected losses of paired-voters $\Hcal^2$ defined by Equation~\eqref{eq:H2Q2} allow one to recover the values of $\eQ$, $\sQ$ and $\dQ$. 
Indeed, it directly follows from Definitions~\ref{def:expected_loss}, \ref{def:disagreement} and~\ref{def:esd}, that
\begin{equation} \label{eq:esd_from_losses}
\eQ \, = \esp{\fpaired \sim Q^2}\!\! \Eaggloss{e}{D'} \Big( \fpaired  \Big)\,; \quad
\sQ \, = \esp{\fpaired \sim Q^2}\!\! \Eaggloss{s}{D'} \Big( \fpaired  \Big)\,; \quad 
\dQ \, = \esp{\fpaired \sim Q^2}\!\! \Eaggloss{d}{D'} \Big( \fpaired  \Big) \, .
\end{equation}

\subsection{PAC-Bayesian Theory For Losses of Paired-voters}
\label{section:PBC1}

As explained in Section~\ref{section:PBzero}, classical PAC-Bayesian theorems, like Corollaries~\ref{cor:pac-bayes-classic} and~\ref{cor:pac-bayes-McAllester},  provide an upper bound on $\RDGQ$ that holds uniformly for all posteriors~$Q$. A bound on $\RDBQ$ is typically obtained by multiplying the former bound by the usual factor of $2$, as in PAC-Bound~\ref{bound:classic}. 

In this subsection, we present a first bound of $\RDBQ$ relying on the \Cbound of Theorem~\ref{thm:C-bound}.
A uniform bound on $\CDQ$ is obtained using the third form of the \Cbound, through a bound on the Gibbs risk $\RDGQ$ and another bound on the disagreement~$\dDQ$. 
The desired bound on $\RDGQ$ is obtained by Corollary~\ref{cor:pac-bayes-classic} as in PAC-Bound~\ref{bound:classic}.
To obtain a bound on $\dDQ$, we capitalize on the notion of paired-voters presented in the previous section. This allows us to express two new PAC-Bayesian bounds on the risk of a majority vote, one for the supervised case and another for the semi-supervised case.

\subsubsection{A PAC-Bayesian Theorem for $\eDQ$, $\sDQ$, or $\dDQ$}

The following PAC-Bayesian theorem can either bound the expected disagreement~$\dDQ$, the expected joint success~$\sDQ$ or the expected joint error~$\eDQ$ of a majority vote (see Definitions~\ref{def:disagreement} and~\ref{def:esd}).

\begin{theorem}
\label{thm:binouille} 
For any distribution $D$ on \mbox{$\Xcal\!\times\!\{-1,1\}$}, for any set $\Hcal$ of voters \mbox{$\Xcal \rightarrow [-1,1]$}, for any prior distribution $P$ on $\Hcal$, and any $\dt\in
(0,1]$, we have
\begin{equation*}
\prob{S\sim D^m}\!\LP\!\!
\begin{array}{l}
  \mbox{For all posteriors $Q$ on $\Hcal$}: \\[1mm]
   \kl\big(\aSQ\,\big\|\,\aDQ\big) \, \le\, 
    \dfrac{1}{m}\LB 2\!\cdot\!\KL(Q\|P) +
\ln\dfrac{\xi(m)}{\dt}\RB
\end{array}
 \!\!\RP \ge \, 1 -\dt\,,
\end{equation*}
where $\aQ$  can be either $\eQ$, $\sQ$ or $\dQ$.
\end{theorem}
\begin{proof}
Theorem~\ref{thm:binouille} is deduced from
Theorem~\ref{thm:pac-bayes-kl-g}. We present here the proof for $\aQ = \eQ$. The two other cases are very similar.\\[2mm]
Consider the set of paired-voters $\Hcal^2$ and the posterior distribution $Q^2$ of Equation~\eqref{eq:H2Q2}. Also consider the  prior distribution $P^2$ on $\Hcal^2$ such that\ \, $P^2(\fpaired )\eqdef P(f_i)\cdot P(f_j) \,.$ \ Then we have,
{\small
\begin{eqnarray*}
\KL(Q^2\| P^2) &=&\esp{\fpaired \sim
Q^2}\ln\frac{Q^2(\fpaired )}{P^2(\fpaired )} \ = \
\esp{\fpaired \sim Q^2}\ln\frac{Q(f_i) \cdot Q(f_j)}{P(f_i)\cdot P(f_j)}\\[2mm]
&=&\esp{\fpaired \sim Q^2}\left[ \ln\frac{Q(f_i)}{P(f_i)} + \ln
\frac{Q(f_j)}{P(f_j)} \right ]\\[3mm]
&=& 2\cdot \KL(Q\| P)\,.
\end{eqnarray*}
}Finally, from Equation~\eqref{eq:esd_from_losses}, we have 
$\esp{\fpaired \sim Q^2}\!\! \Eaggloss{e}{D} \Big( \fpaired  \Big) = \eDQ$\ \ and 
$\esp{\fpaired \sim Q^2}\!\! \Eaggloss{e}{S} \Big( \fpaired  \Big) = \eSQ$\,.
Hence, by applying Theorem~\ref{thm:pac-bayes-kl-g}, we are done.
\end{proof}

\subsubsection{A New Bound for the Risk of the Majority Vote}

Based on the fact that Theorem~\ref{thm:binouille} gives a lower bound on the expected disagreement $\dDQ$, we now
derive PAC-Bound~\ref{bound:variancebinouille}, which is a PAC-Bayesian bound for the \Cbound, and
therefore, for the risk of  the majority vote.

Given any prior distribution $P$ on~$\Hcal$, we need the interval $\Rdel$ of Equation~\eqref{eq:setGibbsRisk}, together with 
\begin{eqnarray}\label{eq:setdQ}
\Ddel &\eqdef&  \bigg\{ d\ \colon \ \kl(\dSQ\|\,d) \, \le \, 
\frac{1}{m}\,\bigg[\, 2\!\cdot\! \KL(Q\|P) + \ln\frac{\xi(m)}{\dt}\bigg]\bigg\}\,.
\end{eqnarray}

\noindent
We then express the following bound on the Bayes risk.
\begin{pacbound}{1}\label{bound:variancebinouille}
For any distribution $D$ on \mbox{$\Xcal\!\times\!\{-1,1\}$}, for any set $\Hcal$ of voters \mbox{$\Xcal \rightarrow [-1,1]$}, for any prior distribution $P$ on $\Hcal$, and any $\dt\in (0,1]$,  we have
\begin{equation*}
  \prob{S\sim D^m}\!\left(
\forall\,Q\mbox{ on }\Hcal\,:\ 
   R_D(B_Q) \ \leq \ 1-\frac{\Big( 1- 2\cdot\sup\Rdeld \Big)^2}{ 1 - 2\cdot \inf\Ddeld }
\right) \ge\ 1 -\dt\,, %
\end{equation*}
where $\Rdeld$ and  $\Ddeld $ are respectively defined by Equations~\eqref{eq:setGibbsRisk} and~\eqref{eq:setdQ}.
\end{pacbound}
\begin{proof}
By Proposition~\ref{prop:dQ_bound}, we have that $\dSQ\le \frac{1}{2}$. This, together with the facts that $m$ is finite and $\dSQ\in\Ddel$, implies that
 \mbox{$\inf\Ddeld<\frac{1}{2}$}, and therefore that the denominator of the fraction in the statement of PAC-Bound~\ref{bound:variancebinouille} is always strictly positive. 

\medskip
\noindent
Necessarily, $\sup\Rdeld\leq\frac{1}{2}$.   Let us consider the two following cases.\\[2mm]
{\it Case 1: }
$\sup\Rdeld=\frac{1}{2}$. Then, $1- 2\cdot\sup\Rdeld=0$, and the bound on $R_D(B_Q)$ is $1$, which is trivially valid.\\[2mm]
{\it Case 2: }
$\sup\Rdeld<\frac{1}{2}$. Then, we can apply the third form of
Theorem~\ref{thm:C-bound} to obtain the upper bound on $\RDBQ$. The desired bound is obtained by replacing  $\dDQ$ by its lower bound $\inf\Ddeld$, 
and  $R_D(G_Q)$, by its upper bound $\sup\Rdeld $.
The two bounds can therefore be deduced by suitably applying
Corollary~\ref{cor:pac-bayes-classic} (replacing $\delta$ by $\delta/2$) and
Theorem~\ref{thm:binouille} (replacing $\aSQ$ by $\dSQ$, $\aDQ$ by $\dDQ$  and $\delta$
by $\delta/2$).
\end{proof}

This bound has a major inconvenience: it degrades rapidly if the bounds on the numerator and the denominator are not tight.
Note however that in the semi-supervised framework, we can achieve tighter results because the labels of the examples do not affect the value of $\dQ$ (see Definition~\ref{def:disagreement}). Indeed, it is generally assumed in this framework that the learner has access to a huge amount $m'$ of unlabeled data (\ie, $m' \gg m$).  One can then obtain a tighter bound of the disagreement. In this context, PAC-Bound~\ref{bound:semisup} stated below is tighter than PAC-Bound~\ref{bound:variancebinouille}.

\begin{pacbound}{1'}[Semi-supervised bound]
\label{bound:semisup}
For any distribution $D$ on \mbox{$\Xcal\!\times\!\{-1,1\}$}, for any set $\Hcal$ of voters \mbox{$\Xcal \rightarrow [-1,1]$}, for any prior distribution $P$ on $\Hcal$, and any $\dt\in (0,1]$,  we have
\begin{equation*}
\prob{
      \begin{array}{r}
      \scriptstyle\phantom{aiiiiiiiiiii} S\sim D_{\phantom{unlabeled}}^m \\
      \scriptstyle\phantom{aiiiiiiiiiii} S_{\cal U}\sim D_{unlabeled}^{m'}
      \end{array}}\hspace{-10mm}
      \left(
\forall\,Q\mbox{ on }\Hcal\,:\ 
   R_D(B_Q) \ \leq \ 1-\frac{\Big( 1- 2\cdot\sup\Rdeld \Big)^2}{ 1 - 2\cdot \inf\Ddelpd}
\right) \ge\ 1 -\dt\,.
\end{equation*}
\end{pacbound}
\begin{proof}
In the presence of a large amount of unlabeled data (denoted by the set $S_{\cal U}$), one can use
Corollary~\ref{thm:binouille} to obtain an accurate lower
bound of $\dDQ$. An upper bound of $\RDGQ$ can also be obtained via
Corollary~\ref{cor:pac-bayes-classic} but, this time, on the labeled data $S$.
Thus, similarly as in the proof of PAC-Bound~\ref{bound:variancebinouille},  the result follows from Theorem~\ref{thm:C-bound}.
\end{proof}

\subsubsection{Computation of PAC-Bounds~\ref{bound:variancebinouille} and~\ref{bound:semisup}}

To compute PAC-Bound~\ref{bound:variancebinouille}, we obtain the values of $r=\sup\Rdeld$ and $d=\inf\Ddeld$ by solving 
\begin{eqnarray*}
\begin{array}{llcll}
&\kl\big(R_S(G_Q)\,\big\|\,r\big) &=& \frac{1}{m}\big[\KL(Q\|P) + \ln\frac{\xi(m)}{\dt/2}\big]\,, & \mbox{with $\RSGQ \leq r \leq \frac{1}{2}$}\,,\\[1mm]
&\mbox{and }\ \kl\big(\dSQ\,\big\|\,d\big) &=& \frac{1}{m}\big[2\cdot\KL(Q\|P) + \ln\frac{\xi(m)}{\dt/2}\big]\,, & \mbox{with $d \leq \dSQ$}\,.
\end{array}
\end{eqnarray*}
These equations are very similar to the one we solved to compute PAC-Bound~\ref{bound:classic}, as described in Section~\ref{section:classic-bound-computation}. Once $r$ and $d$ are computed, the bound on $\RDBQ$ is given by
$1\!-\!\frac{( 1- 2\cdot r)^2}{ 1 - 2\cdot d }$.

The same methodology can be used to compute PAC-Bound~\ref{bound:semisup}, except that in the semi-supervised setting, the disagreement is computed on the unlabeled data $S_\Ucal$.

\subsection{PAC-Bayesian Theory to Directly Bound the \Cbound}
\label{section:PBC2}

PAC-Bounds~\ref{bound:variancebinouille} and~\ref{bound:semisup} of the last section require two approximations to upper bound~$\CDQ$\,: one on $\RDGQ$ and another on $\dDQ$. We introduce below an extension to the PAC-Bayesian theory (Theorem~\ref{thm:pac-trinouille-kl}) that enables us to directly bound $\CDQ$. To do so, we directly bound any
pair of expectations among $\eDQ$, $\sDQ$ and $\dDQ$. For this reason, the new PAC-Bayesian theorem is based on a trivalent random variable instead of a Bernoulli
one (which is bivalent). 
Note that~\citet{seeger-thesis} and~\citet{seldin-tishby-10} have presented more general PAC-Bayesian theorems valid for $k$-valent random variables, for any positive integer $k$. However, our result leads to tighter bounds for the $k=3$ case.

Before we get to this new PAC-Bayesian theorem (Theorem~\ref{thm:pac-trinouille-kl}), we need some preliminary results.

\subsubsection{A General PAC-Bayesian Theorem for Two Losses of Paired-Voters}

Theorem~\ref{thm:pac-trinouille-general} below allows us to simultaneously bound two losses of paired-voters. This result is inspired by the general PAC-Bayesian theorem for real-valued losses (Theorem~\ref{thm:gen-pac-Bayes}).

\begin{theorem} \label{thm:pac-trinouille-general}
For any distribution $D$ on \mbox{$\Xcal\!\times\!\{-1,1\}$}, for any set $\Hcal$ of voters \mbox{$\Xcal \rightarrow [-1,1]$}, for any two losses $\loss_\alpha, \loss_\beta: [-1,1]\times\{-1,1\}\rightarrow [0,1]$ with $\alpha,\beta\in\{e,s,d\}$, for
any prior distribution $P$ on $\Hcal$, for any $\dt\in
(0,1]$, for any $m'>0$, and for any convex function  $\Dcal(q_1,q_2\,\|\,p_1,p_2)$,
we have %
\begin{equation*}
\prob{S\sim D^m}\!\LP\!\!\!
\begin{array}{l}
  \mbox{For all posteriors $Q$ on $\Hcal$}: \\[1mm]
\scriptstyle \Dcal\bigg(
 	 \esp{\fpaired\sim Q^2} \hspace{-1mm} \Eaggloss{\alpha}{S} \big(\fpaired\big)\,,
 	\esp{\fpaired\sim Q^2} \hspace{-1mm} \Eaggloss{\beta}{S} \big(\fpaired\big)
 	\ \bigg\|
 	\esp{\fpaired\sim Q^2} \hspace{-1mm} \Eaggloss{\alpha}{D} \big(\fpaired\big)\,,
 	\esp{\fpaired\sim Q^2} \hspace{-1mm} \Eaggloss{\beta}{D} \big(\fpaired\big)
 	\bigg)    
 \   \\
\hspace{65mm}  \le\ \   \mbox{ \small $\dfrac{1}{m'}\Bigg[ 2\cdot\KL(Q\|P) +
\ln\Bigg(\dfrac{\Omega}{\delta}\,\Bigg)\Bigg]$}
\end{array}
 \!\!\RP \ge \, 1 -\dt\,,\\
\end{equation*}
where \ $ \Omega\,\eqdef \esp{S\sim D^m}\esp{\fpaired\sim P^2}e^{m'\cdot
	\Dcal\Big( \Eaggloss{\alpha}{S}\big(\fpaired\big) , \,\,\Eaggloss{\beta}{S}\big(\fpaired\big) \,\Big\|\,\Eaggloss{\alpha}{D}\big(\fpaired\big) ,\,\,\Eaggloss{\beta}{D}\big(\fpaired\big)\Big)
	}$.
\end{theorem}

\begin{proof} To simplify the notation, first let 
 \mbox{ $\aij \eqdef \, \Eaggloss{\alpha}{D'} \big(\fpaired\big)$}  \, and  
 \mbox{ $\bij  \eqdef \, \Eaggloss{\beta}{D'} \big(\fpaired\big)$}.

Now, since $\Eb_{\fpaired\sim P^2}\,e^{m'\cdot\Dcal\big( \aSij, \bSij \,\big\|\,\aDij,\bDij\big)}$ is a positive 
random variable, Markov's inequality (Lemma~\ref{lem:markov}, in Appendix~\ref{section:appendix_auxmath}) can be applied to give  
{\small
\begin{equation*}
\prob{S\sim D^m}\biggl(\esp{\fpaired\sim P^2}e^{m'\cdot	\Dcal\big( \aSij, \bSij \,\big\|\,\aDij,\bDij\big)} \le
\frac{1}{\delta}\esp{S\sim D^m}\esp{\fpaired\sim P^2}e^{m'\cdot\Dcal\big( \aSij, \bSij \,\big\|\,\aDij,\bDij\big)}
\biggr) \ge 1-\delta\, . 
\end{equation*}
}By exploiting the fact that $\ln(\cdot)$ is an increasing function, and by the definition of $\Omega$,
we obtain
{\small
\begin{equation}\label{eq:gen-pac-bayes-1}
\prob{S\sim D^m}\biggl( 
\ln\LB\esp{\fpaired\sim P^2}e^{m'\cdot\Dcal\big( \aSij, \bSij \,\big\|\,\aDij,\bDij\big)}\RB 
\ \le \
\ln\LB\frac{\Omega}{\delta}\RB \biggr) \ge 1-\delta\, .
\end{equation}
}We apply the change of measure inequality (Lemma~\ref{lem:change-measure}) on the left side of innermost inequality, with $\phi(f) = m'\cdot\Dcal\big( \aSij, \bSij \,\big\|\,\aDij,\bDij\big)$, $P=P^2$ and $Q=Q^2$.
We then use Jensen's inequality (Lemma~\ref{lem:jensen}, in Appendix~\ref{section:appendix_auxmath}), exploiting the convexity of~$\Dcal$ :
{\small
\begin{eqnarray*}
\ln\LB\esp{\fpaired\sim P^2}e^{m'\cdot\Dcal\big( \aSij, \bSij \,\big\|\,\aDij,\bDij\big)}\RB 
 \hspace{-4.5cm}&&\\
&\geq& 
m' \esp{\fpaired\sim Q^2} \Dcal\big( \aSij, \bSij \,\big\|\,\aDij,\bDij\big) -\KL\big(Q^2\big\|P^2\big)
\\
&\geq& 
m'\cdot\Dcal\Bigg(
	\esp{\fpaired\sim Q^2} \hspace{-1mm} \aSij,
	\esp{\fpaired\sim Q^2} \hspace{-1mm} \bSij
	\ \Bigg\|\ 
	\esp{\fpaired\sim Q^2} \hspace{-1mm} \aDij,
	\esp{\fpaired\sim Q^2} \hspace{-1mm} \bDij	
	\Bigg) -\KL\big(Q^2\big\|P^2\big) \\
&=& 
m'\cdot \Dcal\Bigg(
	\esp{\fpaired\sim Q^2} \hspace{-1mm} \aSij,
	\esp{\fpaired\sim Q^2} \hspace{-1mm} \bSij
	\ \Bigg\|\ 
	\esp{\fpaired\sim Q^2} \hspace{-1mm} \aDij,
	\esp{\fpaired\sim Q^2} \hspace{-1mm} \bDij	
	\Bigg) -2\cdot\KL\big(Q\big\|P\big) \,.
\end{eqnarray*}
}The last equality $\KL\big(Q^2\big\|P^2\big)  = 2\cdot\KL\big(Q\big\|P\big)$ has been shown in the proof of Theorem~\ref{thm:binouille}.
The result can then be straightforwardly obtained by inserting the last inequality into Equation~\eqref{eq:gen-pac-bayes-1}.
\end{proof}

\subsubsection{A PAC-Bayesian Theorem for Any Pair Among $\eDQ$, $\sDQ$, and $\dDQ$}

In Section~\ref{section:PB_real_losses}, Theorem~\ref{thm:pac-bayes-kl-g} was obtained from Theorem~\ref{thm:gen-pac-Bayes}. Similarly, the main theorem of this subsection (Theorem~\ref{thm:pac-trinouille-kl}) is deduced from Theorem~\ref{thm:pac-trinouille-general}. However, a notable difference between Theorems~\ref{thm:pac-bayes-kl-g} and~\ref{thm:pac-trinouille-kl} is that the former uses of the KL-divergence $\kl(\cdot\|\cdot)$ between distributions of two Bernoulli (\ie, \emph{bivalent}) random variables, and the latter uses the $\KL$-divergence $\kl(\cdot,\cdot\|\cdot,\cdot)$ between distributions of two \emph{trivalent} random variables. %

\smallskip
Given two trivalent random variables $Y_q$ and $Y_p$ with $P(Y_q\!=\!a)=q_1$,
$P(Y_q\!=\!b)=q_2$, $P(Y_q\!=\!c)=1\!-\!q_1\!-\!q_2$, \, and \,
$P(Y_p\!=\!a)=p_1$, $P(Y_p\!=\!b)=p_2$, $P(Y_p\!=\!c)=1\!-\!p_1\!-\!p_2$, we denote by  
$\kl(q_1,q_2\,\|\,p_1,p_2)$ the Kullback-Leibler divergence between $Y_q$ and $Y_p$. Thus, we have
 \begin{equation}\label{eq:small_kl2}
  \kl(q_1,q_2\,\|\,p_1,p_2) \ \eqdef\ 
  q_1\ln\frac{q_1}{p_1} +q_2\ln\frac{q_2}{p_2}+(1-q_1-q_2)\ln\frac{1-q_1-q_2}{1-p_1-p_2}\,.
 \end{equation}
Note that $\kl\big(q_1,q_2\,\|\,p_1,p_2)$ is a shorthand notation for $\KL(Q\|P)$ of Equation~\eqref{eq:KLQP}, with $Q=(q_1,q_2, 1\!-\!q_1\!-\!q_2)$ and $P = (p_1,p_2, 1\!-\!p_1\!-\!p_2)$.
Corollary~\ref{cor:bintrin} (in Appendix~\ref{section:appendix_auxmath}) shows that $\kl\big(q_1,q_2\,\|\,p_1,p_2)$ is a convex function.

 \medskip
To be able to apply Theorem~\ref{thm:pac-trinouille-general} with 
$\Dcal(q_1,q_2\,\|\,p_1,p_2) = \kl(q_1,q_2\|p_1,p_2)$,
we need Lemma~\ref{lem:xi2} (below). This lemma is inspired by Lemma~\ref{lem:xi}. 
However, in contrast with the latter, which is based on Maurer's lemma, Lemma~\ref{lem:xi2} needs a generalization of it to trivalent random variables (instead of bivalent ones). The proof of this generalization is provided in Appendix~\ref{section:appendix_auxmath}, listed as Lemma~\ref{lem:general-maurer}.
\begin{lemma} \label{lem:xi2}
 For any distribution $D$  on \mbox{$\Xcal\!\times\!\{-1,1\}$}, for any paired-voters $\fpaired$, and any positive integer $m$,  we have
 \begin{equation*}
\esp{S\sim D^m} e^{{m\,\cdot\,\mbox{\small $\kl$}\Big( \Eaggloss{\alpha}{S} (\fpaired),\Eaggloss{\beta}{S} (\fpaired)\ \Big\|\ \Eaggloss{\alpha}{D} (\fpaired), \Eaggloss{\beta}{D} \big(\fpaired\big) \Big)} }
\ \leq \ \xi(m)+m\,,
\end{equation*}
where $\loss_\alpha$ and $\loss_\beta$ can be any two of the three losses $\loss_s$, $\loss_e$ or $\loss_d$,
and where  $\xi(m)$ is defined at Equation~\eqref{eq:xi}. Therefore,\  $m+\sqrt{m} \ \leq \  \xi(m) +m \ \leq\  m+2\sqrt{m}$\,.
\end{lemma}
\begin{proof}
Let $Y_{ij}$ be a random variable that follows a multinomial distribution with three possible outcomes: $a\eqdef(1,0)$, $b\eqdef(0,1)$ and $c\eqdef(0,0)$. The ``\emph{Trinomial}'' distribution is chosen such that $\pr\left(Y_{ij}\!=\!a\right) = \Eaggloss{\alpha}{D}(\fpaired)$, $\pr\left(Y_{ij}\!=\!b\right) = \Eaggloss{\beta}{D}(\fpaired)$ and $\pr\left(Y_{ij}\!=\!c\right) = 1-\Eaggloss{\alpha}{D}(\fpaired)-\Eaggloss{\beta}{D}(\fpaired)$.  Given $m$ trials of $Y_{ij}$, we denote $Y_{ij}^{a}$, $Y_{ij}^{b}$ and $Y_{ij}^{c}$ the number of times each outcome is observed. 
Note that $Y_{ij}$ is totally defined by $(Y_{ij}^a, Y_{ij}^b)$, since $Y_{ij}^c = m - Y_{ij}^a-Y_{ij}^b$. We thus use the notation\\[-5mm]
\begin{eqnarray*}
 Y_{ij} \ = \  (Y_{ij}^a, Y_{ij}^b) \, \sim \, \Tcal_{ij} & \eqdef &\mbox{Trinomial} \Big(m, \, \Eaggloss{\alpha}{D}(\fpaired), \, \Eaggloss{\beta}{D}(\fpaired) \Big).
\end{eqnarray*}
Hence, we have\\[-5mm]
{\small
\begin{eqnarray*} \textstyle
\prob{(Y_{ij}^a, Y_{ij}^b)\sim\Tcal_{ij}} \!\!\Big(Y_{ij}^a \!=\! k_1 \!\wedge\! Y_{ij}^b \!=\! k_2 \Big) 
=
 {m \choose k_1} \!{m-k_1 \choose k_2}\! \big[ \Eaggloss{\alpha}{S} (\fpaired) \big]^{k_1}\! \big[ \Eaggloss{\beta}{S} (\fpaired) \big]^{k_2}\!\big[1\!-\!\Eaggloss{\alpha}{S} (\fpaired) \!-\!\Eaggloss{\beta}{S} (\fpaired)\big]^{m-k_1-k_2}\!,
\end{eqnarray*}
}for any $k_1\in\{0,..,m\}$ \, and \, any $k_2\in\{0,..,m\!-\!k_1\}$.

\noindent
Now, applying Lemma~\ref{lem:general-maurer} to the convex function
$e^{m\,\cdot\,\mbox{\small $\kl$}\big(\,\cdot\,, \,\cdot\ \|\, \Eaggloss{\alpha}{D} (\fpaired), \,\Eaggloss{\beta}{D} (\fpaired) \big)}$, and by the definition of $\kl(\cdot,\cdot\|\cdot,\cdot)$, we have
{\small
\begin{eqnarray*}
\qquad& & \hspace{-1.5cm}
\esp{S\sim D^m} 
 e^{m \,\cdot\,\mbox{\small $\kl$}\Big( \Eaggloss{\alpha}{S} (\fpaired),\Eaggloss{\beta}{S} (\fpaired)\ \Big\|\ \Eaggloss{\alpha}{D} (\fpaired), \Eaggloss{\beta}{D} \big(\fpaired\big) \Big)}\\
 & \leq &
\esp{(Y_{ij}^a, Y_{ij}^b)\sim\Tcal_{ij}} e^{m \,\cdot\,\mbox{\small $\kl$}\Big( \tfrac{1}{m} Y_{ij}^a, \tfrac{1}{m} Y_{ij}^b \ \Big\|\ \Eaggloss{\alpha}{D} (\fpaired), \Eaggloss{\beta}{D} \big(\fpaired\big) \Big)}\\
&=&
\esp{(Y_{ij}^a, Y_{ij}^b)\sim\Tcal_{ij}} 
\left(\frac{\tfrac{1}{m} Y_{ij}^a}{\Eaggloss{\alpha}{S} (\fpaired)}\right)^{Y_{ij}^a}
\left(\frac{\tfrac{1}{m} Y_{ij}^b}{\Eaggloss{\beta}{S} (\fpaired)}\right)^{Y_{ij}^b}
\left(\frac{1-\tfrac{1}{m} Y_{ij}^a-\tfrac{1}{m} Y_{ij}^b}{1-\Eaggloss{\alpha}{S} (\fpaired)-\Eaggloss{\beta}{S} (\fpaired)}\right)^{m-Y_{ij}^a-Y_{ij}^b}\,.
\end{eqnarray*}
}As $Y_{ij}$ follows a trinomial law, we then have
{\small
\begin{eqnarray*}
\qquad& & \hspace{-1.5cm}
 \esp{(Y_{ij}^a, Y_{ij}^b)\sim\Tcal_{ij}} 
\left(\frac{\tfrac{1}{m} Y_{ij}^a}{\Eaggloss{\alpha}{S} (\fpaired)}\right)^{Y_{ij}^a}
\left(\frac{\tfrac{1}{m} Y_{ij}^b}{\Eaggloss{\beta}{S} (\fpaired)}\right)^{Y_{ij}^b}
\left(\frac{1-\tfrac{1}{m} Y_{ij}^a-\tfrac{1}{m} Y_{ij}^b}{1-\Eaggloss{\alpha}{S} (\fpaired)-\Eaggloss{\beta}{S} (\fpaired)}\right)^{m-Y_{ij}^a-Y_{ij}^b}
\\[2mm]
&=& \sum_{k_1=0}^m \sum_{k_2=0}^{m-k_1}  \Bigg[ 
\prob{(Y_{ij}^a, Y_{ij}^b)\sim\Tcal_{ij}} \Big(Y_{ij}^a = k_1 \wedge Y_{ij}^b = k_2 \Big) 
\\[-3mm] & & \hspace{2.5cm} \times
\left(\frac{\tfrac{k_1}{m}}{\Eaggloss{\alpha}{S} (\fpaired)}\right)^{k_1}
\left(\frac{\tfrac{k_2}{m}}{\Eaggloss{\beta}{S} (\fpaired)}\right)^{k_2}
\left(\frac{1-\tfrac{k_1}{m}-\tfrac{k_2}{m}}{1-\Eaggloss{\alpha}{S} (\fpaired)-\Eaggloss{\beta}{S} (\fpaired)}\right)^{m-k_1-k_2}
\Bigg]
\\[2mm]
&=& \sum_{k_1=0}^m \sum_{k_2=0}^{m-k_1}  \Bigg[ 
{m \choose k_1} {m\!-\!k_1 \choose k_2} \!\left( \Eaggloss{\alpha}{S} (\fpaired) \right)^{k_1} \!\left( \Eaggloss{\beta}{S} (\fpaired) \right)^{k_2}\!\left(1\!-\!\Eaggloss{\alpha}{S} (\fpaired)\! -\!\Eaggloss{\beta}{S} (\fpaired)\right)^{m-k_1-k_2}
\\[-3mm]& & \hspace{2.5cm} \times
\left(\frac{\tfrac{k_1}{m}}{\Eaggloss{\alpha}{S} (\fpaired)}\right)^{k_1}
\left(\frac{\tfrac{k_2}{m}}{\Eaggloss{\beta}{S} (\fpaired)}\right)^{k_2}
\left(\frac{1-\tfrac{k_1}{m}-\tfrac{k_2}{m}}{1-\Eaggloss{\alpha}{S} (\fpaired)-\Eaggloss{\beta}{S} (\fpaired)}\right)^{m-k_1-k_2}
\Bigg]
\\[2mm]
&=&\sum_{k_1=0}^m \sum_{k_2=0}^{m-k_1}  {m \choose k_1} {m\!-\!k_1 \choose k_2} \LP \frac{k_1}{m} \RP^{k_1} \LP \frac{k_2}{m} \RP^{k_2}\LP 1-\frac{k_1}{m}-\frac{k_2}{m} \RP^{m-k_1-k_2}
\\[2.5mm]
&=& \xi(m) + m \,.
\end{eqnarray*}
}The last equality  has been proven by~\cite{Younsi-12}.  Recall that $\xi(m)$ is defined by Equation~\eqref{eq:xi}.
\end{proof}
We are now ready to present the main result of this section. By bounding any
pair of expectations among $\eDQ$, $\sDQ$ and $\dDQ$, Theorem~\ref{thm:pac-trinouille-kl} is the perfect tool to directly bound the \Cbound.

\begin{theorem}
\label{thm:pac-trinouille-kl}
For any distribution $D$ on \mbox{$\Xcal\!\times\!\{-1,1\}$}, for any set $\Hcal$ of voters \mbox{$\Xcal \rightarrow [-1,1]$}, for any prior distribution $P$ on $\Hcal$, and any $\dt\in (0,1]$, we have
\begin{equation*}
\prob{S\sim D^m}\LP\!
  \begin{array}{l}
  \mbox{For all posteriors $Q$ on $\Hcal$}: \\[1mm]
  \kl\big(\,\aSQ, \bSQ\,\big\|\,\aDQ, \bDQ\,\big) \,\le\,
   \dfrac{1}{m}\LB 2\!\cdot\!\KL(Q\|P) +
	 \ln\dfrac{\xi(m)+m}{\delta} \RB
	 \end{array}\!\!
	 \RP \ge \,  1 -\delta\,,
\end{equation*}
where $\aQ$ and $\bQ$ can be any two distinct choices among $\dQ$, $\eQ$ and $\sQ$.
\end{theorem}
\begin{proof}
The result follows from Theorem~\ref{thm:pac-trinouille-general} with  
 $ \Dcal(q_1, q_2 \,\|\, p_1, p_2) = \kl(q_1, q_2 \,\|\, p_1, p_2)$ and $m'=m$.
Since Equation~\eqref{eq:esd_from_losses} shows that $\aQ = \esp{\fpaired\sim Q^2} \hspace{-1mm} \aij$ and  $\bQ = \esp{\fpaired\sim Q^2} \hspace{-1mm} \bij$, we have
\begin{multline*}
\prob{S\sim D^m}\Biggl(
\forall\,Q\,\,\text{on}\,\Hcal\colon\,\,
   \kl\left(\aSQ, \bSQ\,\|\,\aDQ, \bDQ\right)   
   \le\\
    \frac{1}{m}\Bigg[ 2\cdot\KL(Q\|P) +
\ln\Bigg(\frac{1}{\delta}
	\esp{S\sim D^m}\esp{\fpaired\sim P^2}e^{m\kl\big( \aSij, \bSij \,\big\|\,\aDij,\bDij\big)}
	\Bigg)\Bigg] \,\Biggr) \ge 1 -\delta\,.
\end{multline*}
As the prior distribution $P^2$ is independent of $S$, we can swap the two expectations in expression 
{\small
$\esp{S\sim D^m}\esp{\fpaired\sim P^2}e^{m\kl( \aSij, \bSij \,\big\|\,\aDij,\bDij)}$}.  
This observation, together with Lemma~\ref{lem:xi2}, gives
 \begin{eqnarray*}
\esp{S\sim D^m}\esp{\fpaired\sim P^2}e^{m\kl\big( \aSij, \bSij \,\big\|\,\aDij,\bDij\big)} 
& = &
\esp{\fpaired\sim P^2} \esp{S\sim D^m} e^{m\kl\big( \aSij, \bSij \,\big\|\,\aDij,\bDij\big)} \\
& \leq & \esp{\fpaired\sim P^2} \xi(m)+m \\
& = &  \xi(m)+m \,.\\[-1cm]
 \end{eqnarray*}
\end{proof}

A first version of Theorem~\ref{thm:pac-trinouille-kl} was proposed by~\cite{nips06-mv}, with the difference that $\ln\frac{(m+1)(m+2)}{2\delta}$ in the latter is now replaced by $\ln\frac{\xi(m)+m}{\delta}$ in the former. Since \mbox{$\xi(m)+m < \frac{(m+1)(m+2)}{2}$}, the new theorem is therefore tighter.

\subsubsection{Another Bound for the Risk of the Majority Vote}
 \label{sec:pac-bounds}

First, we need the following notation that
is related to Theorem~\ref{thm:pac-trinouille-kl}. Given any prior distribution $P$ on $\Hcal$,
\begin{eqnarray}
\label{eq:Adel}
\Adel &\eqdef&  \bigg\{ (d,e)\ \colon \ \kl(\dSQ,\eSQ\|\,d,e) \, \le \, 
\frac{1}{m}\,\bigg[\, 2\!\cdot\! \KL(Q\|P) + \ln\tfrac{\xi(m)+m}{\dt}\bigg]\bigg\}\,.
\end{eqnarray}

The bound is obtained by seeking the point of $\Adel$ maximizing the \Cbound. Since a point $(d,e)$ of~$\Adel$ expresses a disagreement $d$ and a joint error $e$, we directly compute the bound on $\CDQ$ using  Equation~\eqref{eq:Cq_ed}. %

Note however that $\Adel$ can contain points that are not possible in practice, \ie, points that are not achievable with any data-generating distribution $D$. Indeed, by Proposition~\ref{prop:dQ_bound}, we know that 
\begin{eqnarray}\nonumber
 &\dDQ & \leq  \ \, 2 \cdot \RDGQ \cdot \big( 1-\RDGQ \big)\,. %
  \end{eqnarray}
Based on this property, it is possible to significantly reduce the achievable region of $\Adel$. To do so, we must first rewrite this property based on $\dDQ$ and $\eDQ$ only. 
\begin{eqnarray}\nonumber
 &\dDQ & \leq  \ 2 \cdot \RDGQ \cdot \big( 1-\RDGQ \big) 
 \ = \ 2 \cdot \big(\eDQ+\tfrac{1}{2}\dDQ\big) \cdot \big( 1- (\eDQ+\tfrac{1}{2}\dDQ) \big) 
 \\ \nonumber
  \Leftrightarrow \quad & 0 & \leq \ - \tfrac{1}{2}(\dDQ)^2 - 2\eDQ\cdot\dDQ + 2\eDQ - 2(\eDQ)^2
   \\ \label{ineAdelrestreint}  
  \Leftrightarrow \quad & \dDQ & \leq \ 2\cdot\left(\sqrt{\eDQ}-\eDQ\right).
  \end{eqnarray}
  
  \bigskip
  Note also that if $R_D(G_Q)\geq\frac{1}{2}$, there is no bound on $R_D(B_Q)$ better than the trivial one $R_D(B_Q)\le 1$. We therefore consider only the pairs $(d,e)\in\Adel$ that do not correspond to that situation. Since $R_D(G_Q)=\frac{1}{2}(2\eDQ+\dDQ)$ (Equation~\ref{eq:Gibbs_vs_e_d}), this is therefore equivalent to considering only the pairs $(d,e)$ such that $2e+d <1$. We later show that this still gives a valid bound. 
\noindent
Thus, from all these ideas, we restrain $\Adel$ (Equation~\ref{eq:Adel}) as follows:
\begin{eqnarray}\label{eq:adelrestreint}
\Adelrestreint& \eqdef &  \bigg\{ (d,e)\in \Adel \ \colon \   d \,\leq\, 2(\sqrt{e} - e) \ \ \ \mbox{and} \  \ \ 2e+d< 1\bigg\}\,,
\end{eqnarray}
and obtain the following bound that, in contrast with PAC-Bound~\ref{bound:variancebinouille}, directly bounds $\CDQ$.

\begin{pacbound}{2}\label{bound:trinouille}
For any distribution $D$ on \mbox{$\Xcal\!\times\!\{-1,1\}$}, for any set $\Hcal$ of voters \mbox{$\Xcal \rightarrow [-1,1]$}, for any prior distribution $P$ on $\Hcal$, and any $\dt\in (0,1]$, 
 we have
\begin{equation*}
 \prob{S\sim D^m}\!\LP
\forall\,Q\mbox{ on }\Hcal\,:\ 
   R_D(B_Q) \, \leq 
	\sup\limits_{(d,e)\in\Adelrestreint} \left[1-\frac{\Big(1-(2e+d)\Big)^2}{1-2d}\ \right]       
\RP \ge \, 1 -\dt\,.
\end{equation*}
\end{pacbound}
\begin{proof} 
We need to show that the supremum value in the statement of PAC-Bound~\ref{bound:trinouille} is a valid upper bound of $\RDBQ$. Note that if $\Adelrestreint=\emptyset$, then the supremum is $+\infty$, and the bound is trivially valid. Therefore, we assume below that $\Adelrestreint$ is not empty.

Let us consider $(d,e)\in \Adelrestreint$. From the conditions $d \,\leq\, 2(\sqrt{e} - e)$ and $2e+d< 1$, it follows by straightforward calculations that $d<\frac{1}{2}$. This implies that
\begin{equation*} 
1-\frac{\big(1-(2e+d)\big)^2}{1-2d} \ < \ 1\,,
\end{equation*}
because both the numerator and the denominator of the fraction are strictly positive (remember that $2e+d< 1$). 
Thus, the supremum is at most $1$. %

\pagebreak
\noindent
Let us consider the three following cases.\\[2mm]
{\it Case 1:\ The supremum is not attained in $\Adelrestreint$.}\  Note that as $\Adelrestreint$ is a subset of $\mathbb{R}^2$, the supremum must be attained for a pair in the closure of $\Adelrestreint$. The latter is not a closed set only because of its $2e+d< 1$ constraint.
Therefore, the supremum is achieved for a pair $(d,e)$ in the closure for which $1-(2e+d)=0$, implying that the value of the supremum is in that case $1$, which trivially is a valid bound for $R_D(B_Q)$.\\[2mm]
{\it Case 2: \ The supremum is attained in $\Adelrestreint$  and has value $1$.}\  In that case, the bound is again trivially valid.\\[2mm]
{\it Case 3: The supremum is attained in $\Adelrestreint$ 
and has a value strictly lower than $1$.} \ In that case, there must be an $\epsilon>0$ such that
$2e+d<1-\epsilon$ for all $(d,e)\in \Adelrestreint$. Hence, because
 of Equation~\eqref{ineAdelrestreint} and Theorem~\ref{thm:pac-trinouille-kl},
 we have that $2\eDQ+\dDQ < 1-\epsilon$ with probability $1\!-\!\delta$.
 Since $R_D(G_Q) = \frac{1}{2}(2\eDQ+\dDQ)$ (Equation~\ref{eq:Gibbs_vs_e_d}), this implies that, with probability $1\!-\!\delta$,  $R_D(G_Q)<1/2-1/2\epsilon$. Hence, with probability $1\!-\!\delta$, Theorem~\ref{thm:C-bound} is valid -- \ie, $\CDQ$ bounds $\RDBQ$ -- and $(\dDQ,\eDQ)\in\Adelrestreint$.
Thus, 
\begin{eqnarray*}
R_D(B_Q) \ \leq \ \ \CDQ \ = \ 1-\frac{\Big(1-(2\eDQ+\dDQ)\Big)^2}{1-2\dDQ} \ \leq  \ \sup\limits_{(d,e)\in\Adelrestreint} \left[1-\frac{\Big(1-(2e+d)\Big)^2}{1-2d}\ \right],
\end{eqnarray*}
and we are done.
\end{proof}

In some situations, we can slightly improve PAC-Bound~\ref{bound:trinouille} by bounding the joint error~$\eDQ$ via Theorem~\ref{thm:binouille} with $\delta$ replaced by $\delta/2$.   This removes all pairs $(d,e)$ such that $e$ does not belong to the set $\Edeld$ defined as
\begin{eqnarray*}
\Edeld &\eqdef&  \bigg\{ e\ \colon \ \kl(\eSQ \|\, e) \, \le \, 
\frac{1}{m}\,\bigg[\, 2\!\cdot\! \KL(Q\|P) + \ln\tfrac{\xi(m)}{\dt/2}\bigg]\bigg\} \,.
\end{eqnarray*}

Then, by applying  PAC-Bound~\ref{bound:trinouille}, with $\delta$ replaced by $\delta/2$, one can obtain the following slightly improved bound.
\begin{pacbound}{2'} \label{bound:trinouille2}
For any distribution $D$ on \mbox{$\Xcal\!\times\!\{-1,1\}$}, for any set $\Hcal$ of voters \mbox{$\Xcal \rightarrow [-1,1]$}, for any prior distribution $P$ on $\Hcal$, and any $\dt\in (0,1]$, we have
\begin{equation*}
 \prob{S\sim D^m}\!\LP
\forall\,Q\mbox{ on }\Hcal\,:\ 
   R_D(B_Q) \ \leq 
	\sup\limits_{(d,e)\in\AdelrestreintEd} \left[1-\frac{\Big(1-(2e+d)\Big)^2}{1-2d}\ \right]       
\RP \ge \, 1 -\dt\,,
\end{equation*}
where
\begin{equation}\label{eq:AdelrestreintEd}
\AdelrestreintEd \ \eqdef \  \left\{(d,e)\in\AdelEd\ :\ d \,\leq\, 2(\sqrt{e} - e)\,, \ \ 2e+d<1 \  \ \mbox{\rm and } \ e \, \le\,  \sup \Edeld\right\}.
\end{equation}
\end{pacbound}
\begin{proof}
 Immediate consequence of Theorem~\ref{thm:binouille}, PAC-Bound~\ref{bound:trinouille}, and the union bound.
\end{proof}

\subsubsection{Computation of PAC-Bounds~\ref{bound:trinouille} and~\ref{bound:trinouille2}}

\newcommand{\fC}{F_{\!_\Ccal}}

Let us consider the \Cbound as a function $\fC$
of two variables $(d,e)\in  [0, \frac{1}{2}]\times[0,1]$, instead of a function of the distribution~$Q$.
\begin{equation} \label{eq:fC(d,e)}
 \fC(d,e) \ \eqdef\ 1- \frac{\big[\,1-(2e+d)\,\big]^2}{1-2d}\,.
\end{equation}
Proposition~\ref{prop:CQConvexe} (provided in Appendix~\ref{section:appendix_auxmath}) shows that $\fC$ is a concave function. Therefore, PAC-Bound~\ref{bound:trinouille} is obtained by maximizing $\fC(d,e)$ in the domain $\Adelrestreint$ (Equation~\ref{eq:adelrestreint}), which is both bounded and convex. Several optimization methods can achieve this.  In our experiments, we decompose $\fC(d,e)$ in two nested functions of a single argument: 
$$
\sup_{(d,e)\in\Adelrestreint} \!\!\!\!\Big[ \fC(d,e) \Big] \ = 
\sup_{d:(d,\cdot)\in\Adelrestreint} \!\!\!\!\Big[ \fC^*(d) \Big], \quad \mbox{ where } \,
\fC^*(d) \ \eqdef \!\!\!\sup_{e:(d,e)\in\Adelrestreint} \!\!\!\!\Big[ \fC(d,e) \Big]\,.
$$
Thus, we implement the maximization of $\fC$ using a one-dimensional optimization algorithm twice.  
Figure~\ref{fig:BoundTri} shows an application example of PAC-Bound~\ref{bound:trinouille}.

 \begin{figure}[t]
\centering
\subfloat[Contour plot of $\kl(0.4, 0.1 \| d,e)$.]{\includegraphics[height=5.5cm]{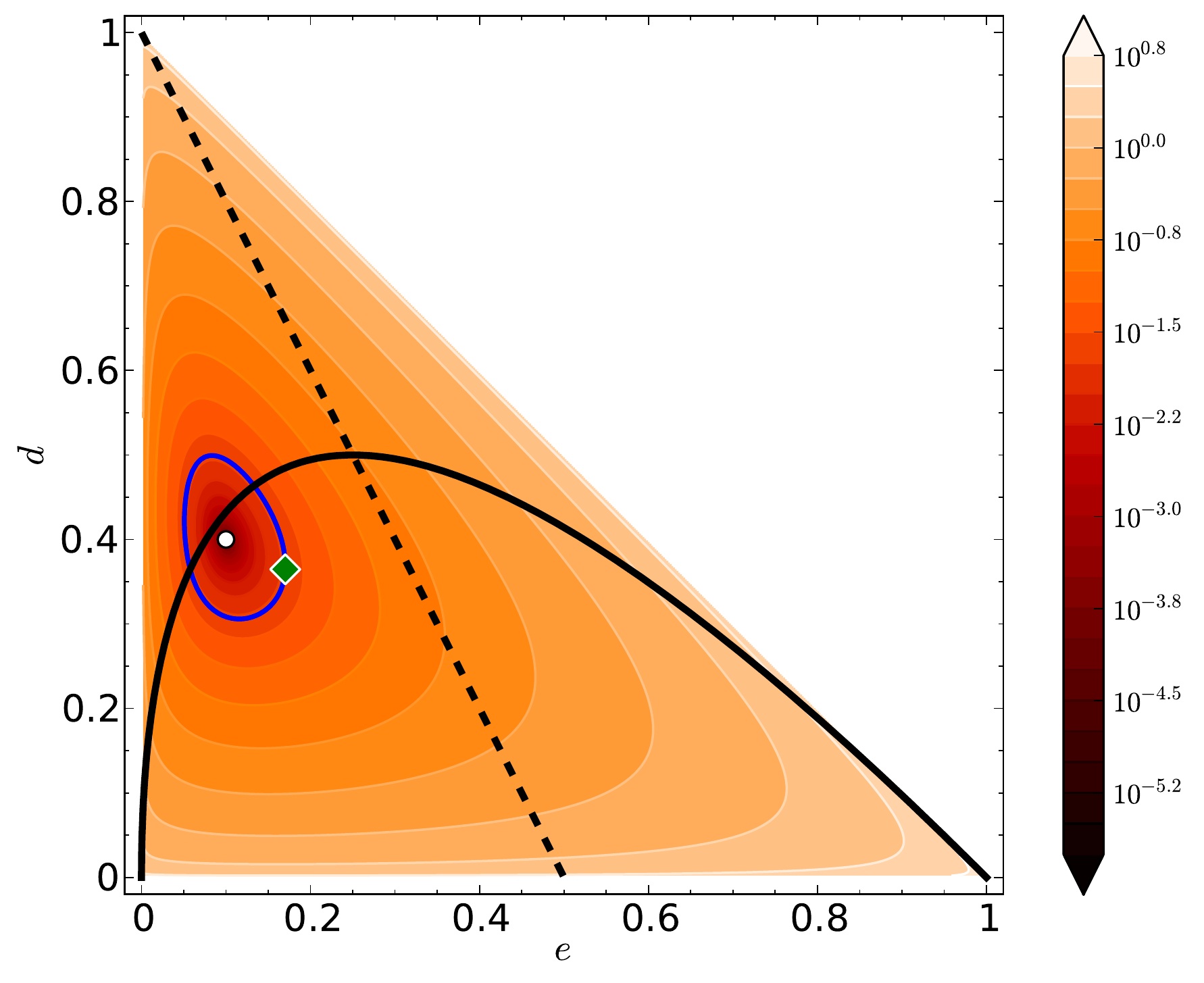} \label{fig:BoundTri_KL}} \quad
\subfloat[Contour plot of $\fC(d,e)$.]{\includegraphics[height=5.5cm]{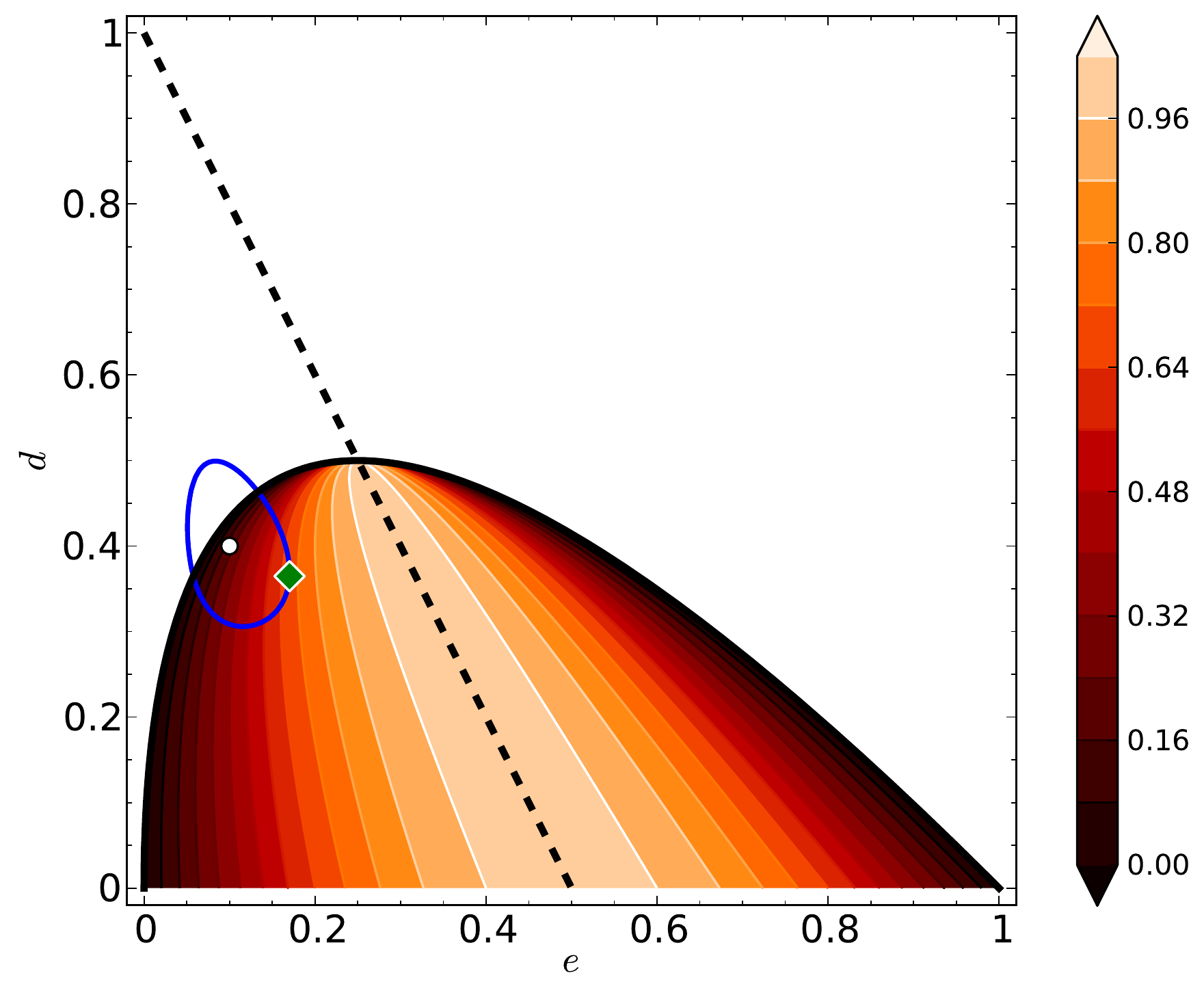} \label{fig:BoundTri_CQ} } 
\caption{Example of application of PAC-Bound~\ref{bound:trinouille}. We suppose that $\KL(Q\|P)=5$, $m=1000$ and $\delta=0.05$. 
 If we observe an empirical joint error $\eSQ=0.10$ and an empirical disagreement $\dSQ=0.40$ (thus, a Gibbs risk $\RSGQ = 0.1 + \frac{1}{2}\cdot 0.4 = 0.30$), then we need to maximize the function $\fC(d,e)$ over the domain $\Adelrestreint$ given by three constraints: $\kl(0.4,0.1\|\,d,e) \le \frac{1}{m}\big[2\!\cdot\! \KL(Q\|P) + \ln\frac{\xi(m)+m}{\dt}\big] \approx 0.0199$ (blue oval), $d \leq 2(\sqrt{e} - e)$ (black curve) and $2e+d<1$ (black dashed line).  Therefore, we obtain a bound  $\RDBQ \leq 0.679$ (corresponding to the green diamond marker).
}
\label{fig:BoundTri}
\vspace{-2mm}
\end{figure}
 \begin{figure}[h]
\centering
\includegraphics[height=5.5cm]{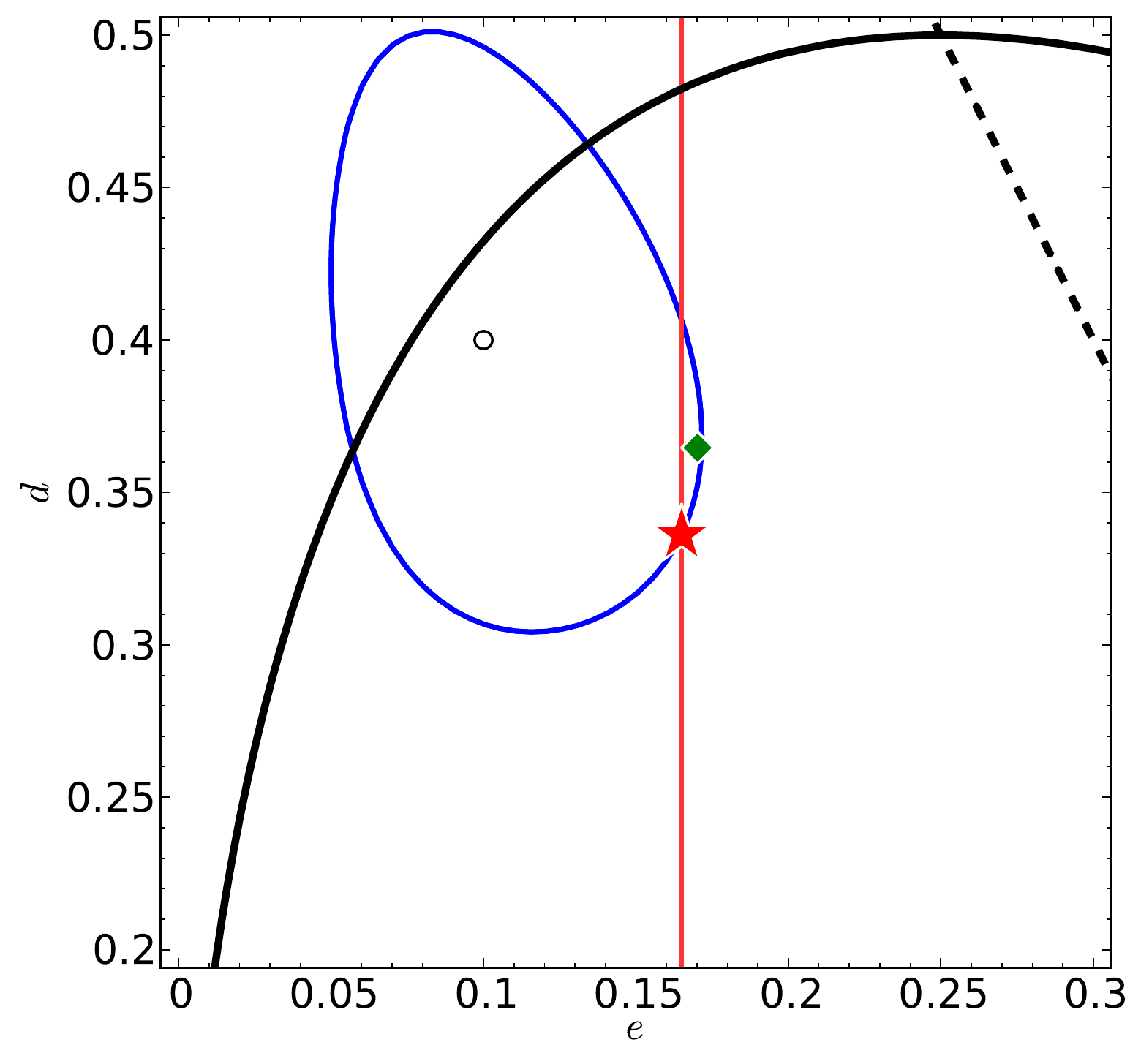} 
\caption{Example of application of PAC-Bound~\ref{bound:trinouille2}. We use the same quantities as for Figure~\ref{fig:BoundTri}. The red vertical line corresponds to the upper bound on the joint error, resulting in an improved bound of $\RDBQ \leq 0.660$ (corresponding to the red star marker). Note however that, even if the bound here is tighter, the egg-region is a bit bigger than in the case of  PAC-Bound~\ref{bound:trinouille} because all the $\delta$ has been replaced by $\delta/2$.}
\label{fig:BoundTri_prime}
\end{figure}

The computation of PAC-Bound~\ref{bound:trinouille2} is done using the same method, but we optimize over the domain $\AdelrestreintEd$ (Equation~\ref{eq:AdelrestreintEd}) instead of $\Adelrestreint$, which is also bounded and convex.  Of course, this requires computing $\sup \Edeld$ beforehand, using the same technique as for PAC-Bounds~\ref{bound:classic}, \ref{bound:variancebinouille} and~\ref{bound:semisup}.
Figure~\ref{fig:BoundTri_prime} shows an application example of PAC-Bound~\ref{bound:trinouille2}.

\subsection{Empirical Comparison Between PAC-Bounds on the Bayes Risk $\RDBQ$}

We now propose an empirical comparison of all PAC-Bounds we presented so far. The numerical results of Figure~\ref{figBornes123} are obtained
 by using AdaBoost \citep{schapire99} with decision stumps on the
 Mushroom UCI data set (which contains 8124 examples). This data set
 is randomly split into two halves: one training set $S$ and one
 testing set $T$. For each round of boosting, we compute the usual PAC-Bayesian bound of twice the Gibbs risk (PAC-Bound~\ref{bound:classic}) of the corresponding majority vote classifier, as well as the other variants of the PAC-Bayesian bounds presented in this paper.
 
\begin{figure}[t] 
\begingroup
  \makeatletter
  \providecommand\color[2][]{%
    \GenericError{(gnuplot) \space\space\space\@spaces}{%
      Package color not loaded in conjunction with
      terminal option `colourtext'%
    }{See the gnuplot documentation for explanation.%
    }{Either use 'blacktext' in gnuplot or load the package
      color.sty in LaTeX.}%
    \renewcommand\color[2][]{}%
  }%
  \providecommand\includegraphics[2][]{%
    \GenericError{(gnuplot) \space\space\space\@spaces}{%
      Package graphicx or graphics not loaded%
    }{See the gnuplot documentation for explanation.%
    }{The gnuplot epslatex terminal needs graphicx.sty or graphics.sty.}%
    \renewcommand\includegraphics[2][]{}%
  }%
  \providecommand\rotatebox[2]{#2}%
  \@ifundefined{ifGPcolor}{%
    \newif\ifGPcolor
    \GPcolorfalse
  }{}%
  \@ifundefined{ifGPblacktext}{%
    \newif\ifGPblacktext
    \GPblacktexttrue
  }{}%
  \let\gplgaddtomacro\g@addto@macro
  \gdef\gplbacktext{}%
  \gdef\gplfronttext{}%
  \makeatother
  \ifGPblacktext
    \def\colorrgb#1{}%
    \def\colorgray#1{}%
  \else
    \ifGPcolor
      \def\colorrgb#1{\color[rgb]{#1}}%
      \def\colorgray#1{\color[gray]{#1}}%
      \expandafter\def\csname LTw\endcsname{\color{white}}%
      \expandafter\def\csname LTb\endcsname{\color{black}}%
      \expandafter\def\csname LTa\endcsname{\color{black}}%
      \expandafter\def\csname LT0\endcsname{\color[rgb]{1,0,0}}%
      \expandafter\def\csname LT1\endcsname{\color[rgb]{0,1,0}}%
      \expandafter\def\csname LT2\endcsname{\color[rgb]{0,0,1}}%
      \expandafter\def\csname LT3\endcsname{\color[rgb]{1,0,1}}%
      \expandafter\def\csname LT4\endcsname{\color[rgb]{0,1,1}}%
      \expandafter\def\csname LT5\endcsname{\color[rgb]{1,1,0}}%
      \expandafter\def\csname LT6\endcsname{\color[rgb]{0,0,0}}%
      \expandafter\def\csname LT7\endcsname{\color[rgb]{1,0.3,0}}%
      \expandafter\def\csname LT8\endcsname{\color[rgb]{0.5,0.5,0.5}}%
    \else
      \def\colorrgb#1{\color{black}}%
      \def\colorgray#1{\color[gray]{#1}}%
      \expandafter\def\csname LTw\endcsname{\color{white}}%
      \expandafter\def\csname LTb\endcsname{\color{black}}%
      \expandafter\def\csname LTa\endcsname{\color{black}}%
      \expandafter\def\csname LT0\endcsname{\color{black}}%
      \expandafter\def\csname LT1\endcsname{\color{black}}%
      \expandafter\def\csname LT2\endcsname{\color{black}}%
      \expandafter\def\csname LT3\endcsname{\color{black}}%
      \expandafter\def\csname LT4\endcsname{\color{black}}%
      \expandafter\def\csname LT5\endcsname{\color{black}}%
      \expandafter\def\csname LT6\endcsname{\color{black}}%
      \expandafter\def\csname LT7\endcsname{\color{black}}%
      \expandafter\def\csname LT8\endcsname{\color{black}}%
    \fi
  \fi
  \setlength{\unitlength}{0.0500bp}%
  \begin{picture}(8640.00,5544.00)%
    \gplgaddtomacro\gplbacktext{%
      \csname LTb\endcsname%
      \put(726,465){\makebox(0,0)[r]{\strut{} 0}}%
      \put(726,994){\makebox(0,0)[r]{\strut{} 0.1}}%
      \put(726,1523){\makebox(0,0)[r]{\strut{} 0.2}}%
      \put(726,2052){\makebox(0,0)[r]{\strut{} 0.3}}%
      \put(726,2581){\makebox(0,0)[r]{\strut{} 0.4}}%
      \put(726,3110){\makebox(0,0)[r]{\strut{} 0.5}}%
      \put(726,3639){\makebox(0,0)[r]{\strut{} 0.6}}%
      \put(726,4168){\makebox(0,0)[r]{\strut{} 0.7}}%
      \put(726,4697){\makebox(0,0)[r]{\strut{} 0.8}}%
      \put(726,5226){\makebox(0,0)[r]{\strut{} 0.9}}%
      \put(2025,220){\makebox(0,0){\strut{} 10}}%
      \put(3294,220){\makebox(0,0){\strut{} 20}}%
      \put(4563,220){\makebox(0,0){\strut{} 30}}%
      \put(5831,220){\makebox(0,0){\strut{} 40}}%
      \put(7100,220){\makebox(0,0){\strut{} 50}}%
      \put(8242,220){\makebox(0,0){\small rounds}}%
    }%
    \gplgaddtomacro\gplfronttext{%
      \csname LTb\endcsname%
      \put(6879,2736){\makebox(0,0)[r]{\strut{}Bound 0}}%
      \csname LTb\endcsname%
      \put(6879,2516){\makebox(0,0)[r]{\strut{}Bound 1}}%
      \csname LTb\endcsname%
      \put(6879,2296){\makebox(0,0)[r]{\strut{}Bound 2}}%
      \csname LTb\endcsname%
      \put(6879,2076){\makebox(0,0)[r]{\strut{}Bound 2$\mathbf{'}$\hspace{-1mm}}}%
      \csname LTb\endcsname%
      \put(6879,1856){\makebox(0,0)[r]{\strut{}Bound 1$\mathbf{'}$\hspace{-1mm}}}%
      \csname LTb\endcsname%
      \put(6879,1636){\makebox(0,0)[r]{\strut{}$\CSQ$ ($\Cbound$ on train)}}%
      \csname LTb\endcsname%
      \put(6879,1416){\makebox(0,0)[r]{\strut{}$R_T(B_Q)$ (Risk on test)}}%
    }%
    \gplbacktext
    \put(0,0){\includegraphics{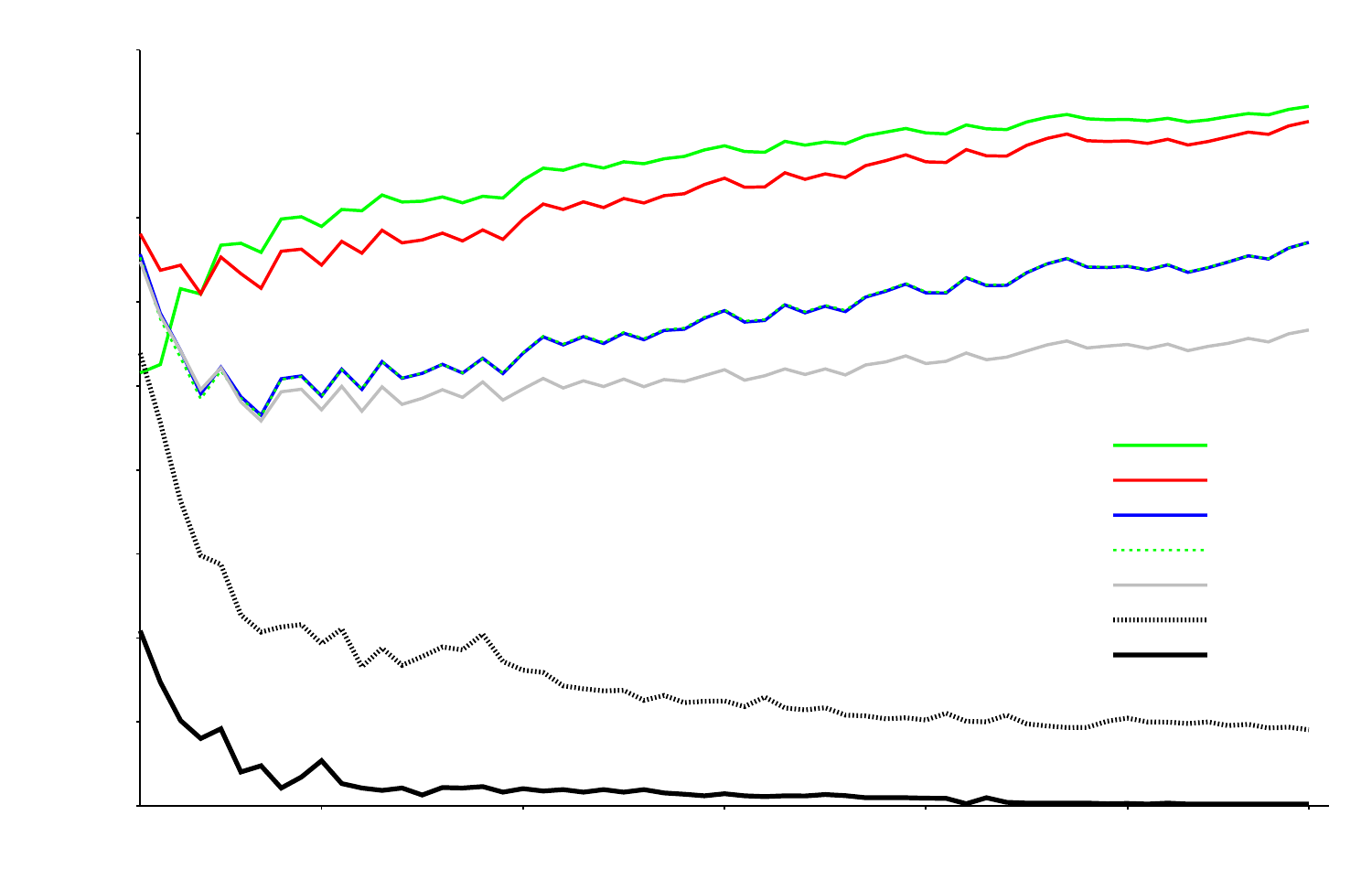}}%
    \gplfronttext
  \end{picture}%
\endgroup
\caption{Comparison of bounds of $\RDBQ$ during 60 rounds of Boosting.}
\label{figBornes123} 
\end{figure}

We can see that PAC-Bound~\ref{bound:variancebinouille} is generally tighter than PAC-Bound~\ref{bound:classic}, and we obtain a substantial improvement with PAC-Bound~\ref{bound:trinouille}. Almost no improvement is obtained with PAC-Bound~\ref{bound:trinouille2} in that case.
We can also see that using unlabeled data  to estimate $\dDQ$ helps, as PAC-Bound~\ref{bound:semisup} is the tightest.\footnote{To obtain PAC-Bound~\ref{bound:semisup}, we simulate the  case where we have access to a large number of unlabeled data by simply using the empirical value of $\dTQ$ computed on the testing set.}

However, we see in Figure~\ref{figBornes123} that after $8$ rounds of boosting, all the bounds are degrading even if the value of $\CSQ$ 
 continues to decrease. This drawback is due to the fact that the denominator of $\CSQ$ tends to $0$, that is the second moment of the margin $\momenttwo(\MQ{S})$ is close to 0 (see the first or the second forms of Theorem~\ref{thm:C-bound}).
 Hence, in this context, the first moment of the margin $\momentone(\MQ{S})$ must be small as
 well. Thus, any slack in the bound of $\momentone(\MQ{D})$ has a multiplicative
 effect on each of the three proposed PAC-bounds of $\RDBQ$.
 Unfortunately, Boosting algorithms tend to construct majority votes
 with $\momentone(\MQ{S})$ just slightly larger than~0.

\section{PAC-Bayesian Bounds without $\mathbf{\KL}$}
\label{section:Further-PAC_Bayes}

Having PAC-Bayesian theorems that bound the difference between $\CSQ$ and $\CDQ$ opens the way to structural \Cbound minimization algorithms. As for most PAC-Bayesian results, the bound on $\CDQ$ depends on an empirical estimate of it, and on the Kullback-Leibler divergence $\KL(Q\|P)$ between the output distribution $Q$ and the a priori defined distribution $P$.
In this section, we present a theoretical extension of our PAC-Bayesian approach that is mandatory to develop the $\CDQ$-minimization algorithm of Section~\ref{sect:mincq}.

The next theorems introduce PAC-Bayesian bounds that have the surprising property of having no $\KL$ term. This new approach is driven by the fact that  our attempts to construct algorithms that minimize any of the PAC-Bounds presented in the previous section ended up being unsuccessful. Surprisingly, the KL-divergence is a poor regularizer in this case, as its empirical value tends to be overweighted in comparison with the empirical value of the \Cbound (\ie, $\CSQ$). 

\medskip

There have already been some attempts to develop PAC-Bayesian bounds that do not rely on the KL-divergence
(see the localized priors of~\citealp{c-07}, or the distribution-dependent priors of \citealp{lls-13}). The usual idea is to bound the KL-divergence via some concentration inequality. In the following, the $\KL$ term simply vanishes from the bound, provided that we restrict ourselves to \emph{aligned posteriors}, a notion that is properly defined later on in this section. The fact that these new PAC-Bayesian bounds do not contain any KL divergence terms indicates that the restriction to aligned posteriors has some ``built in'' regularization action. 

The following theory is similar to the one used by \cite{gllms-11}, in which two learning algorithms inspired by the PAC-Bayesian theory are compared: one regularized with the $\KL$ divergence, using a hyperparameter to control its weight, and one regularized by restricting the posterior distributions to be \emph{aligned} on the prior distribution. Surprisingly, the latter algorithm uses one less parameter, and has been shown to have an as good accuracy.

\subsection{Self-Complemented Sets of Voters and Aligned Distributions}\label{sec:aligneetquasi-unif}
In this section, we assume that the (possibly infinite) set of voters $\Hcal$ is \emph{self-complemented}\footnote{In~\cite{lmr-11}, this notion was introduced as an \emph{auto-complemented} set of voters. However, \emph{self-complemented} is a more suitable name. Also, note that a similar notion, called a \emph{symmetric hypothesis class}, is introduced in~\cite{DBLP:journals/corr/DanielySBS13}.}.

\begin{definition} \rm \label{def:H_autocomp}
A set of voters $\Hcal$ is said to be \emph{self-complemented} if there exists a bijection $c: \Hcal \rightarrow \Hcal$ such that for any $f \in \Hcal$, $$c(f) \ =\  -f\,.$$ 
\end{definition}

Moreover, we say that a distribution $Q$ on any self-complemented $\Hcal$ is \emph{aligned} on a prior distribution $P$ if $$Q(f) + Q(c(f)) \ =\  P(f) + P(c(f))\,, \ \ \ \forall f \in \Hcal\,.$$

When $P$ is the uniform prior distribution and $Q$ is aligned on $P$, we say that $Q$ is \emph{quasi-uniform}. Note that the uniform distribution is itself a quasi-uniform distribution.

In the finite case, we consider self-complemented sets $\Hcal$ of $2n$ voters $\Xcal \rightarrow \Yover$. In this setting, for any $\xb\in\Xcal$ and any $i\in\{1, \ldots , n\}$, we have that $f_{i+n}(\xb) = -f_i(\xb)$. Moreover, finite quasi-uniform distributions $Q$ is such that for any $i\in\{1,\ldots, n\}$, 
\begin{equation} \label{eq:quasi-uniform}
Q(f_i)+Q(f_{i+n}) \ =\ \frac{1}{n}\,.
\end{equation}

Equation~\eqref{eq:quasi-uniform} shows that when a distribution $Q$ is restricted to being quasi-uniform, the sum of the weight given to a pair of complementary voters is equal to $\frac{1}{n}$. As $Q$ is a distribution, this means that the weight of any voter is lower-bounded by $0$ and upper-bounded by~$\frac{1}{n}$, giving rise to an $L_\infty$-norm regularization.
Note that, in this context, the maximum value of $\KL(Q\|P)$ is reached when all voters have a weight of either $0$ or $\frac{1}{n}$. Indeed, a quasi-uniform distribution $Q$ is such that $\KL(Q\|P)\leq n(\frac{1}{n}) \ln(\frac{1}{n}/\frac{1}{2n})=\ln 2$. Consequently, the value of the $\KL$ term is necessarily small and plays a little role in PAC-Bayesian bounds computed with quasi-uniform distributions.
The following theorems and corollaries are specializations that allow to slightly improve these PAC-Bayesian bounds by getting rid of the $\KL$ term completely. To achieve these results,  the associated proofs require restrictions on the choice of convex function $\Dcal$ and loss function~$\loss$.

\subsection{PAC-Bayesian Theorems without $\mathbf{\KL}$ for the Gibbs Risk}
Let us first specialize Theorem~\ref{thm:gen-pac-Bayes} to aligned distributions and linear loss $\linloss$. We first need a new change of measure inequality, as this is the part of Theorem~\ref{thm:gen-pac-Bayes} where the $\KL$ term appears.
\begin{lemma}[Change of measure inequality for aligned posteriors] \label{lem:change-measure-aligned}~\\
For any self-complemented set $\Hcal$, for any distribution $P$ on $\Hcal$, any distribution $Q$ aligned on $P$, and for any measurable function $\phi:\Hcal­\to \Reals$ such that $\phi(f) = \phi(c(f))$ for all $f \in \Hcal$, we have
\begin{equation*}
\esp{f\sim Q} \phi(f) \ \leq \ \ln \left( \esp{f\sim P} e^{\phi(f)} \right) .
\end{equation*}
\end{lemma}
\begin{proof}
First, note that one can change the expectation over $Q$ to an expectation over $P$, using the fact that $\phi(f) = \phi(c(f))$ for any $f \in \Hcal$, and that $Q$ is aligned on $P$.
\begin{align*}
  2\cdot\esp{f\sim Q}\phi(f) \ &=\  \int_{\Hcal} df \ Q(f)\, \phi(f) + \int_{\Hcal} df \ Q(c(f))\, \phi(c(f))\\
  &=\ \int_{\Hcal} df \ Q(f)\, \phi(f) + \int_{\Hcal} df \ Q(c(f))\, \phi(f)\\
  &=\ \int_{\Hcal} df \ \Big(Q(f) + Q(c(f))\Big)\, \phi(f)\\
  &=\ \int_{\Hcal} df \ \Big(P(f) + P(c(f))\Big)\, \phi(f)\\
  &=\ \int_{\Hcal} df \ P(f)\, \phi(f) + \int_{\Hcal} df \ P(c(f))\, \phi(f)\\
  &=\ \int_{\Hcal} df \ P(f)\, \phi(f) + \int_{\Hcal} df \ P(c(f))\, \phi(c(f))\\
  &=\ 2\cdot\esp{f\sim P} \phi(f)\, .
\end{align*}
The result is obtained by changing the expectation over $Q$ to an expectation over $P$, and then by applying Jensen's inequality (Lemma~\ref{lem:jensen}, in Appendix~\ref{section:appendix_auxmath}).
\begin{eqnarray*}
\esp{f \sim Q} \phi(f) 
& = &\esp{f \sim P} \phi(f) %
\ = \ \esp{f \sim P} \ln e^{\phi(f)}  %
\ \leq \ \ln \left(\esp{f \sim P} e^{\phi(f)} \right). %
\end{eqnarray*}
\end{proof}

\begin{theorem}[PAC-Bayesian theorem for aligned posteriors] \label{thm:gen-pac-Bayes-qu}
For any distribution $D$ on \mbox{$\Xcal\!\times\!\{-1,1\}$}, any self-complemented set $\Hcal$ of voters \mbox{$\Xcal \rightarrow [-1,1]$}, any prior distribution $P$ on $\Hcal$, any convex function $\Dcal : [0,1]\times [0,1] \rightarrow \Reals$ for which $\Dcal(q, p) = \Dcal(1-q, 1-p)$, for any $m'>0$ and any $\dt\in (0,1]$, we have
\begin{equation*}
\prob{S\sim D^m}\!\LP \!\!
\begin{array}{l}
  \mbox{For all posteriors $Q$ aligned on $P$}: \\[1mm]
  \Dcal\big(\RSGQ, \RDGQ\big) \, \le\,  
  \dfrac{1}{m'}\LB 
  \ln\LP\dfrac{1}{\dt}\,\esp{S\sim D^m}\esp{f\sim P}\!\!e^{\,m'\cdot\Dcal\left(\ElinlossS(f),\,\ElinlossD(f)\right)}\RP\RB 
\end{array}
\!\! \RP  \ge\, 1 -\dt \,.
\end{equation*}
\end{theorem}

Similarly to Theorem~\ref{thm:gen-pac-Bayes}, the statement of Theorem~\ref{thm:gen-pac-Bayes-qu} above contains a value $m'$ which is likely to be set to $m$ in most cases. However, the distinction between $m$ and $m'$ is mandatory to develop the PAC-Bayesian theory for sample-compressed voters in Section~\ref{sec:sample-compression}. Indeed, in proofs of forthcoming Theorems~\ref{thm:McAllester-sc}, \ref{thm:McAllester-sc-qu} and~\ref{thm:aligned-sc-disagreement}, we have $m'=m-\lambda$, where $\lambda$ is the size of the voters compression sequence (this concept is properly defined in Section~\ref{sec:sample-compression}).\\

\begin{proof}
The proof follows the exact same steps as the proof of Theorem~\ref{thm:gen-pac-Bayes}, using the linear loss $\loss = \linloss$ and replacing the use of the change of measure inequality (Lemma~\ref{lem:change-measure}) \linebreak by the change of measure inequality for aligned posteriors (Lemma~\ref{lem:change-measure-aligned}), with \linebreak \mbox{$\phi(f) = m'\cdot\Dcal\left(\ElinlossS(f),\,\ElinlossD(f)\right)$}. Note that this function has the required property, as \[\Dcal\left(\ElinlossS(f),\,\ElinlossD(f)\right) = \Dcal\left(1-\ElinlossS(c(f)),\,1-\ElinlossD(c(f))\right) = \Dcal\left(\ElinlossS(c(f)),\,\ElinlossD(c(f))\right)\,.\]
The other steps of the proof stay exactly the same as the proof of Theorem~\ref{thm:gen-pac-Bayes}.
\end{proof}
Appendix~\ref{section:gen-pac-Bayes-prod-qu} presents more general versions of the last two results.

\medskip
Let us specialize Theorem~\ref{thm:gen-pac-Bayes-qu} to the case where $\Dcal(q,p) = \kl(q\|p)$. Doing so, we recover the classical PAC-Bayesian theorem (Theorem~\ref{thm:pac-bayes-kl-g}), but for aligned posteriors, which therefore has no $\KL$ term.

\begin{corollary} \label{cor:pac-bayes-classic-qu}%
For any distribution $D$ on \mbox{$\Xcal\!\times\!\{-1,1\}$}, any prior distribution $P$ on a self-complemented set $\Hcal$ of voters \mbox{$\Xcal \rightarrow [-1,1]$}, and any $\dt\in (0,1]$, we have
\begin{equation*}
\prob{S\sim D^m}\left(\!\!
\begin{array}{l}
  \mbox{For all posteriors $Q$ aligned on $P$}: \\[1mm]
   \kl\big(R_S(G_Q) \,\big\|\, R_D(G_Q)\big) \, \le\, \dfrac{1}{m}\LB\ln\dfrac{\xi(m)}{\dt}\RB
\end{array}\!\!
\right) \ge \,1 -\dt\,,
\end{equation*}
where $\kl(q\|p)$ and $\xi(m)$ and  defined by Equations~\eqref{eq:small_kl} and~\eqref{eq:xi} respectively.
\end{corollary}
\begin{proof}
This result follows from
Theorem~\ref{thm:gen-pac-Bayes-qu} by choosing $\Dcal(q,p) =
\kl(q,p)$ and $m'=m$. The rest of the proof relies on Lemma~\ref{lem:xi} (as for the proof of Theorem~\ref{thm:pac-bayes-kl-g}).
\end{proof}

The following corollary is very similar to the original PAC-Bayesian bound of~\cite{m-03}, but without the $\KL$ term.

\begin{corollary} \label{cor:pac-bayes-McAllester-qu}
For any distribution $D$ on \mbox{$\Xcal\!\times\!\{-1,1\}$}, any self-complemented set $\Hcal$ of voters \mbox{$\Xcal \rightarrow [-1,1]$}, any prior distribution $P$ on $\Hcal$, and any $\dt\in (0,1]$, we have
\begin{equation*}
\prob{S\sim D^m}\left(\!\!
\begin{array}{l}
    \mbox{For all posteriors $Q$ aligned on $P$}: \\[1mm]
    R_D(G_Q) \, \le \, 
    R_S(G_Q) + \sqrt{ \dfrac{1}{2m}\LB \ln\tfrac{\xi(m)}{\dt}\RB }
\end{array}\!\!
\right) \ge\, 1 -\dt\,.
\end{equation*}
\end{corollary}
\begin{proof}
 The result is derived from Corollary~\ref{cor:pac-bayes-classic-qu}, by using $2(q-p)^2 \leq \kl(q\|p)$ (Pinsker's inequality), and isolating $\RDGQ$ in the obtained inequality.
\end{proof}

Unlike Theorem~\ref{thm:gen-pac-Bayes}, Theorem~\ref{thm:gen-pac-Bayes-qu} cannot straightforwardly be used for pairs of voters, as we did in the proof of Theorem~\ref{thm:binouille}. The reason is that a posterior distribution that is the result of the product of two aligned posteriors is not necessarily aligned itself. So, we have to ensure that we can get rid of the $\KL$ term even in that case.

\bigskip
\subsection{PAC-Bayesian Theorems without $\KL$ for the Expected Disagreement $\dDQ$}
The following theorem is similar to Theorem~\ref{thm:gen-pac-Bayes-qu} for aligned posteriors, but deals with paired-voters. Instead of the linear loss $\linloss$, we use the loss $\loss_d$ of Equation~\eqref{eqn:lossd}, which is a linear loss defined on a pair of voters. Again, the next two results can be seen as a particular case of the two theorems from Appendix~\ref{section:gen-pac-Bayes-prod-qu}.

In this subsection, we use the following shorthand notation. Given $\fpaired = \langle f_i, f_j \rangle$ as defined in Definition~\ref{def:paired_voters}, the voters $f_{i^c j}$, $f_{ij^c}$ and $f_{i^c j^c}$ are defined as
$$f_{i^c j}(\xb)  \eqdef \langle c(f_i)(\xb), f_j(\xb) \rangle, \  
f_{ij^c}(\xb)  \eqdef \langle f_i(\xb), c(f_j)(\xb) \rangle, \ \mbox{and} \
f_{i^c j^c}(\xb) \eqdef \langle c(f_i)(\xb), c(f_j)(\xb) \rangle.$$

Recall that from Equation~\eqref{eq:H2Q2}, we have $\Hcal^2 \ \eqdef\ \{\fpaired \, :\, f_i, f_j \in \Hcal\}$ and $Q^{2}(\fpaired)\ \eqdef\ Q(f_i)\cdot Q(f_j)$. Similarly, we define $P^{2}(\fpaired)\ \eqdef\ P(f_i)\cdot P(f_j)$. Using this notation, let us first generalize the change of measure inequality of Lemma~\ref{lem:change-measure-aligned} to paired-voters.

\begin{lemma}\label{lem:change-measure-paired-aligned}\textbf{\emph{(Change of measure inequality for paired-voters and aligned posteriors)}} 
For any self-complemented set $\Hcal$, for any distribution $P$ on $\Hcal$, any distribution $Q$ aligned on $P$, and for any measurable function $\phi:\Hcal^2­\to \Reals$ such that $\phi(f_{ij}) = \phi(f_{i^c j}) = \phi(f_{ij^c}) = \phi(f_{i^c j^c})$ for all $f_{ij} \in \Hcal^2$, we have
\begin{equation*}
\esp{f_{ij} \sim Q^2} \phi(f_{ij}) \ \leq \ \ln \left( \esp{f_{ij} \sim P^2} e^{\phi(f_{ij})} \right) \,.
\end{equation*}
\end{lemma}
\begin{proof}
First, note that one can change the expectation over $Q^2$ to an expectation over $P^2$, using the fact that $\phi(f_{ij}) = \phi(f_{i^c j}) = \phi(f_{ij^c}) = \phi(f_{i^c j^c})$ for any $f_{ij} \in \Hcal^2$, and that $Q$ is aligned on $P$. More specifically, we have the following.
{\small \allowdisplaybreaks[1]
\begin{align*}
 4\cdot&\esp{f_{ij}\sim Q^2}\phi(\fpaired) \  \\
&=\, \int_{\Hcal^2}\!\!\!\!\! d \fpaired Q^2(f_{ij})\, \phi(f_{ij}) +\! \int_{\Hcal^2}\!\!\!\!\! d \fpaired Q^2(f_{i^c j})\, \phi(f_{i^c j}) %
      +\! \int_{\Hcal^2}\!\!\!\!\! d \fpaired Q^2(f_{ij^c})\, \phi(f_{ij^c}) +\! \int_{\Hcal^2}\!\!\!\!\!  d\fpaired Q^2(f_{i^c j^c})\, \phi(f_{i^c j^c})\\[2mm]
&=\, \int_{\Hcal^2}\!\!\!\!\! d \fpaired \, Q^2(f_{ij})\, \phi(f_{ij}) + \int_{\Hcal^2}\!\!\!\!\! d \fpaired \, Q^2(f_{i^c j})\, \phi(f_{ij}) %
      + \int_{\Hcal^2} \!\!\!\!\! d \fpaired \, Q^2(f_{ij^c})\, \phi(f_{ij}) + \int_{\Hcal^2} \!\!\!\!\! d\fpaired \, Q^2(f_{i^c j^c})\, \phi(f_{ij})\\[2mm]
  &=\, \int_{\Hcal^2} \!\!\!\!\! d \fpaired \Big( Q^2(f_{ij}) + Q^2(f_{i^c j}) + Q^2(f_{ij^c}) + Q^2(f_{i^c j^c})\Big) \phi(f_{ij})\\[2mm]
  &=\, \int_{\Hcal^2} \!\!\!\!\! d \fpaired \Big( P^2(f_{ij}) + P^2(f_{i^c j}) + P^2(f_{ij^c}) + P^2(f_{i^c j^c})\Big) \phi(f_{ij})\\[-2mm]
&\ \, \vdots\\
&=\, 4\cdot\esp{f_{ij}\sim P^2}\phi(\fpaired)\,.
\end{align*}}The result is then obtained by changing the expectation over $Q^2$ to an expectation over $P^2$, and then by applying Jensen's inequality (Lemma~\ref{lem:jensen}, in Appendix~\ref{section:appendix_auxmath}).
\begin{eqnarray*}
\esp{\fpaired \sim Q^2} \phi(\fpaired) 
& = &\esp{\fpaired \sim P^2} \phi(\fpaired) %
\ = \ \esp{\fpaired \sim P^2} \ln e^{\phi(\fpaired)}  %
\ \leq \ \ln \left(\esp{\fpaired \sim P^2} e^{\phi(\fpaired)} \right)\,. %
\end{eqnarray*}\end{proof}

\begin{theorem}[PAC-Bayesian theorem for paired-voters and aligned posteriors] \label{thm:pac-Bayes-q2-qu}
For any distribution $D$ on \mbox{$\Xcal\!\times\!\{-1,1\}$}, any self-complemented set $\Hcal$ of voters \mbox{$\Xcal \rightarrow [-1,1]$}, any prior distribution $P$ on $\Hcal$, any convex function $\Dcal : [0,1]\times [0,1]
\rightarrow \Reals$ for which $\Dcal(q, p) = \Dcal(1-q, 1-p)$, for any $m'>0$ and any $\dt\in (0,1]$, we have
\begin{equation*}
\prob{S\sim D^m}\!\!\LP \!\!
\begin{array}{l}
 \mbox{For all posteriors $Q$ aligned on $P$}: \\[1mm]
 \Dcal\big(\,\dSQ,\, \dDQ\,\big)  \,\le\, 
  \dfrac{1}{m'}\!\LB 
  \ln\LP\dfrac{1}{\dt}\,\esp{S\sim D^m}\esp{\fpaired\sim P^2}\!\!\!e^{\,m'\cdot\Dcal(\Eaggloss{d}{S}(\fpaired),\,\Eaggloss{d}{D}(\fpaired))}\RP\RB 
\end{array}
\!\!\! \RP  \ge\, 1 -\dt \,,
\end{equation*}
where $\fpaired$ is given in Definition~\ref{def:paired_voters}, and where $P^2(\fpaired) \eqdef P(f_i)\cdot P(f_j).$ 
\end{theorem}
\hspace{1cm}

\begin{proof}
Theorem~\ref{thm:pac-Bayes-q2-qu} is deduced from Theorem~\ref{thm:gen-pac-Bayes-qu}, by using the change of measure inequality given by Lemma~\ref{lem:change-measure-paired-aligned} instead of the one from Lemma~\ref{lem:change-measure-aligned}, with \mbox{$\phi(\fpaired) = m'\cdot\Dcal(\Eaggloss{d}{S}(\fpaired),\,\Eaggloss{d}{D}(\fpaired))$}. As the loss $\loss_d$ is such that
\begin{equation*}
\Eaggloss{d}{D'}(f_{i^c j^c}) \ =\ \Eaggloss{d}{D'}(f_{ij})\,,\quad \mbox{ and } \quad
\Eaggloss{d}{D'}(f_{i^c j}) \ =\ \Eaggloss{d}{D'}(f_{ij^c}) \ =\ 1-\Eaggloss{d}{D'}(f_{ij})\,,
\end{equation*}
we then have that $\phi(\fpaired)$ has the required property to apply Lemma~\ref{lem:change-measure-paired-aligned}.
\end{proof}

\noindent
Let us now specialize Theorem~\ref{thm:pac-Bayes-q2-qu} to $\Dcal(q,p) = \kl(q\|p)$. 
\begin{corollary}\label{thm:binouille-qu} 
For any distribution $D$ on \mbox{$\Xcal\!\times\!\{-1,1\}$}, any self-complemented set $\Hcal$ of voters \mbox{$\Xcal \rightarrow [-1,1]$}, any prior distribution $P$ on $\Hcal$, and any $\dt\in(0,1]$, we have
\begin{equation*}
\prob{S\sim D^m}\LP\!\!
\begin{array}{l}
\mbox{For all posteriors $Q$ aligned on $P$}: \\[1mm]
   \kl\big(\dSQ\,\|\,\dDQ\big) \ \le \ 
    \dfrac{1}{m}\,\LB \ln\tfrac{\xi(m)}{\dt}\RB
\end{array}
 \!\!\RP \ge \, 1 -\dt\,.
\end{equation*}
\end{corollary}
\begin{proof}
 The result is directly obtained from Theorem~\ref{thm:pac-Bayes-q2-qu}, 
 by choosing $\Dcal(q,p) = \kl(q,p)$.
 The rest of the proof relies on Lemma~\ref{lem:xi}.
\end{proof}

\noindent
Similarly as for Corollary~\ref{cor:pac-bayes-McAllester-qu}, we can easily derive the following result.

\begin{corollary} \label{cor:pac-bayes-McAllester-qu2}
For any distribution $D$ on \mbox{$\Xcal\!\times\!\{-1,1\}$}, for any self-complemented set $\Hcal$ of voters $\Xcal \rightarrow [-1,1]$, any prior distribution $P$ on $\Hcal$, and any $\dt\in (0,1]$, we have
\begin{equation*}
\prob{S\sim D^m}\left(\!\!
\begin{array}{l}
\mbox{For all posteriors $Q$ aligned on $P$}: \\[1mm]
    \dDQ \, \ge \,
    \dSQ - \sqrt{ \dfrac{1}{2m}\LB \ln\tfrac{\xi(m)}{\dt}\RB }
\end{array}\!\!
\right) \ge\, 1 -\dt\,.
\end{equation*}
\end{corollary}
\begin{proof}
 The result is derived from Corollary~\ref{thm:binouille-qu}, by using $2(q-p)^2\leq  \kl(q\|p)$ (Pinsker's inequality), and isolating $\dDQ$ in the obtained inequality.
\end{proof}

\subsection{A Bound for the Risk of the Majority Vote without $\KL$ Term}  

Finally, we make use of these results to bound $\CDQ$ -- and therefore $\RDBQ$ -- for aligned posteriors $Q$,  giving rise to PAC-Bound~\ref{bound:variancebinouille-qu}.  Aside from the fact  that this bound has no $\KL$ term, it is similar to PAC-Bound~\ref{bound:variancebinouille},
as it separately bounds the Gibbs risk and the expected disagreement. 
 This new PAC-Bayesian bound provides us with a starting point to design the MinCq leaning algorithm introduced in Section~\ref{sect:mincq}.
\begin{pacbound}{3}\label{bound:variancebinouille-qu}
For any distribution $D$ on \mbox{$\Xcal\!\times\!\{-1,1\}$}, for any self-complemented set $\Hcal$ of voters \mbox{$\Xcal \rightarrow [-1,1]$}, for any prior distribution $P$ on $\Hcal$, and any $\dt\in (0,1]$, we have
\begin{equation*}
  \prob{S\sim D^m}\!\left(\!
  \begin{array}{l}
	\forall\,Q \mbox{ aligned on } P\,:\\
	   \RDBQ \,\, \leq \,\, 
   1-\dfrac{\big( \,1- 2\cdot \overline{r} \,\big)^2}{ 1 - 2\cdot \underline{d} } \ = \
   1-\dfrac{\big( \, \underline{\mu_1}_{} \, \big)^2 }{ \overline{\mu_2} } 
   \end{array}\!%
\right) \ge\, 1 -\dt\,,
\end{equation*}
where \\[-6mm]
\begin{align*}
 \overline{r} &\, \eqdef\, \min\left(\tfrac{1}{2},\, \RSGQ + \sqrt{ \mbox{\small $\frac{1}{2m}\LB \ln\frac{\xi(m)}{\dt/2}\RB$} } \right),
 &\quad
 \underline{d} &\, \eqdef\, \max\left(0,\, \dSQ - \sqrt{ \mbox{\small $\frac{1}{2m}\LB \ln\frac{\xi(m)}{\dt/2}\RB$} }\right),\\
 \underline{\mu_1} &\, \eqdef\, \max\left(0,\,\momentone(\MQ{S}) - \sqrt{ \mbox{\small $\frac{2}{m}\LB \ln\frac{\xi(m)}{\dt/2}\RB$} }\right),
 &\quad
 \overline{\mu_2} &\, \eqdef\, \min\left(1,\, \momenttwo(\MQ{S}) + \sqrt{  \mbox{\small $\frac{2}{m}\LB \ln\frac{\xi(m)}{\dt/2}\RB$}}\right).
\end{align*}
\end{pacbound}
\begin{proof}
 The inequality is a consequence of Theorem~\ref{thm:C-bound}, as well as Corollaries~\ref{cor:pac-bayes-McAllester-qu} and~\ref{cor:pac-bayes-McAllester-qu2}. The equality
 $1-\frac{( 1- 2\cdot \overline{r} )^2}{ 1 - 2\cdot \underline{d} } = 1-\frac{(\underline{\mu_1}_{} )^2 }{ \overline{\mu_2} } $
 is a direct application of Equations~\eqref{eq:RGibbsMq} and \eqref{eq:dQ_Mq}.
\end{proof}

PAC-Bound~\ref{bound:variancebinouille-qu-sc} that is presented at the end of Section~\ref{sec:sample-compression} accepts voters that are kernel functions defined using a part of the training set $S$. This is unusual in the PAC-Bayesian theory, since the prior $P$ on the set of voters has to be defined before seeing the training set~$S$. To overcome this difficulty, we use the sample compression theory. 

\section{PAC-Bayesian Theory for Sample-Compressed Voters}
\label{sec:sample-compression}

\newcommand{\sigb}{{\boldsymbol{\sigma}}}

\newcommand{\fsc}{{f_{(S_\ib, \sigma)}}}
\newcommand{\fscprim}{{f_{(\ib'\!\!, \sigma'\!)}}}
\newcommand{\fscpaired}{{f_{ (\ib, \sigma)(\ib'\!\!, \sigma'\!) }}}

\newcommand{\Rsc}{{\Rcal(S_\ib, \sigb)}}
\newcommand{\Rscprim}{{\Rcal(S_{\ib'}, \sigb')}}
\newcommand{\Rscpaired}{{\overline{\Rcal}(S_{\ib,\ib'}, \sigb,\sigb')}}

\newcommand{\Il}{{\Ical_\lambda}}
\newcommand{\Sl}{{\Sigma_\lambda}}
\newcommand{\ISl}{{\Il\times\Sl}}
\newcommand{\HSl}{\Hcal^\Rcal_{S,\lambda}}

\newcommand{\HSlk}{\Hcal^{\Rcal_k}_{S,1}}

PAC-Bayesian theorems of Sections~\ref{section:PAC-Bayes}~and~\ref{section:Further-PAC_Bayes} are not valid
when $\Hcal$ consists of a set of functions of the form $\pm k(\xb_i, \cdot)$ for some kernel \mbox{$k:\Xcal\times\Xcal\rightarrow[-1,1]$}, as is the case with the Support Vector Machine classifier (see Equation~\ref{eq:svm_intro}). This is because the  definition of each involved voter depends on an example $(\xb_i,y_i)$ of the training data $S$. This is problematic from the PAC-Bayesian point of view because the prior on the voters is supposed to be defined before seeing the data $S$. There are two known methods to overcome this problem.

 The first method, introduced by~\cite{ls-03}, considers a surrogate set of voters $\Hcal^k$ of \emph{all} the linear classifiers in the space induced\footnote{This space is also known as a Reproducible  Kernel Hilbert Space (RKHS). For more details, see \cite{cs-00} and~\citet{shs-01}} by the kernel $k$.
They then make use of the representer theorem to show that the classification function turns out to be a linear combination of the examples, similar to the Support Vector Machine classifier (Equation~\ref{eq:svm_intro}). To avoid the curse of dimensionality, they propose 
 restricting the choice of the prior and posterior distributions on $\Hcal^k$ to isotropic Gaussian centered on a vector representing a particular linear classifier. Based on this approach, \cite{gllm-09} suggests a learning algorithm for linear classifiers that exactly consists in a PAC-Bayesian bound minimization.

The second method, that is presented in the present section, is based on the sample compression setting of~\citet{fw-95}. It has been adapted to the PAC-Bayesian theory by~\citet{lm-05,lm-07}, allowing one to directly deal with the case where voters are constructed using examples in the training set, without involving any RKHS notion nor any representer theorem. Conversely to the first method described above, the sample compression approach allows one not only to deal with kernel functions, but with any kind of similarity measure between examples, hence to deal with any kind of voters.

\subsection{The General Sample Compression Setting}
In the \emph{sample compression setting}, learning algorithms have access to a data-dependent set of voters, that we refer to as \emph{sc-voters}. Given a training sequence\footnote{The sample compression theory considers the training examples as a sequence instead of a set, because it refers to the training examples by their indices.} $S=\tuple{(\xb_1,y_1), \dots,(\xb_{m},y_{m})}$, each sc-voter is described by a sequence~$S_\ib$ of elements of $S$ called the \textit{compression sequence}, 
and a \textit{message}~$\sigma$ which represents the additional information needed to obtain a voter from $S_\ib$. If $\ib=\langle i_1,i_2,..,i_k\rangle$, then $S_\ib\eqdef \langle(x_{i_1}, y_{i_1}),\, (x_{i_2}, y_{i_2}),\ldots,  (x_{i_k}, y_{i_k})\rangle$.
In this paper, repetitions are allowed in $S_\ib$, 
and $k$, the number of indices present in $\ib$ (counting the repetitions), is denoted by $\nib$.

\medskip
The fact that each sc-voter is described by a compression sequence and a message implies that there exists a \textit{reconstruction function}
$\Rcal(S_\ib,\sigma)$ that outputs a classifier
when given an arbitrary compression sequence $S_\ib$ and a message $\sigma$. The message $\sigma$ is chosen from
 the set~$\Sigma_{S_\ib}$
 of all messages that can be supplied with the compression sequence $S_\ib$. In the PAC-Bayesian setting, $\Sigma_{S_\ib}$ must be defined a priori (before
 observing the training data) for all possible sequences $S_\ib$, and can be either a discrete or a continuous set. The sample compression setting strictly generalizes the (classical) non-sample-compressed setting, since the latter corresponds to the case where $\nib=0$, the voters being then defined only via the messages.

\subsection{A Simplified Sample Compression Setting}
For the needs of this paper, we consider a simplified framework where sc-voters have a compression sequence of at most $\lambda$ examples (possibly with repetitions) and a message string of $\lambda$ bits that we represent by a sequence of ``$-1$'' and ``$+1$''.
Instead of being defined on sc-voters, the weighted distribution $Q$ is defined on $\ISl$, where 
\begin{eqnarray}
\label{eq:ISlambda}
\Il \, \eqdef\, \Big\{\langle i_1,i_2,..,i_k\rangle\, :\, k\in\{0,..,\lambda\}   \, \mbox{ and } \, i_j\in \{1,..,m\} \Big\}
\quad \mbox{ and } \quad 
\Sl \, \eqdef\,  \Big\{\!-1,1\Big\}^\lambda .
\end{eqnarray}
In other words, $Q(\ib,\sigb)$ corresponds to the weight of the sc-voter output by $\Rsc$, \ie, the sc-voter of compression sequence $\ib=\langle i_1,\ldots,i_\nib\rangle\in\Il$ 
and message 
\mbox{$\sigb = \langle \sigma_1,\ldots,\sigma_\lambda\rangle\in\Sl$}.
In particular, a prior (resp., a posterior) on the set of all sc-voters is now simply a prior on the set $\Ical_{\lambda}\times\Sigma_{\lambda}$. Thus, such a prior can really be defined \emph{a priori\/}, before seeing the data~$S$.\footnote{\cite{lm-07} describe a more general setting where, for each $S\in(\Xcal\times\Ycal)^m$, a prior is defined on
$\Ical_{\lambda}\times\Sigma_{S_\ib}$. Hence, the messages may depend on the compression sequence $S_\ib$.}
  The set of sc-voters is therefore only defined when the training sequence $S$ is given, and corresponds to 
\begin{eqnarray*}
\HSl \, \eqdef\, \{\Rsc \, :\, \ib\in\Il,\,\sigb\in\Sl\}\,.
\end{eqnarray*}

Finally, given a training sequence~$S$ and a reconstruction function $\Rcal$, for a distribution~$Q$ on $\Ical_{\lambda}\times\Sigma_{\lambda}$, we define the Bayes classifier as
\begin{equation*}
\BQS \ \eqdef\ \sgn\left[\esp{(\ib,\sigb)\sim Q} \, \Rsc \,\right]\,.
\end{equation*}
We then define the Bayes risk~$\RBQS$ and the Gibbs risk~$\RGQS$ of a distribution $Q$ on $\ISl$ relative to $D'$ as 
\begin{eqnarray*} %
\RBQS &\eqdef &  \Ezoloss{D'}\big( \BQS \big)\,,\\[1mm]
\RGQS &\eqdef &\esp{(\ib,\sigb)\sim Q} \ElinlossDprim\big( \, \Rsc \, \big)\,.
\end{eqnarray*}

\subsection{A First Sample-Compressed PAC-Bayesian Theorem}

To derive PAC-Bayesian bounds for majority votes of sc-voters, one must deal with the following issue: even if the training sequence $S$ is drawn i.i.d.\ from a data-generating distribution~$D$, the empirical risk of the Gibbs $\RSGQS$ is not an unbiased estimate of its true risk $\RDGQS$. For instance, the reconstruction function $\Rcal$ can be such that an sc-voter output by $\Rsc$ never errs on an example belonging to its compression sequence $S_\ib$; this biases the empirical risk because examples of $S_\ib$ are all in $S$.

To deal with this bias, the~$\frac{1}{m}$ factor in the usual PAC-Bayesian bounds is replaced by a factor of the form $\frac{1}{m-l}$ in their sample compression versions. In \citet{lm-05,lm-07}, $l$ corresponds to the $Q$-average size of the sample compression sequence. In  the present paper, we restrain ourselves to a simpler case, where $l$ is the maximum possible size of a compression sequence (\ie, $l=\lambda$).  This simplification allows us to deal with the biased character of the empirical Gibbs risk using a proof approach similar to the one proposed in \cite{gllms-11}. The key step of this approach is summarized in the following lemma.

\begin{lemma}
\label{lem:LaplaceTransform}
Let $\Rcal$ be a reconstruction function that outputs sc-voters of size at most~$\lambda$ $($where $\lambda<m)$.
For any distribution $D$ on \mbox{$\Xcal\!\times\!\{-1,1\}$}, and
for any prior distribution~$P$ on $\ISl$,
\begin{eqnarray*}
\hspace{-2cm}  \esp{S\sim D^m}\esp{(\ib,\sigb)\sim P}\!\!\!\!\!\! e^{(m-\lambda)\cdot2\cdot\big(\ElinlossS(\Rsc)-\ElinlossD(\Rsc)\big)^2} 
&\leq& 
e^{4\,\lambda} \cdot \xi(m\!-\!\lambda)\,,
\end{eqnarray*}
where $\xi(\cdot)$ is defined by Equation~\eqref{eq:xi}, and therefore we have that \,$\xi(m\!-\!\lambda)\leq 2\sqrt{m\!-\!\lambda}$\,.
\end{lemma}
\begin{proof}
As the the choice of $(\ib,\sigb)$ according to the prior $P$ is independent\footnote{Note that because of this independence, the exchange in the order of the two expectations (Line~\ref{eq:proof-sc-newa}) is trivial. This independence is a direct consequence of our choice to only consider the simplified setting described by Equation~\eqref{eq:ISlambda}. In the more general setting of \citet{lm-07}, this part of the proof is more complicated.} of $S$, we have
{\small
\begin{eqnarray} 
\nonumber
& & \hspace{-2cm} \esp{S\sim D^m}\esp{(\ib,\sigb)\sim P}\!\!\!\!\!\! e^{(m-\lambda)\cdot2\cdot\big(\ElinlossS(\Rsc)-\ElinlossD(\Rsc)\big)^2} \\
\label{eq:proof-sc-newa}
&=& 
\esp{(\ib,\sigb)\sim P} \esp{S\sim D^m}\!\!\!\!\! e^{(m-\lambda)\cdot2\cdot\big(\ElinlossS(\Rsc)-\ElinlossD(\Rsc)\big)^2} \\
\label{eq:proof-sc-new}
&=& 
\esp{(\ib,\sigb)\sim P}  \esp{S_\ib\sim D^\lambda} \esp{S_\ibc\sim D^{m-\lambda}}\!\!\!\!\! e^{(m-\lambda)\cdot2\cdot\big(\ElinlossS(\Rsc)-\ElinlossD(\Rsc)\big)^2}.
\end{eqnarray}
}

\vspace{-4mm}
Let us now rewrite the empirical loss of an sc-voter as a combination of the loss on its compression sequence $S_\ib$ and the loss on the other training examples $S_\ibc$.
\begin{eqnarray*}
 \ElinlossS(\Rsc) 
 & = & \frac{1}{m} \left[ \lambda \cdot \ElinlossSib(\Rsc)+(m\!-\!\lambda)\cdot\ElinlossSibc(\Rsc) \right]\,.
\end{eqnarray*}

\noindent
Since $0\leq \ElinlossDprim(\Rsc) \leq 1$ \ and \ $2\cdot(q-p)^2 \leq \kl(q\|p)$ (Pinsker's inequality), we have
{\small
\begin{eqnarray}
\nonumber
\hspace{.2cm} & & \hspace{-1.2cm} (m-\lambda)\cdot2\cdot\Big(\ElinlossS(\Rsc)-\ElinlossD(\Rsc)\Big)^2 \\
\nonumber
 &=& 
 (m-\lambda) \cdot 2\cdot\Big(\tfrac{1}{m} \big[ \lambda \cdot \ElinlossSib(\Rsc)+(m\!-\!\lambda)\cdot\ElinlossSibc(\Rsc) \big] -\ElinlossD(\Rsc)\Big)^2 \\
\nonumber
 &=& 
 (m-\lambda) \cdot 2\cdot\Big(\tfrac{\lambda}{m}\big[ \ElinlossSib(\Rsc)-\ElinlossSibc(\Rsc)\big]+\big[\ElinlossSibc(\Rsc) -\ElinlossD(\Rsc)\big]\Big)^2 \\
 \nonumber
 &=&   
 (m-\lambda) \cdot 2\cdot\Big( \big(\tfrac{\lambda}{m}\big)^2 \big[ \ElinlossSib(\Rsc)-\ElinlossSibc(\Rsc)\big]^2 + 
 \big[\ElinlossSibc(\Rsc) \!-\!\ElinlossD(\Rsc)\big]^2  \\
 \nonumber & & \hspace{35mm} +\, \tfrac{2\lambda}{m} \big[ \ElinlossSib(\Rsc)-\ElinlossSibc(\Rsc)\big] \big[\ElinlossSibc(\Rsc) \!-\!\ElinlossD(\Rsc)\big] \Big)\\
\nonumber
 &\leq&   
      (m-\lambda) \cdot 2\cdot\Big( \big(\tfrac{\lambda}{m}\big)^2 + 
    \big[\ElinlossSibc(\Rsc) -\ElinlossD(\Rsc)\big]^2 + \tfrac{2\lambda}{m}\Big) \\
\nonumber
 &=&   
      2\, \lambda \cdot \Big(2 - \tfrac{\lambda}{m} - \big(\tfrac{\lambda}{m}\big)^2 \Big) + (m-\lambda) \cdot 2\cdot
      \big[\ElinlossSibc(\Rsc) -\ElinlossD(\Rsc)\big]^2 \\
\nonumber
 &\leq& 
 4\,\lambda + (m-\lambda) \cdot 2\cdot\big[\ElinlossSibc(\Rsc) -\ElinlossD(\Rsc)\big]^2 \\
 \label{eq:pudimagination}
 &\leq& 
 4\,\lambda + (m-\lambda) \cdot \kl\big(\ElinlossSibc(\Rsc) \,\|\, \ElinlossD(\Rsc)\big)
 \,.
\end{eqnarray}
}Note that $\Rsc$ does not depend on examples contained in $S_{\ibc}$. Thus,  from the point of view of $S_{\ibc}$, \ $\Rsc$ is a classical voter (not a sample-compressed one). Therefore, one can apply Lemma~\ref{lem:xi}, replacing $S\!\sim\! D^m$\,  by  
$S_{\ibc}\!\sim\! D^{m-\lambda}$, and~$f$~by~$\Rsc$. Lemma~\ref{lem:xi}, together with Equations~\eqref{eq:proof-sc-new} and~\eqref{eq:pudimagination}, gives
\begin{eqnarray*}
   & & \hspace{-2cm} \esp{(\ib,\sigb)\sim P}  \esp{S_\ib\sim D^\lambda} \esp{S_\ibc\sim D^{m-\lambda}}\!\!\!\! e^{(m-\lambda)\cdot2\cdot\big(\ElinlossS(\Rsc)-\ElinlossD(\Rsc)\big)^2} \\[-1mm]
   &\leq& e^{4\,\lambda} \cdot \esp{(\ib,\sigb)\sim P}  \esp{S_\ib\sim D^\lambda} \esp{S_\ibc\sim D^{m-\lambda}}\!\!\!\! e^{ (m-\lambda) \cdot \kl\big(\ElinlossSibc(\Rsc) \,\|\, \ElinlossD(\Rsc)\big)} \\
   &\leq& e^{4\,\lambda} \cdot \esp{(\ib,\sigb)\sim P} \esp{S_\ib\sim D^\lambda} \xi(m\!-\!\lambda) %
   \ =\ e^{4\,\lambda} \cdot \xi(m\!-\!\lambda)\,,
\end{eqnarray*}
and we are done.
\end{proof}
The next PAC-Bayesian theorem presents the generalization of McAllester's PAC-Bayesian bound (Corollary~\ref{cor:pac-bayes-McAllester}) for the sample compression case.
\begin{theorem} \label{thm:McAllester-sc}
Let $\Rcal$ be a reconstruction function that outputs sc-voters of size at most~$\lambda$ $($where $\lambda<m)$.
For any distribution $D$ on \mbox{$\Xcal\!\times\!\{-1,1\}$}, for any prior distribution~$P$ on $\ISl$ , and any $\dt\in (0,1]$, we have
\begin{equation*}
\prob{S\sim D^m}\left(
\begin{array}{l}
   \mbox{For all posteriors $Q$ \,:}\\[1mm]
    \RDGQS\,\le\,
    \RSGQS + \sqrt{ \dfrac{1}{2(m\!-\!\lambda)}\LB \KL(Q\|P) + 4\lambda + \ln\tfrac{\xi(m-\lambda)}{\dt}\RB }
\end{array}
\right) \ge\, 1 -\dt\,.
\end{equation*}
\end{theorem}

\bigskip
\begin{proof}
We apply the exact same steps as in the proof of
Theorem~\ref{thm:gen-pac-Bayes}, with $m'=m-\lambda$,\, $f=\Rsc$, \,and\, $\Dcal(q,p) = 2(q-p)^2$, we obtain
{\small
\begin{align*}
\prob{S\sim D^m}\!\!\LP \!\!
\begin{array}{l}
  \text{For all posteriors}\, Q\,\,\colon \\
 \! 2\Big( \RSGQS\!-\!\RDGQS \Big)^2\\
  \quad \le 
  \dfrac{1}{m\!-\!\lambda}\!\LB  \KL(Q\|P) +
  \ln\!\LP\dfrac{1}{\dt}\esp{S\sim D^m}\esp{(\ib,\sigb)\sim P}\!\!\!\!\!\! e^{(m-\lambda)\cdot2\cdot\big(\ElinlossS(\Rsc)-\ElinlossD(\Rsc)\big)^2}\RP\RB 
\end{array}
\!\!\!
 \RP %
 \mbox{\normalsize $\ge \, 1 -\dt$} \, .
\end{align*}
}The result then follows from Lemma~\ref{lem:LaplaceTransform} and easy calculations.
\end{proof}

All the PAC-Bayesian results presented in the preceding sections can be similarly generalized. We leave them to the reader with the exception of the PAC-Bayesian bounds that have no $\KL$, that are used in the next section, as we present the learning algorithm MinCq that minimizes the \Cbound.

\subsection{Sample-Compressed PAC-Bayesian Bounds without $\KL$}

The bounds presented in this section generalize the results presented in Section~\ref{section:Further-PAC_Bayes} to the sample compression case.  We first need to generalize the notion of self-complement (Definition~\ref{def:H_autocomp}) to sc-voters.

\begin{definition} \rm \label{def:R_autocomp}
A reconstruction function $\Rcal$ is said to be \emph{self-complemented} if for any training sequence $S\in(\Xcal\times\Ycal)^m$ and any $(\ib,\sigb)\in\Il \times \Sl$, we have
\begin{eqnarray*}
-\Rsc&=&\Rcal(S_\ib,-\sigb)\,,
\end{eqnarray*}
where, if \ $\sigb=\langle \sigma_1, .., \sigma_\lambda\rangle$,\  then  \ $-\sigb=\langle -\sigma_1, .., -\sigma_\lambda\rangle$.
\end{definition}

\subsubsection{A PAC-Bayesian Theorem for the Gibbs Risk of Sc-Voters}

\begin{theorem} \label{thm:McAllester-sc-qu}
Let $\Rcal$ be a self-complemented reconstruction function that outputs sc-voters of size at most~$\lambda$ $($where $\lambda<m)$.
For any distribution $D$ on \mbox{$\Xcal\!\times\!\{-1,1\}$}, for any prior distribution~$P$ on $\ISl$ , and any $\dt\in (0,1]$, we have
\begin{equation*}
\prob{S\sim D^m}\left(
\begin{array}{l}
   \mbox{For all posteriors $Q$ aligned on $P$\,:}\\[1mm]
    \RDGQS\,\le\,
    \RSGQS + \sqrt{ \dfrac{1}{2(m\!-\!\lambda)}\LB 4\lambda + \ln\tfrac{\xi(m-\lambda)}{\dt}\RB }
\end{array}
\right) \ge\, 1 -\dt\,.
\end{equation*}
\end{theorem}

\begin{proof}
First note that $2\!\cdot (q-p)^2=2\!\cdot ((1-q)-(1-p))^2$.
Then apply the exact same steps as in the proof of Theorem~\ref{thm:gen-pac-Bayes-qu}  with $m'=m-\lambda$,\, $f=\Rsc$, \,and\, $\Dcal(q,p) = 2(q-p)^2$ to obtain
{\small
\begin{align*}
\prob{S\sim D^m}\!\!\LP \!\!
\begin{array}{l}
  \text{For all posteriors}\, Q\,\,\text{aligned on}\,\,P\colon \\
 \! 2\Big( \RSGQS\!-\!\RDGQS \Big)^2
  \!\! \le \!
  \dfrac{1}{m\!-\!\lambda}\!\LB 
  \ln\!\LP\!\dfrac{1}{\dt}\esp{S\sim D^m}\!\esp{(\ib,\sigb)\sim P}\!\!\!\!\!\! e^{(m-\lambda)\cdot2\cdot\big(\ElinlossS(\Rsc)-\ElinlossD(\Rsc)\big)^2}\RP\!\RB 
\end{array}
\!\!\!
 \RP \\
\mbox{\normalsize $\ge \, 1 -\dt$\,.\phantom{.}} 
\end{align*}
}The result then follows from Lemma~\ref{lem:LaplaceTransform} and easy calculations.
\end{proof}

\subsubsection{A PAC-Bayesian Theorem for the Disagreement of Sc-Voters}
Given a training sequence~$S$ and a reconstruction function $\Rcal$, we define the expected disagreement of a distribution $Q$ on $\ISl$ relative to $D'$ as 
\begin{eqnarray*} 
  \dQS &\eqdef& \esp{\xb\sim D_\Xcal'} \esp{(\ib,\sigb)\sim Q} \esp{(\ib',\sigb')\sim Q}  \linloss \,\big(\,\Rsc(\xb), \Rscprim(\xb) \,\big)\\
  &=& \esp{(\ib, \ib', \sigb, \sigb') \sim Q^2}\!\! \Eaggloss{d}{D'} \Big( \Rscpaired    \Big) \, ,
\end{eqnarray*}
where 
\begin{eqnarray*}
Q^2(\ib, \ib', \sigb, \sigb')&\eqdef &Q(\ib,\sigb)\cdot Q(\ib',\sigb')\,,\\
\Rscpaired(x) &\eqdef& \tuple{\Rsc(x), \Rscprim(x)}\,.
\end{eqnarray*}

\noindent
Thus, $\overline{\Rcal}$ is a new reconstruction function that outputs an \emph{sc-paired-voter} which is the sample-compressed version of the paired-voter of Definition~\ref{def:paired_voters}. From there, we adapt Corollary~\ref{cor:pac-bayes-McAllester-qu2} to sc-voters, and we obtain the following PAC-Bayesian theorem. This result bounds $\dDQS$ for posterior distributions $Q$ aligned on a prior distribution $P$.

\begin{theorem}
 \label{thm:aligned-sc-disagreement}
Let $\Rcal$ be a self-complemented reconstruction function that outputs sc-voters of size at most~$\lambda$ $($where $\lambda<\lfloor\frac{m}{2}\rfloor)$. For any distribution $D$ on \mbox{$\Xcal\!\times\!\{-1,1\}$}, for any prior distribution $P$ on $\ISl$, and any $\dt\in (0,1]$, we have
\begin{equation*}
\prob{S\sim D^m}\left(
\begin{array}{l}
  \mbox{For all posteriors $Q$ aligned on $P$}: \\[1mm]
    \dDQS \,\ge\,
    \dSQS - \sqrt{ \dfrac{1}{2(m\!-\!2\, \lambda)}\LB 8 \lambda + \ln\tfrac{\xi(m-2\, \lambda)}{\dt}\RB }
\end{array}
\right) \ge\, 1 -\dt\,.
\end{equation*}
\end{theorem}

\begin{proof}
Let  $P^2(\ib, \ib', \sigb, \sigb')\eqdef P(\ib,\sigb)\cdot P(\ib',\sigb')$.
Now note that $2\cdot (q-p)^2=2\cdot ((1\hspace{-0.3mm}-\hspace{-0.3mm}q)-(1\hspace{-0.3mm}-\hspace{-0.3mm}p))^2$.
Then apply the exact same steps as in the proof of
Theorem~\ref{thm:pac-Bayes-q2-qu} 
with $m'=m-2\lambda$,\, \linebreak $f_{ij}=\Rscpaired$ \,and\, $\Dcal(q,p) = 2(q-p)^2$ to obtain
{\small
\begin{align*}
\prob{S\sim D^m}\!\!\LP \!\!
\begin{array}{l}
  \text{For all posteriors}\, Q\,\,\text{aligned on}\,\,P\colon \\
 \! 2\Big( \dSQS\!-\!\dDQS \Big)^2\!\!
  \! \le \!
  \dfrac{1}{m}\!\LB 
  \ln\!\LP\dfrac{1}{\dt}\esp{S\sim D^m}
  \hspace{-.75cm} \esp{\qquad (\ib, \ib', \sigb, \sigb') \sim P^2} \hspace{-11.5mm} e^{m\cdot 2\cdot\big(\Eaggloss{d}{S}(\Rscpaired)-\Eaggloss{d}{D}(\Rscpaired)\big)^2}\RP\RB 
\end{array}
\!\!\! \RP \\ 
 \mbox{\normalsize $\ge \ 1 -\dt$} \, .
\end{align*}
}Calculations similar to the ones of the proof of Lemma~\ref{lem:LaplaceTransform} 
 (with $\lambda$ replaced by $2\lambda$) give
\begin{eqnarray*}
\esp{S\sim D^m}
  \hspace{-.75cm} \esp{\qquad (\ib, \ib', \sigb, \sigb') \sim P^2} \hspace{-11.5mm} e^{(m-2\lambda)\cdot 2\cdot\big(\Eaggloss{d}{S}(\Rscpaired)-\Eaggloss{d}{D}(\Rscpaired)\big)^2} 
&\leq& 
e^{8\,\lambda} \cdot \xi(m\!-\!2\lambda)\,.
\end{eqnarray*}
Therefore, we have
{\small
\begin{equation*}
\prob{S\sim D^m}\!\!\LP \!\!
\begin{array}{l}
  \text{For all posteriors}\, Q\,\,\text{aligned on}\,\,P\colon \\
 \! 2\Big( \dSQS\!-\!\dDQS \Big)^2\!\!
  \le
  \dfrac{1}{m\!-\!2\lambda}\!\LB 
 8\lambda +\ln\tfrac{ \xi(m-2\, \lambda)}{\dt}\RB 
\end{array}
\!\!\! \RP \ge \ 1 -\dt\,.
\end{equation*}
}and the result is obtained by isolating $\dDQS$ in the inequality.
\end{proof}

\subsubsection{A Sample Compression Bound for the Risk of the Majority Vote}  
Let us now exploit Theorems~\ref{thm:McAllester-sc-qu} and \ref{thm:aligned-sc-disagreement}, together with the \Cbound of Theorem~\ref{thm:C-bound}, to obtain a bound on the risk on a majority vote with kernel functions as voters.  Given any similarity function (possibly a kernel) $k:\Xcal\times\Xcal\rightarrow [-1,1]$ and a training sequence size of $m$, 
we consider a majority vote of sc-voters of compression size at most $1$ given by the following reconstruction function,
\begin{equation*}
\Rcal_{k}\big(S_\ib, \langle\sigma\rangle\big)(x) \ \eqdef \ 
\begin{cases}
\sigma & \mbox{ if \,$\ib\!=\!\langle\,\rangle$,}\\
\sigma \cdot k(x_i, x) & \mbox{ otherwise ( $\ib \!=\!\langle i \rangle$ ),}
\end{cases}
\end{equation*}
where $\ib \in \Ical_1 = \{ \langle\,\rangle, \langle 1 \rangle, \langle 2 \rangle, \ldots, \langle m \rangle\}$ and $\langle\sigma\rangle\in\Sigma_1$ (thus, $\sigma \in \{-1,1\}$). 
Here, the elements of sets $\Ical_1$ and $\Sigma_1$ are obtained from Equation~\eqref{eq:ISlambda}, with $\lambda=1$. Note that $\Rcal_{k}$  is self-complemented (Definition~\ref{def:R_autocomp}) because
$-\Rcal_{k}\big(S_\ib, \langle\sigma\rangle\big)=\Rcal_{k}\big(S_\ib, \langle-\sigma\rangle\big)$ \ for any $(\ib,\sigb)$.

Once the training sequence $S\sim D^m$ is observed, the (self-complemented) reconstruction function $\Rcal_k$ gives rise to the following set of $2m\!+\!2$ sc-voters,
\begin{equation*}
\HSlk \ \eqdef \ \Big\{ b(\cdot), k(x_1,\cdot), k(x_2,\cdot), \ldots, k(x_m,\cdot),  -b(\cdot), -k(x_1,\cdot), -k(x_2,\cdot), \ldots, -k(x_m,\cdot) \Big\}\,,
\end{equation*}
where $b:\Xcal\rightarrow \{1\}$ is a ``dummy voter'' that always outputs $1$ and allows introducing a \emph{bias} value into the majority vote classifier.
Note that $\HSlk$ is a self-complemented set of sc-voters, and the margin of the majority vote given by the distribution $Q$ on $\HSlk$ is
\begin{equation*}
M_{Q,S}(x,y) \ \eqdef \ y \left( Q\big(\,b(\cdot)\,\big) - Q\big(\,-\!b(\cdot)\,\big)
+ \sum_{i=1}^m \left[ Q\big(\,k(x_i,\cdot)\,\big) -  Q\big(\,-\!k(x_i,\cdot)\,\big) \right] k(x_i,x) 
\right)\,.
\end{equation*}
Consequently, the empirical first and second moments of this margin are
\begin{equation*} %
\momentone(\MQS{S}) \ = \ \frac{1}{m}\sum_{i=1}^m M_{Q,S} (x_i,y_i),
\quad\mbox{\normalsize and}\quad 
\momenttwo(\MQS{S}) \ = \ \frac{1}{m} \sum_{i=1}^m \Big[M_{Q,S} (x_i,y_i)\Big]^2\,.
\end{equation*}
Hence,  the empirical Gibbs risk and the empirical expected disagreement can be expressed by
\begin{equation} \label{eq:RDkernels}
\RSGQS \ = \ \frac{1}{2} \left(1-\momentone(\MQS{S}) \right),
\quad\mbox{\normalsize and}\quad 
\dSQS \ = \ \frac{1}{2} \left(1-\momenttwo(\MQS{S}) \right)\,.
\end{equation}

Thus, we obtain the following bound on the risk of a majority vote of kernel voters $R_D(B_{Q,S})$  for aligned posteriors $Q$.  

\begin{pacbound}{3'}\label{bound:variancebinouille-qu-sc}
Let $k:\Xcal\times\Xcal\rightarrow [-1,1]$.
For any distribution $D$ on \mbox{$\Xcal\!\times\!\{-1,1\}$}, for any prior distribution $P$ on $\HSlk$, and any $\dt\in (0,1]$, we have
\begin{equation*}
  \prob{S\sim D^m}\!\left(\!
  \begin{array}{l}
	\forall\,Q \mbox{ aligned on } P\,:\\
   \RDBQS \,\, \leq \,\, 1-\dfrac{\big( \,1- 2\cdot \overline{r} \,\big)^2}{ 1 - 2\cdot \underline{d} } \ = \
   1-\dfrac{\big( \, \underline{\mu_1}_{} \, \big)^2 }{ \overline{\mu_2} } 
   \end{array}
\right) \ge\, 1 -\dt\,,
\end{equation*}
where \\[-8mm]
\begin{align*}
 \overline{r} &\, \eqdef\, \min\left(\tfrac{1}{2},\, \RSGQS + \sqrt{ \mbox{\small $\frac{1}{2(m-1)}\LB 4 + \ln\frac{\xi(m-1)}{\dt/2}\RB$}} \right),
 \\
 \underline{d} &\, \eqdef\, \max\left(0,\, \dSQS - \sqrt{ \mbox{\small $\frac{1}{2(m-2)}\LB 8 + \ln\frac{\xi(m-2)}{\dt/2}\RB$} }\right),\\
 \underline{\mu_1} &\, \eqdef\, \max\left(0,\,\momentone(\MQS{S}) - \sqrt{ \mbox{\small $\frac{2}{m-1}\LB 4 + \ln\frac{\xi(m-1)}{\dt/2}\RB$} }\right),
 \\
 \overline{\mu_2} &\, \eqdef\, \min\left(1,\, \momenttwo(\MQS{S}) + \sqrt{  \mbox{\small $\frac{2}{m-2}\LB 8 + \ln\frac{\xi(m-2)}{\dt/2}\RB$} }\right).
\end{align*}

\end{pacbound}
\begin{proof}
The proof is almost identical to the one of PAC-Bound~\ref{bound:variancebinouille-qu}, except that it relies on sample-compressed PAC-Bayesian bounds.
 Indeed, the inequality is a consequence of Theorem~\ref{thm:C-bound}, as well as Theorems~\ref{thm:McAllester-sc-qu} and~\ref{thm:aligned-sc-disagreement}. The equality 
 $1-\frac{( 1- 2\cdot \overline{r} )^2}{ 1 - 2\cdot \underline{d} } = 1-\frac{(\underline{\mu_1}_{} )^2 }{ \overline{\mu_2} } $
     is a direct application of Equation~\eqref{eq:RDkernels}.
\end{proof}

PAC-Bounds~\ref{bound:variancebinouille-qu} and~\ref{bound:variancebinouille-qu-sc} are expressed in two forms.
The first form relies on bounds on the Gibbs risk and the expected disagreement (denoted $\overline{r}$ and $\underline{d}$). The second form relies on bounds on the first and second moments of the margin (denoted $ \underline{\mu_1}$ and $\overline{\mu_2}$). This latter form is used to justify the learning algorithm presented in Section~\ref{sect:mincq}. 

\section{MinCq: Learning by Minimizing the \Cbound}
\label{sect:mincq}
In this section, we propose a new algorithm, that we call MinCq, for constructing a weighted majority vote of voters. One version of this algorithm is designed for the supervised inductive framework and minimizes the \Cbound. A second version of MinCq that minimizes the \Cbound in the transductive (or semi-supervised) setting can be found in~\cite{lmr-11}. Both versions can be expressed as quadratic programs on positive semi-definite matrices.

\smallskip
As is the case for Boosting algorithms \citep{schapire99}, MinCq is designed to output a $Q$-weighted majority vote of voters that perform rather poorly individually and, consequently, are often called weak learners. Hence, the decision of each vote is based on a small majority (\ie, with a Gibbs risk just a bit lower than 1/2). Recall that in situations where the Gibbs risk is high (\ie, the first moment of the margin is close to $0$), the  \Cbound can nevertheless remain small if the voters of the majority vote are maximally uncorrelated.

\smallskip
Unfortunately, minimizing the empirical value of the \Cbound tends to overfit the data. To overcome this problem, MinCq uses a distribution~$Q$ of voters which is constrained to be quasi-uniform (see Equation~\ref{eq:quasi-uniform}) and for which the first moment of the margin is forced to be not too close to $0$. More precisely, the value $\momentone(\MQ{S})$ is constrained to be bigger than some strictly positive constant $\mu$. This $\mu$ then becomes a hyperparameter of the algorithm that has to be fixed by cross-validation, as the parameter $C$ is for SVM. This new learning strategy is justified by PAC-Bound~\ref{bound:variancebinouille-qu}, dedicated to quasi-uniform posteriors\footnote{PAC-Bound~\ref{bound:variancebinouille-qu} is dedicated to posteriors $Q$ that are aligned on a prior distribution $P$, but in this section we always consider that the prior distribution $P$ is uniform, thus leading to a quasi-uniform posterior~$Q$.}, and PAC-Bound~\ref{bound:variancebinouille-qu-sc}, that is specialized to kernel voters. Hence, MinCq can be viewed as the algorithm that simply looks for the majority vote of margin at least $\mu$ that minimizes PAC-Bound~\ref{bound:variancebinouille-qu} (or PAC-Bound~\ref{bound:variancebinouille-qu-sc} in the sample compression case).

\smallskip
MinCq is also justified by two important properties of quasi-uniform majority votes. First, as we shall see in Theorem~\ref{thm:bayes-equiv}, there is no generality loss when restricting ourselves to quasi-uniform distributions. Second, as we shall see in Theorem~\ref{thm:equiv-fixe-marge}, for any margin threshold $\mu> 0$ and any quasi-uniform distribution $Q$ such that $\momentone(\MQ{S}) \geq \mu$, there is another quasi-uniform distribution $Q'$ whose margin is exactly~$\mu$ that achieves the same majority vote and therefore has the same \Cbound value.

\smallskip
Thus, to minimize the \Cbound, the learner must substantially reduce the variance of the margin distribution -- \ie, $\momenttwo(\MQ{S})$ -- while maintaining its first moment -- \ie, $\momentone(\MQ{S})$ -- over the threshold $\mu$. Many learning algorithms actually exploit this strategy in different ways. Indeed, the variance of the margin distribution is controlled by~\citet{b-01} for producing random forests, by~\citet{citeulike:5895440} in the transfer learning setting, and by~\citet{MDBoost2010Shen} in the Boosting setting. Thus, the idea of minimizing the variance of the margin is well-known and used. We propose a new theoretical justification for all these types of algorithms and propose a novel learning algorithm, called MinCq, that directly minimizes the \Cbound.

\subsection{From the \Cbound to the MinCq Learning Algorithm}
We only consider learning algorithms that construct majority votes based on a (finite) self-complemented hypothesis space $\Hcal=\{f_1, \ldots, f_{2n}\}$ of real-valued voters. Recall that these voters can be classifiers such as decision stumps or can be given by a kernel $k$ evaluated on the examples of $S$ such as $f_i(\cdot) = k(\xb_i, \cdot)$.

We consider the second form of the \Cbound, which relies on the first two moments of the margin of the majority vote classifier (see Theorem~\ref{thm:C-bound}):
\begin{equation*}
\CQ \ =\  1-\frac{\Big(\momentone(\MQ{D'})\Big)^2}{\phantom{\Big(} \momenttwo(\MQ{D'})\phantom{\Big)}}\,.
\end{equation*}
Our first attempts to minimize the \Cbound confronted us with two problems.\\[2mm]
\emph{Problem}~1:  an empirical \Cbound minimization without any regularization tends to overfit the %
data.\\[2mm]
\emph{Problem}~2: most of the time, the distributions $Q$ minimizing the \Cbound $\CSQ$  are such that both $\momentone(\MQ{S})$ and $\momenttwo(\MQ{S})$ are very close to $0$. Since $\CSQ=1-(\momentone(\MQ{S}))^2/\momenttwo(\MQ{S})\,
$, this gives a $0/0$ numerical instability. 
Since $(\momentone(\MQ{D}))^2/\momenttwo(\MQ{D})\,$ can only be empirically estimated by $(\momentone(\MQ{S}))^2/\momenttwo(\MQ{S})$, Problem~2 amplifies Problem~1.

\smallskip
A natural way to resolve Problem~1 is to restrict ourselves to quasi-uniform distributions, \ie, distributions that are aligned on the uniform prior (see Section~\ref{sec:aligneetquasi-unif} for the definition). In Section~\ref{section:Further-PAC_Bayes}, we show that with such distributions, one can upper-bound the Bayes risk without needing a $\KL$-regularization term. Hence, according to this PAC-Bayesian theory, these distributions have some ``built-in'' regularization effect that should prevent overfitting. Section~\ref{sec:sample-compression} generalizes these results to the sample compression setting, which is necessary in the case where voters such as kernels are defined using the training set.

\smallskip
The next theorem shows that this restriction on $Q$ does not reduce the set of possible majority votes.

\begin{theorem}\label{thm:bayes-equiv}
Let $\Hcal$ be a self-complemented set.
For all distributions $Q$ on $\Hcal$, there exists a quasi-uniform distribution $Q'$ on $\Hcal$ that gives the same majority vote as $Q$, and that has the same empirical and true \Cbound values, i.e.,
\begin{eqnarray*}
 B_{Q'} = B_Q\,, \quad
 \CSQprime = \CSQ \quad \mbox{ and } \quad \CDQprime = \CDQ \,.
\end{eqnarray*}
\end{theorem}
\begin{proof}
Let $Q$ be a distribution on $\Hcal\!=\!\{f_1, \ldots, f_{2n}\}$, let \mbox{$M \eqdef \max_{ i \in \{ 1,.., n \} }\! | Q(f_{i+n})-Q(f_i) | $}, and let $Q'$ be defined as 
$$Q'(f_i)\ \eqdef\ \frac{1}{2n} +\frac{Q(f_i)\,-\,Q(f_{i+n})}{2nM}\,,$$ 
where the indices of $f$ are defined \mbox{modulo $2n$} (\ie, $f_{(i+n)+n}=f_i$). Then it is easy to show that $Q'$ is a quasi-uniform distribution. Moreover, for any example $\xb\in \Xcal$, we have
\begin{eqnarray*}
\esp{f\sim Q'}f(\xb) &\ \eqdef \ &\sum_{i=1}^{2n} Q'(f_i)\,f_i(\xb) %
\ = \ \sum_{i=1}^{n} (Q'(f_i)-Q'(f_{i+n}))\,f_i(\xb) \\
 &=& \sum_{i=1}^{n} \frac{2Q(f_i)-2Q(f_{i+n})}{2nM}\,f_i(\xb) %
 \ = \ \frac{1}{nM}\sum_{i=1}^{2n} Q(f_i)\,f_i(\xb) \\
 &=& \frac{1}{nM} \esp{f\sim Q}f(\xb)\,.
\end{eqnarray*}
Since $nM>0$, this implies that $B_{Q'}(\xb)=B_Q(\xb)$ for all $\xb\in\Xcal$. It also shows that
$M_{Q'}(\xb,y)\!=\!\frac{1}{nM} M_Q(\xb,y)$, which implies that
$\left(\momentone(\MQprime{D'})\right)^2\!\!=\!\left(\frac{1}{nM} \momentone(\MQ{D'})\right)^2$ \,and $\momenttwo(\MQprime{D'})\!=\left(\frac{1}{nM}\right)^2\!\! \momenttwo(\MQ{D'})$ for both
$D'=D$ and $D'=S$.\\[2mm]
The theorem then follows from the definition of the \Cbound.
\end{proof}

Theorem~\ref{thm:bayes-equiv} points out a nice property of the \Cbound: different distributions $Q$ that give rise to a same majority vote have the same (real and empirical) \Cbound values. Since the \Cbound is a bound on majority votes, this is a suitable property.
Moreover, PAC-Bounds~\ref{bound:variancebinouille-qu} and~\ref{bound:variancebinouille-qu-sc}, together with Theorem~\ref{thm:bayes-equiv}, indicate that restricting ourselves to quasi-uniform distributions is a natural solution to the problem of overfitting (see Problem~1). Unfortunately, Problem~2 remains %
since a consequence of the next theorem is that, among all the posteriors $Q$ that minimize the $\Cbound$, there is always one whose empirical margin $\momentone(\MQ{S})$ is as close to  $0$ as we want.

\begin{theorem}\label{thm:equiv-fixe-marge}
Let $\Hcal$ be a self-complemented set.
For all $\mu\in(0,1]$ and
for all quasi-uniform distributions $Q$ on $\Hcal$ having an empirical margin $\momentone(\MQ{S})\geq \mu$, there exists a quasi-uniform distribution $Q'$ on $\Hcal$, having an empirical margin equal to~$\mu$, such that $Q$ and $Q'$ induce the same majority vote and have the same empirical and true $\Cbound$ values, i.e.,
\begin{equation*}
\momentone(\MQprime{S}) =\mu \, , \ \  B_{Q'} = B_Q \,, \ \
\CSQprime = \CSQ \ \  \mbox{ and } \ \  \CDQprime= \CDQ .
\end{equation*}
\end{theorem}
\begin{proof}
Let $Q$ be a quasi-uniform distribution on $\Hcal\!=\!\{f_1, \ldots, f_{2n}\}$ such that $\momentone(\MQ{S})\geq \mu$. We define $Q'$ as
\vspace{-2mm}
\begin{equation*}
Q'(f_i)\ \eqdef\  \frac{\mu}{\momentone(\MQ{S})}\cdot Q(f_i) \,+\, \left(1-\frac{\mu}{\momentone(\MQ{S})}\right)\cdot 1/2n\, , \quad i \in \{1,..,2n \} \,. 
\end{equation*}
Clearly $Q'$ is a quasi-uniform distribution since it is a convex combination of a quasi-uniform distribution and the uniform one.
Then, similarly as in the proof of Theorem~\ref{thm:bayes-equiv}, one can easily show that $\esp{f\sim Q'}f(\xb) = \frac{\mu}{\momentone(\MQ{S})}\esp{f\sim Q}f(\xb)$, which implies the result.
\end{proof}
Training set bounds (such as VC-bounds for example)  are known to degrade when the capacity of classification increases. As shown by Theorem~\ref{thm:equiv-fixe-marge} for the majority vote setting, this capacity increases as $\mu$ decreases to 0. Thus, we expect that any training set bound degrades for small $\mu$. This is  not the case for the $\Cbound$ itself, but the $\Cbound$ is not a training set bound. To obtain a training set bound, we have to relate the empirical value $\CSQ$ to the true one $\CDQ$, which is done via PAC-Bounds~\ref{bound:variancebinouille-qu} and~\ref{bound:variancebinouille-qu-sc}. In these bounds, there is indeed a degradation as $\mu$ decreases because the true \Cbound is of the form $1- (\momentone(\MQ{D}))^2/\momenttwo(\MQ{D})$. Since $\mu=\momentone(\MQ{S})$, and because a small $\momentone(\MQ{S})$ tends to produce small $\momenttwo(\MQ{S})$, the bounds on $\CDQ$ given $\CSQ$ that outcomes from PAC-Bounds~\ref{bound:variancebinouille-qu} and~\ref{bound:variancebinouille-qu-sc} are therefore much looser for small $\mu$ because of the 0/0 instability.
As explained in the introduction of the present section, one way to overcome the instability identified in Problem~2 is to restrict ourselves to quasi-uniform distributions whose empirical margin is greater or equal than some threshold $\mu$. Interestingly, thanks to Theorem~\ref{thm:equiv-fixe-marge}, this is equivalent to restricting ourselves to distributions having empirical margin \emph{exactly equal to} $\mu$. From Theorems~\ref{thm:C-bound} and~\ref{thm:equiv-fixe-marge}, it then follows that \emph{minimizing the \Cbound, under the constraint $\momentone(\MQ{S})\!\geq\! \mu$, is equivalent to minimizing $\momenttwo(\MQ{S})$, under the constraint $\momentone(\MQ{S})\!=\! \mu$}\,, from this observation, and the fact that minimizing PAC-Bounds~\ref{bound:variancebinouille-qu} and~\ref{bound:variancebinouille-qu-sc} is equivalent to minimizing the empirical \Cbound $\CSQ$, we can now define the algorithm MinCq.

\smallskip
In this section, $\mu$ always represents a restriction on the margin. Moreover, we say that a value $\mu$ is $D'$-\emph{realizable} if there exists some quasi-uniform distribution~$Q$ such that $\momentone(\MQ{D'})=\mu$. The proposed algorithm, called MinCq, is then defined as follows.

\begin{definition}[MinCq Algorithm]\rm  \label{def:MinCq}
Given a self-complemented set $\Hcal$ of voters, a training set $S$, and a $S$-realizable $\mu>0$, among all
quasi-uniform distributions $Q$ of empirical margin $\momentone(\MQ{S})$ exactly equal to $\mu$,
 the algorithm MinCq consists in finding one that minimizes $\momenttwo(\MQ{S})$.
\end{definition}

This algorithm can be translated as a simple quadratic program (QP) that has only $n$ variables (instead of $2n$), and thus can be easily solved by any QP solver. In the next subsection, we explain how the algorithm of Definition~\ref{def:MinCq} can be turned into a~QP.

\subsection{MinCq as a Quadratic Program}
Given a training set $S$, and a self-complemented set $\Hcal$ of voters $\{f_1, f_2, \ldots, f_{2n}\}$, let

\begin{equation*} %
 \Mcal_{i}  \eqdef  \esp{(\xb,y)\sim S} y\,f_i(\xb)\,\quad \quad \mbox{and} \quad \quad
 \Mcal_{i,j}  \eqdef \esp{(\xb,y)\sim S} f_i(\xb)\,f_j(\xb)\,.
\end{equation*}

Let $\Mb$ be a symmetric $n \times n$ matrix, $\ab$ be a column vector of $n$ elements, and $\mb$ be a column vector of $n$ elements defined by

\begin{align}
\Mb \eqdef
\begin{bmatrix}
\Mcal_{1,1} & \Mcal_{1,2} & \ldots & \Mcal_{1,n} \\
\Mcal_{2,1} & \Mcal_{2,2} & \ldots & \Mcal_{2,n} \\
\vdots     & \vdots     & \ddots & \vdots     \\
\Mcal_{n,1} & \Mcal_{n,2} & \ldots & \Mcal_{n,n} \\
\end{bmatrix}\,,
\quad  
\ab \eqdef
\begin{bmatrix}
\vspace{2mm}
\frac{1}{n}\sum_{i=1}^n \Mcal_{i,1} \\
\frac{1}{n}\sum_{i=1}^n \Mcal_{i,2} \\
\vdots    \\
\frac{1}{n}\sum_{i=1}^n \Mcal_{i,n} \\
\end{bmatrix}\,,
\quad \mbox{and} \quad
\mb \eqdef
\begin{bmatrix}
\Mcal_{1} \\
\Mcal_{2} \\
\vdots    \\
\Mcal_{n} \\
\end{bmatrix}.
\label{eq:mincqmat}
\end{align}

Finally,  let $\qb$ be the column vector of $n$ QP-variables, where each element $q_i$ represents the weight $Q(f_i)$.

Using the above definitions and the fact that $\Hcal$ is self-complemented, one can show that
\begin{equation*}
\Mcal_{i+n} = -\Mcal_i\,, \quad \Mcal_{i+n,j} = \Mcal_{i,j+n} = -\Mcal_{i,j}\,,
\quad \mbox{and} \quad q_{i+n} = \frac{1}{n} - q_i\,.
\end{equation*}
Moreover, it follows from the definitions of the first two moments of the margin $\momentone(\MQ{S})$ and $\momenttwo(\MQ{S})$ (see  Equations~\ref{eq:margin_moment_one} and~\ref{eq:margin_moment_two}) that
\begin{eqnarray*}
 \momentone(\MQ{S}) \ =\ \sum_{i=1}^{2n} q_i\, \Mcal_i\,, \quad \mbox{ and } \quad
\momenttwo(\MQ{S}) \ =\  \sum_{i=1}^{2n}\sum_{j=1}^{2n} q_iq_j\,\Mcal_{i,j} \,.
\end{eqnarray*}

\medskip
As MinCq consists in finding the quasi-uniform distribution $Q$ that minimizes $\momenttwo(\MQ{S})$, with a margin $\momentone(\MQ{S})$ exactly equal to the hyperparameter $\mu$, let us now rewrite $\momenttwo(\MQ{S})$ and $\momentone(\MQ{S})$ using the vectors and matrices defined in Equation~\eqref{eq:mincqmat}. It follows that

{\small
\begin{eqnarray}%
\nonumber
\momenttwo(\MQ{S}) \ = \ \sum_{i=1}^{2n}\sum_{j=1}^{2n} q_i q_j\, \Mcal_{i,j} %
\nonumber
&=& \sum_{i=1}^{n}\sum_{j=1}^{n} \Big[ q_i q_j - q_{i\!+\!n} q_j - q_i q_{j\!+\!n} + q_{i\!+\!n} q_{j\!+\!n} \Big] \Mcal_{i,j} \\[2mm]
\nonumber
&=& \sum_{i=1}^{n}\sum_{j=1}^{n} \left[4q_i q_j-\frac{4}{n}\,q_i+\frac{1}{n^2}\right] \Mcal_{i,j} \\[2mm]
\nonumber
&=& 4 \sum_{i=1}^{n}\sum_{j=1}^{n}q_i q_j\ \Mcal_{i,j}  - \frac{4}{n} \sum_{i=1}^{n}\sum_{j=1}^{n}q_i\ \Mcal_{i,j} + \frac{1}{n^2} \sum_{i=1}^{n}\sum_{j=1}^{n}\Mcal_{i,j} \\[2mm]
\label{eq:quatredetrop}
&=& 4\Big( \qb^\top \, \Mb \, \qb \ - \ \ab^\top \ \qb\Big)
    \  + \ \frac{1}{n^2}\sum_{i=1}^{n}\sum_{j=1}^{n}\Mcal_{i,j}\,,
\end{eqnarray}
}and
{\small
\begin{eqnarray*}
 \momentone(\MQ{S}) \ = \ \sum_{i=1}^{2n} q_i \Mcal_i 
 &=& \sum_{i=1}^{n} \big(q_i - q_{i+n}\big) \Mcal_i \ =\ \sum_{i=1}^{n} \left(2q_i-\frac{1}{n}\right) \Mcal_i 
 \ = \ 2\sum_{i=1}^{n} q_i\ \Mcal_i - \frac{1}{n}\sum_{i=1}^{n} \Mcal_i \\
 &=& 2{\mb}^\top \qb - \frac{1}{n} \sum_{i=1}^{n}  \Mcal_i\,.
\end{eqnarray*}
} %

As the objective function $\momenttwo(\MQ{S})$ and the constraint $ \momentone(\MQ{S}) =\mu$ of the QP can be defined using only $n$ variables, there is no need to consider in the QP the weights of the last $n$ voter. These weights can always be recovered from the $n$ first, because 
$
\mbox{$q_{i+n} \,= \, \tfrac{1}{n}- q_i$}, \  \mbox{for any }\,i\,.
$
 Note however that to be sure that the solution of the QP has the quasi-uniformity property, we have to add the following constraints to the program:
$$
q_{i} \ \in \ [0, \tfrac{1}{n}] \quad\quad\quad\quad\quad\quad\ \mbox{for any }\,i\,.
$$

Note that  the multiplicative constant $4$ and the additive constant $\frac{1}{n^2}\sum_{i=1}^{n}\sum_{j=1}^{n}\Mcal_{i,j}$ from Equation~\eqref{eq:quatredetrop} can be omitted, as the optimal solution will stay the same.
From all that precedes and given any $S$-realizable~$\mu$, MinCq solves the optimization problem described by Program~1.%

\floatstyle{boxed} \floatname{algorithm}{Program}
\begin{algorithm}[H]%
\begin{algorithmic}%
\vspace{1mm}
\STATE{\textbf{Solve} $\quad\quad\arg\!\min_{\qb} \quad \qb^\top \  \Mb \ \qb
     \,\,-\,\, {\ab}^\top\  \qb$
}
\vspace{3mm}
\STATE{\hspace{2mm}\textbf{under constraints :} ${\mb}^\top\   \qb= \frac{\mu}{2} \!+\! \frac{1}{2n}\sum_{i=1}^{n} \Mcal_i$
}
\vspace{2mm}
\STATE{\hspace{28.7mm}\textbf{and :} $0 \leq q_i \leq \frac{1}{n}\quad\forall i\in\{1,\ldots,n\}$
}
\end{algorithmic}
\caption[]{:  \textbf{MinCq} - \emph{a quadratic program for classification}}
\end{algorithm}
To prove that Program~1 %
is a quadratic program, it suffices to show that $\Mb$ is a positive semi-definite matrix. This is a direct consequence of the fact that each $\Mcal_{i,j}$ can be viewed as a scalar product, since%
\begin{eqnarray*}
\Mcal_{i,j} =\  \scriptstyle \left(\!\sqrt{\frac{1}{|S|}}\, f_i(\xb)\!\right)_{\xb\in S_\Xcal}\ \,  \mbox{\huge{$_\cdot$}}\  \ \left(\!\sqrt{\frac{1}{|S|}}\, f_j(\xb)\!\right)_{\xb\in S_\Xcal}\,,
\quad \mbox{ where $S_\Xcal\eqdef \{\xb\colon (\xb,y)\in S\}$.}
\end{eqnarray*}

\bigskip
\noindent
Finally, the $Q$-weighted majority vote output by MinCq is
\begin{eqnarray*}
B_Q(\xb) \ = \ \sgn\LB \esp{f\sim Q} f(\xb) \RB
& = &
 \sgn\LB\sum_{i=1}^{2n} q_i f_i(x)    \RB
 \ = \ 
 \sgn\LB\sum_{i=1}^{n} q_i f_i(x) + \sum_{i=n+1}^{2n} q_i f_i(x)   \RB\\
  & = &
 \sgn\LB\sum_{i=1}^{n} q_i f_i(x) + \sum_{i=1}^{n} (\tfrac{1}{n}-q_i) \cdot -f_i(x)   \RB\\
 & = &
 \sgn\LB\sum_{i=1}^{n} (2q_i-\tfrac{1}{n}) f_i(x)    \RB.
\end{eqnarray*}

\subsection{Experiments}

We now compare MinCq to state-of-the-art learning algorithms in three different contexts: \emph{handwritten digits recognition}, \emph{classical binary classification tasks}, and \emph{Amazon reviews sentiment analysis}. A \emph{context}~\citep{lacoste-2012} represents a distribution on the different tasks a learning algorithm can encounter, and a sample from a context is a collection of data sets.

For each context, each data set is randomly split into a training set $S$ and a testing set $T$. When hyperparameters have to be chosen for an algorithm, 5-fold cross-validation is run on the training set $S$, and the hyperparameter values that minimize the mean cross-validation risk are chosen. Using these values, the algorithm is trained on the whole training set $S$, and then evaluated on the testing set $T$.

For the first two contexts, we compare MinCq using decision stumps as voters (referred to as StumpsMinCq), MinCq using RBF kernel functions $k(\xb, \xb') = \exp(-\gamma||\xb - \xb'||^2)$ as voters (referred to as RbfMinCq), AdaBoost~\citep{schapire99} using decision stumps (referred to as StumpsAdaBoost), and the soft-margin Support Vector Machine (SVM)~\citep{DBLP:journals/ml/CortesV95} using the RBF kernel, referred to as RbfSVM. For the last context, we compare MinCq using linear kernel functions $k(\xb, \xb') = \xb \cdot \xb'$ as voters (referred to as LinearMinCq), and the SVM using the same linear kernel, referred to as LinearSVM.

For the three variants of MinCq, the quadratic program is solved using CVXOPT~\citep{dahvan07}, an off-the-shelf convex optimization solver.

\begin{description}
\item[StumpsAdaBoost:] For StumpsAdaBoost, we use decision stumps as weak learners. For each attribute, 10 decision stumps (and their complement) are generated, for a total of 20 decision stumps per attribute. The number of boosting rounds is chosen among the following 15 values: 10, 25, 50, 75, 100, 125, 150, 175, 200, 225, 250, 275, 500, 750 and 1000.
\item[StumpsMinCq:] For StumpsMinCq, we use the same 10 decision stumps per attribute as for StumpsAdaBoost. Note that we do not need to consider the complement stumps in this case, as MinCq automatically considers self-complemented sets of voters. MinCq's hyperparameter $\mu$ is chosen among 15 values between $10^{-4}$ and $10^{0}$ on a logarithmic scale.
\item[RbfSVM:] The $\gamma$ hyperparameter of the RBF kernel and the $C$ hyperparameter of the SVM are chosen among 15 values between $10^{-4}$ and $10^1$ for $\gamma$, and among 15 values between $10^0$ and $10^8$ for $C$, both on a logarithmic scale.
\item[RbfMinCq:] For RbfMinCq, we consider 15 values of $\mu$ between $10^{-4}$ and $10^{-2}$ on a logarithmic scale, and the same 15 values of $\gamma$ as in SVM for the RBF kernel voters. 
\item[LinearSVM:] When using the linear kernel, the $C$ parameter of the SVM is chosen among 15 values between $10^{-4}$ and $10^{2}$, on a logarithmic scale. All SVM experiments are done using the implementation of~\cite{scikit-learn}.
\item[LinearMinCq:] For LinearMinCq, we consider 15 values of $\mu$ between $10^{-4}$ and $10^{-2}$ on a logarithmic scale.
\end{description}

When using the RBF kernel for the SVM or MinCq, each data set is normalized using a hyperbolic tangent. For each example $x$, each attribute $x_1, x_2, \ldots, x_n$ is renormalized with $x_{i}^{'} = \tanh\Big[\frac{x_i - \overline{x_i}}{\sigma_i} \Big]$, where $\overline{x_i}$ and $\sigma_i$ are the mean and standard deviation of the $i^{\text{\tiny th}}$ attribute respectively, calculated on the training set $S$. Normalizing the features when using the RBF kernel is a common practice and gives better results for both MinCq and SVM. Empirically, we observe that the performance gain of RbfMinCq with normalized data is even more significant than for RbfSVM.

\begin{figure}[t]
  \hspace{-12mm}\includegraphics[scale=0.45]{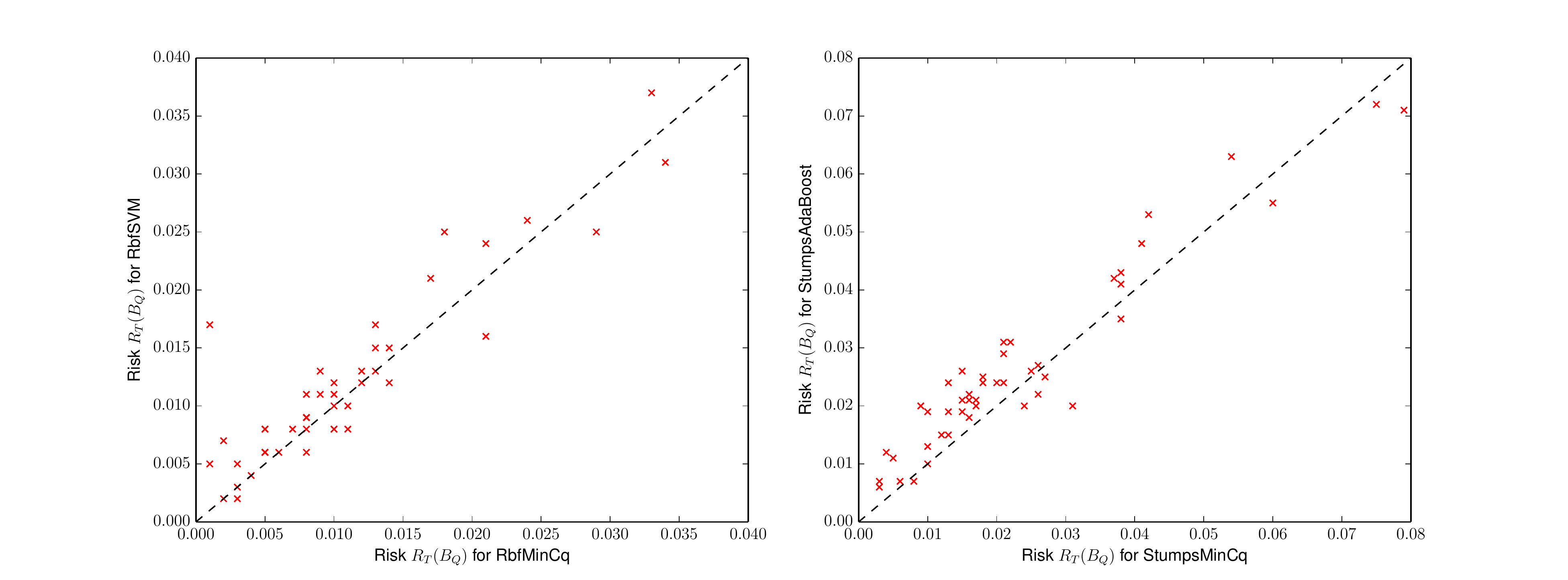}
\vspace{-4mm}
 \caption{Comparison of the risks on the testing set for each algorithm and each MNIST binary data set. The figure on the left shows a comparison of the risks of RbfMinCq ($x$-axis) and RbfSVM ($y$-axis). The figure on the right compares StumpsMinCq ($x$-axis) and StumpsAdaBoost ($y$-axis). On each scatter plot, a point represents a pair of risks for a particular MNIST binary data set. A point above the diagonal line indicates better performance for MinCq.}
 \vspace{-4mm}
 \label{fig:mincq_mnist}
\end{figure}
\begin{table}[t]
\begin{center}
\begin{scriptsize}

\rowcolors{3}{}{black!10}
\begin{tabular}{lcc}
\toprule
\multicolumn{3}{c}{Statistical Comparison Tests} \\
\cmidrule(l r){1-3}
& RbfMinCq vs RbfSVM & StumpsMinCq vs StumpsAdaBoost \\
\cmidrule(l r){2-2} \cmidrule(l r){3-3}
Poisson binomial test  & 88\% & 99\%  \\
Sign test ($p$-value) & 0.01 & 0.00  \\
\bottomrule
\hline

\end{tabular}
\end{scriptsize}
\end{center}
\vspace{-2mm}
\caption{Statistical tests comparing MinCq to either RbfSVM or StumpsAdaBoost. The Poisson binomial test gives the probability that MinCq has a better performance than another algorithm on this context. The sign test gives a $p$-value representing the probability that the null hypothesis is true (\ie, MinCq and the other algorithm both have the same performance on this context).}
\vspace{-4mm}
\label{tab:mincq_mnist} 
\end{table}

\subsubsection{Handwritten Digits Recognition Context}
The first context of interest to compare MinCq with other learning algorithms is the handwritten digits recognition. For this task, we use the \emph{MNIST database of handwritten digits} of~\cite{mnistlecun}. We split the original data set into 45 binary classification tasks, where the union of all binary data sets recovers the original data set, and the intersection of any pair of binary data sets gives the empty set. Therefore, any example from the original data set appears on one and only one binary data set, thus avoiding any correlation between the binary data sets. For each resulting binary data set, we randomly choose $500$ examples to be in the training set $S$, and the testing set $T$ consists of the remaining examples. Figure~\ref{fig:mincq_mnist} shows the resulting test risk for each binary data set and each algorithm.

Table~\ref{tab:mincq_mnist} shows two statistical tests to compare the algorithms on the handwritten digits recognition context: the Poisson binomial test~\citep{lacoste-2012} and the sign test~\citep{mendenhall1983nonparametric}. Both methods suggest that RbfMinCq outperforms RbfSVM on this context, and that StumpsMinCq outperforms StumpsAdaBoost.

\subsubsection{Classical Binary Classification Tasks Context}
This second context of interest is a more general one: it consists of multiple binary classification data sets coming from the UCI Machine Learning Repository~\citep{uci-98}. These data sets are commonly used as a benchmark for learning algorithms, and may help to answer the question ``How well may a learning algorithm perform on many unrelated classification tasks''. For each data set, half of the examples (up to a maximum of $500$) are randomly chosen to be in the training set $S$, and the remaining examples are in the testing set $T$. Table~\ref{tab:mincq_uci} shows the resulting test risks on this context, for each algorithm.

\begin{table}[t]
\begin{center}
\begin{scriptsize}

\rowcolors{5}{}{black!10}
\begin{tabular}{lcccccc}
\toprule
\multicolumn{3}{c}{Data Set Information}  & \multicolumn{4}{c}{Risk $\RTBQ$ for Each Algorithm} \\
\cmidrule(l r){1-3} \cmidrule(l r){4-7}
Name & $|S|$ & $|T|$ & RbfMinCq &  RbfSVM &  StumpsMinCq & StumpsAdaBoost \\
\cmidrule(l r){1-3} \cmidrule(l r){4-5} \cmidrule(l r){6-7}
  Australian & 345 & 345   &      0.142       & \textbf{0.133}       &    \textbf{0.165}       &           0.168           \\
     Balance & 313 & 312   &      0.054       & \textbf{0.042}       &            0.042        &   \textbf{0.032}          \\
BreastCancer & 350 & 349   & \textbf{0.037}   &         0.046        &    \textbf{0.037}       &           0.060           \\
         Car & 500 & 1228  &      0.074       & \textbf{0.032}       &            0.320        &   \textbf{0.291}          \\
         Cmc & 500 & 973   & \textbf{0.303}   &         0.306        &            0.140        &   \textbf{0.134}          \\
    Credit-A & 345 & 345   & \textbf{0.122}   &         0.133        &    \textbf{0.304}       &           0.308           \\
    Cylinder & 270 & 270   & \textbf{0.204}   &         0.233        &    \textbf{0.125}       &           0.148           \\
       Ecoli & 168 & 168   &      0.077       & \textbf{0.071}       &    \textbf{0.289}       &   \textbf{0.289}          \\
       Flags & 97  & 97    & \textbf{0.289}   &         0.320        &    \textbf{0.071}       &   \textbf{0.071}          \\
       Glass & 107 & 107   & \textbf{0.206}   & \textbf{0.206}       &    \textbf{0.268}       &           0.309           \\
       Heart & 135 & 135   &      0.163       & \textbf{0.156}       &    \textbf{0.262}       &           0.271           \\
   Hepatitis & 78  & 77    &      0.169       & \textbf{0.143}       &    \textbf{0.185}       &   \textbf{0.185}          \\
       Horse & 184 & 184   & \textbf{0.185}   &         0.196        &    \textbf{0.169}       &           0.221           \\
  Ionosphere & 176 & 175   &      0.114       & \textbf{0.069}       &            0.245        &   \textbf{0.174}          \\
   Letter:AB & 500 & 1055  &      0.007       & \textbf{0.003}       &    \textbf{0.109}       &           0.120           \\
   Letter:DO & 500 & 1058  &      0.021       & \textbf{0.018}       &    \textbf{0.005}       &           0.010           \\
   Letter:OQ & 500 & 1036  & \textbf{0.023}   &         0.036        &    \textbf{0.020}       &           0.048           \\
       Liver & 173 & 172   & \textbf{0.267}   &         0.285        &    \textbf{0.042}       &           0.052           \\
       Monks & 216 & 216   &      0.245       & \textbf{0.208}       &            0.306        &   \textbf{0.236}          \\
     Nursery & 500 & 12459 & \textbf{0.025}   &         0.026        &    \textbf{0.025}       &           0.026           \\
   Optdigits & 500 & 3323  &      0.034       & \textbf{0.027}       &    \textbf{0.089}       &   \textbf{0.089}          \\
   Pageblock & 500 & 4973  & \textbf{0.045}   &         0.048        &            0.059        &   \textbf{0.055}          \\
   Pendigits & 500 & 6994  & \textbf{0.007}   &         0.008        &    \textbf{0.069}       &           0.084           \\
        Pima & 384 & 384   & \textbf{0.253}   &         0.255        &            0.273        &   \textbf{0.250}          \\
     Segment & 500 & 1810  & \textbf{0.017}   &         0.018        &            0.040        &   \textbf{0.022}          \\
    Spambase & 500 & 4101  & \textbf{0.067}   &         0.077        &            0.133        &   \textbf{0.070}          \\
 Tic-tac-toe & 479 & 479   &      0.033       & \textbf{0.025}       &    \textbf{0.330}       &           0.353           \\
      USvote & 218 & 217   & \textbf{0.051}   & \textbf{0.051}       &    \textbf{0.051}       &   \textbf{0.051}          \\
        Wine & 89  & 89    & \textbf{0.034}   &         0.045        &            0.169        &   \textbf{0.034}          \\
       Yeast & 500 & 984   &      0.286       & \textbf{0.279}       &            0.324        &   \textbf{0.306}          \\
         Zoo & 51  & 50    & \textbf{0.040}   &         0.060        &            0.060        &   \textbf{0.040}          \\
\bottomrule
\end{tabular}

\vspace{5mm}

\rowcolors{3}{}{black!10}
\begin{tabular}{lcc}
\toprule
\multicolumn{3}{c}{Statistical Comparison Tests} \\
\cmidrule(l r){1-3}
& RbfMinCq vs RbfSVM & StumpsMinCq vs StumpsAdaBoost \\
\cmidrule(l r){2-2} \cmidrule(l r){3-3}
Poisson binomial test  & 54\% & 48\%  \\
Sign test ($p$-value) & 0.36 & 0.35  \\
\bottomrule
\hline

\end{tabular}

\end{scriptsize}
\end{center}
\vspace{-3mm}
\caption{Risk on the testing set for all algorithms, on the classical binary classification task context. See Table~\ref{tab:mincq_mnist} for an explanation of the statistical tests.}
\label{tab:mincq_uci} 
\vspace{-3mm}
\end{table}

Table~\ref{tab:mincq_uci} also shows a statistical comparison of all algorithms on the classical binary classification tasks context, using the Poisson binomial test and the sign test. On this context, both statistical tests show no significant performance difference between RbfMinCq and RbfSVM, and between StumpsMinCq and StumpsAdaBoost, implying that these pairs of algorithms perform similarly well on this general context.

\subsubsection{Amazon Reviews Sentiment Analysis} This context contains 4 sentiment analysis data sets, representing product types (\emph{books}, \emph{DVDs}, \emph{electronics} and \emph{kitchen appliances}). The task is to learn from an Amazon.com product user review in natural language, and predict the \emph{polarity} of the review, that is either negative (3 stars or less) or positive (4 or 5 stars). The data sets come from~\cite{blitzer2007biographies}, where the natural language reviews have already been converted into a set of \emph{unigrams} and \emph{bigrams} of terms, with a count. For each data set, a training set of $1000$ positive reviews and $1000$ negative reviews are provided, and the remaining reviews are available in a testing set. The original feature space of these data sets is between $90,000$ and $200,000$ dimensions. However, as most of the unigrams and bigrams are not significant and to reduce the dimensionality, we only consider unigrams and bigrams that appear at least $10$ times on the training set \citep[as in][]{chen2011cotraining}, reducing the numbers of dimensions to between $3500$ and $6000$. Again as in~\cite{chen2011cotraining}, we apply standard \emph{tf-idf} feature re-weighting~\citep{salton1988term}. Table~\ref{tab:mincq_amazon} shows the resulting test risks for each algorithm.

\begin{table}[t]
\begin{center}
\begin{scriptsize}
\rowcolors{3}{black!10}{}
\begin{tabular}{lcccc}
\toprule
\multicolumn{3}{c}{Data Set Information}  & \multicolumn{2}{c}{Risk $\RTBQ$ for Each Algorithm} \\
\cmidrule(l r){1-3} \cmidrule(l r){4-5}
Name & $|S|$ & $|T|$ & LinearMinCq &  LinearSVM \\
\cmidrule(l r){1-3} \cmidrule(l r){4-5}
      Books & 2000 & 4465 &  \textbf{0.158} & \textbf{0.158}  \\
        DVD & 2000 & 3586 &  \textbf{0.162} &         0.163   \\
    Kitchen & 2000 & 5945 &  \textbf{0.130} &         0.131   \\
Electronics & 2000 & 5681 &  \textbf{0.116} &         0.118   \\
\bottomrule
\end{tabular}

\vspace{5mm}

\rowcolors{4}{black!10}{}
\begin{tabular}{lcc}
\toprule
\multicolumn{2}{c}{Statistical Comparison Tests} \\
\cmidrule(l r){1-2}
& LinearMinCq vs LinearSVM \\
\cmidrule(l r){2-2}
Poisson binomial test  & 68\%  \\
Sign test ($p$-value) & 0.31  \\
\bottomrule
\hline

\end{tabular}

\end{scriptsize}
\end{center}
\caption{Risk on the testing set for all algorithms, on the Amazon reviews sentiment analysis context. See Table~\ref{tab:mincq_mnist} for an explanation of the statistical tests.}
\label{tab:mincq_amazon} 
\vspace{-3mm}
\end{table}

Table~\ref{tab:mincq_amazon} also shows a statistical comparison of the algorithms on this context, again using the Poisson binomial test and the sign test. LinearMinCq has an edge over LinearSVM, as it wins or draws on each data set. However, both statistical tests show no significant performance difference between LinearMinCq and LinearSVM.

\medskip

These experiments show that minimizing the $\Cbound$, and thus favoring majority votes for which the voters are maximally uncorrelated, is a sound approach. MinCq is very competitive with both AdaBoost and the SVM on the classical binary tasks context and the Amazon reviews sentiment analysis context. MinCq even shows a highly significant performance gain on the handwritten digits recognition context, implying that on certain types of tasks or data sets, minimizing the $\Cbound$ offers a state-of-the-art performance.

However, for all above experiments, we observe that the empirical values of the PAC-Bounds are trivial (close to $1$). Remember that, inspired by PAC-Bounds~\ref{bound:variancebinouille-qu} and~\ref{bound:variancebinouille-qu-sc}, the MinCq algorithm learns the weights of a majority vote by minimizing the second moment of the margin while fixing its first moment $\mu$ to some value. In these experiments, the value of $\mu$ chosen by cross-validation is always very close to $0$ (basically, $\mu = 10^{-4}$). This implies that $\CSQ = 1 - \frac{\mu^2}{\momenttwo(\MQ{S})}$ is very close to the $1 - \frac{0}{0}$ form, leading to a severe degradation of PAC-Bayesian bounds for $\CDQ$. Note that the voters were all \emph{weak} in the former experiments. This explains why very small values of $\mu$ were selected by cross-validation.

\subsubsection{Experiments with Stronger Voters}

In the following experiment, we show that one can obtain much better bound values by using \emph{stronger} voters, that is, voters with a better individual performance. To do so, instead of considering decision stumps, we consider decision trees.\footnote{A decision stump can be seen as a (weak) decision tree of depth $1$.} We use $100$ decision tree classifiers generated with the implementation of~\cite{scikit-learn} (we set the maximum depth to $10$ and the number of features per node to $1$). By using these strong voters, it is possible to achieve higher values of $\mu$.\footnote{Note that the set of decision trees was learned on a \emph{fresh} set of examples, disjoint from the training data. We do so to ensure that all computed PAC-Bounds are valid, even if they are not designed to handle \emph{sample-compressed} voters.}

\begin{figure}[h]
\centering
	\includegraphics[width=0.85\linewidth]{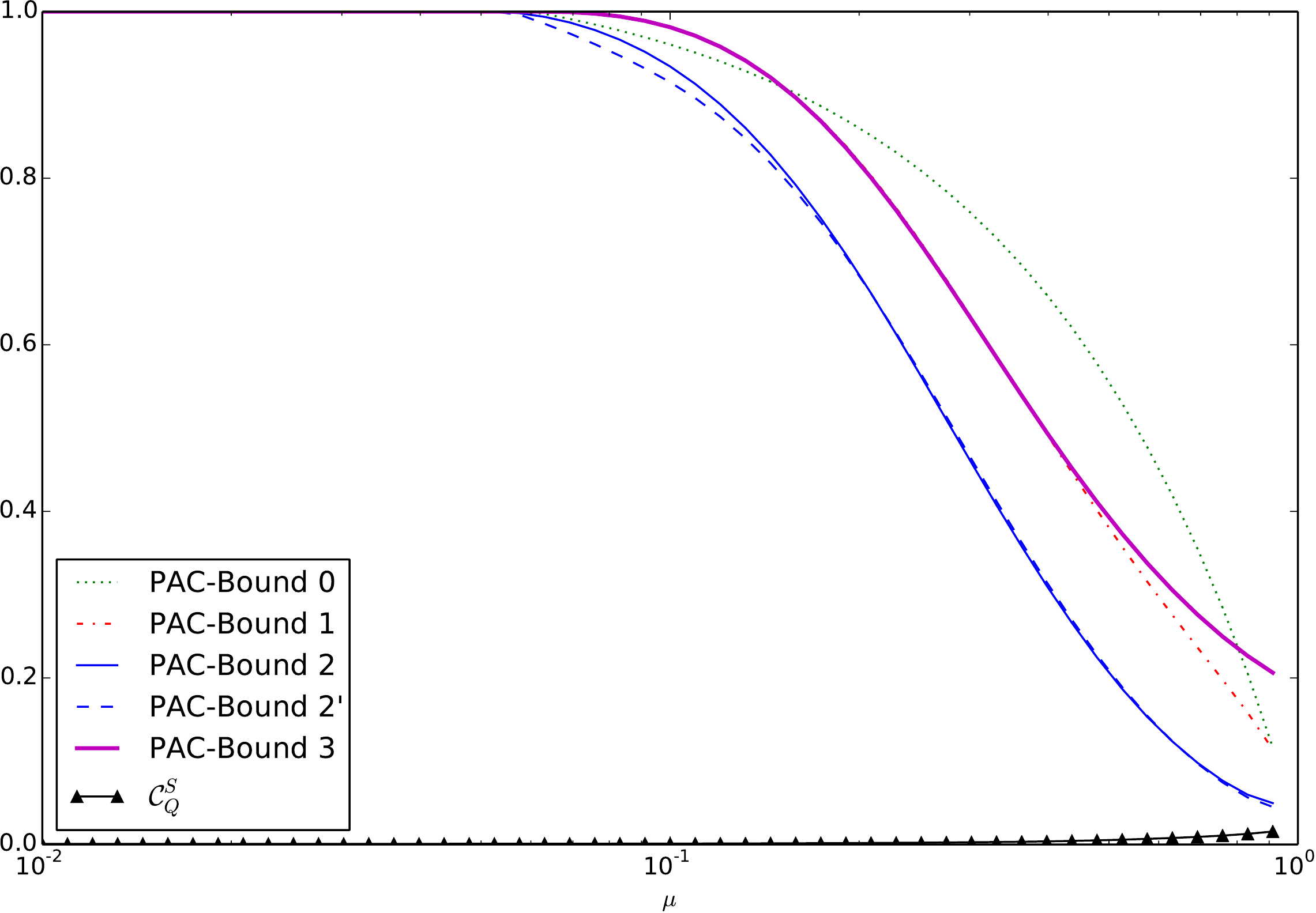}
	\caption{Values of empirical \Cbound and corresponding PAC-Bounds~\ref{bound:classic}, \ref{bound:variancebinouille}, \ref{bound:trinouille}, \ref{bound:trinouille2} and \ref{bound:variancebinouille-qu} on the majority votes output by MinCq, for multiple values of $\mu$.}
	\label{fig:boundsMinCq}
\end{figure}

Figure~\ref{fig:boundsMinCq} shows the empirical \Cbound value and its corresponding PAC-Bayesian bound values for multiple values of $\mu$ on the Mushroom UCI data set. From the $8124$ examples, $500$ have been used to construct the set of voters, $4062$ for the training set, and the remaining examples for the testing set. The figure shows the PAC-Bayesian bounds get tighter when $\mu$ is increasing. Note however that the empirical \Cbound slightly increases from $0.001$ to $0.016$. The risk on the testing set of the majority vote (not shown in the figure) is $0$ for most values of $\mu$, but also increases a bit for the highest values (remaining below $0.001$). 

Hence, we obtain tight bounds for high values of $\mu$ (PAC-Bounds~\ref{bound:trinouille} and \ref{bound:trinouille2} are under $0.2$). Nevertheless,  these PAC-Bayesian bounds are not tight enough to precisely guide the selection of $\mu$. This is why we rely on cross-validation to select a good value of $\mu$.

Finally, we also see that PAC-Bound~\ref{bound:variancebinouille-qu} is looser than other bounds over $\CDQ$, but this was expected as it was not designed to be as tight as possible. That being said, PAC-Bound~\ref{bound:variancebinouille-qu} has the same behavior than PAC-Bounds \ref{bound:variancebinouille} and \ref{bound:trinouille}. This suggests that we can rely on it to justify the MinCq learning algorithm once the hyperparameter $\mu$ is fixed.

\section{Conclusion}
\label{section:Finis-tu, never ending paper?}

In this paper, we have revisited the work presented in \cite{nips06-mv} and \cite{lmr-11}. We clarified the presentation of previous results and extended them, as well as actualizing the discussion regarding the ever growing development of PAC-Bayesian theory.
 
 We have derived a risk bound (called the \Cbound) for the weighted majority vote that depends on the first and the second moment of the associated margin distribution (Theorem~\ref{thm:C-bound}). 
 The proposed bound
 is based on the one-sided Chebyshev inequality, which, under the mild condition of Proposition~\ref{prop:CQ_optimal_suite},  is the
 tightest inequality for any real-valued random variable given only
 its first two moments. Also, as shown empirically by Figure~\ref{fig:C_Q
 est bon}, this bound has a strong predictive power on the
 risk of the majority vote.
 
 We have also shown that the original PAC-Bayesian theorem, together
 with new ones, can be used to obtain high-confidence estimates of
 this new risk bound that holds uniformly for \emph{all} posterior
 distributions. We have generalized these PAC-Bayesian results to the (more general) sample compression setting, allowing
 one to make use of voters that are constructed with elements of the training data, such as kernel functions $y_i k(\xb_i, \cdot)$. Moreover, we have presented PAC-Bayesian bounds that have the uncommon property of having no Kullback-Leibler divergence term (PAC-Bounds~\ref{bound:variancebinouille-qu} and~\ref{bound:variancebinouille-qu-sc}). These bounds, together with the \Cbound, gave the theoretical foundation to the learning algorithm introduced at the end of the paper, that we have called MinCq. The latter turns out to be expressible in the nice form of a quadratic program. MinCq is not only based on solid theoretical guarantees, it also performs very well on natural data, namely when compared with the state-of-the-art SVM.
 
 \bigskip

This work tackled the simplest problem in machine learning (the supervised binary classification in presence of i.i.d.\ data), and we now consider that the PAC-Bayesian theory is mature enough to embrace a variety of more sophisticated frameworks. Indeed, in the recent years several authors applied this theory to many more complex paradigms: Transductive Learning~\citep{Derbeko04,c-07,begin-14}, \,Domain Adaptation~\citep{germain-da-13}, Density Estimation~\citep{seldin-tishby-09,higgs-10}, Structured output Prediction~\citep{mcallester-07,giguere-13,london-14}, Co-clustering \citep{seldin-tishby-09,seldin-tishby-10},  Martingales~\citep{seldin-12}, U-Statistics of higher order~\citep{lls-13} or other non-i.i.d. settings~\citep{ralaivola-10}, Multi-armed Bandit~\citep{seldin-11} and Reinforcement Learning~\citep{fard-10,fard-11}.

\section*{Acknowledgements}
This work has been supported by National Science and Engineering Research Council (NSERC) Discovery
grants 262067 and 0122405.
Computations were performed on the Colosse supercomputer grid at Universit\'e Laval, under the auspices of Calcul Qu\'ebec and Compute Canada. The operations of Colosse are funded by the NSERC, the Canada Foundation for Innovation (CFI), NanoQu\'ebec, and the Fonds de recherche du Qu\'ebec -- Nature et technologies (FRQNT).

\medskip
\noindent
We would also like to thank Nicolas Usunier for his insightful comments and participation in the preliminary work on majority votes, and Malik Younsi for solving one of our conjectures leading to the nice $\xi(m)\!+\!m$ term in the statement of Theorem~\ref{thm:pac-trinouille-kl}. 

\medskip
\noindent
Finally, we sincerely thank the three anonymous reviewers for the exceptional quality of their work.

\appendix
\section{Auxiliary mathematical results}
\label{section:appendix_auxmath}

\begin{lemma}[Markov's inequality]
\label{lem:markov}
For any random variable $X$ such that $\Eb(X)=\mu$, and for any $a>0$, we have
\begin{equation*}
\pr\,(|X|\geq a) \ \leq \ \frac{\mu}{a}\,.
\end{equation*}
\end{lemma}

\begin{lemma}[Jensen's inequality]
\label{lem:jensen}
For any random variable $X$ and any convex function~$f$, we have 
$$
f(\esp{}[X])\ \leq\ \esp{}[f(X)]\,.
$$
\end{lemma}

\begin{lemma}[One-sided Chebyshev inequality]
\label{th:chebychev}
For any random variable $X$ such that $\Eb(X)=\mu$
\,and\, $\Var(X)=\sigma^2$, \, and for any $a>0$, we have
\begin{equation*}
   \pr\,\Big(X-\mu \geq a\Big) \ \leq\  \frac{\sigma^2}{\sigma^2 + a^2} \,.
\end{equation*}
\end{lemma}
\begin{proof}
First observe that 
$\pr\Big(X-\mu\geq a\Big) \leq 
\pr\Big(\left[X-\mu +\frac{\sigma^2}{a}\right]^2\geq \left[a+\frac{\sigma^2}{a}\right]^2\Big)$.
Let us now apply Markov's inequality (Lemma~\ref{lem:markov}) to bound this probability. We obtain
\begin{eqnarray*}
 \pr\LP\left[X-\mu +\frac{\sigma^2}{a}\right]^2\geq \left[a+\frac{\sigma^2}{a}\right]^2 \RP
 &\leq& \frac{\Eb \LB X-\mu +\frac{\sigma^2}{a} \RB^2}{\LB a +\frac{\sigma^2}{a} \RB^2}
  \hspace{20mm} \mbox{\small(Markov's inequality)}\hspace{-20mm}\\ 
 &=& \frac{\Eb \LP X-\mu \RP^2 +2 \LP\frac{\sigma^2}{a} \RP \Eb \LP X-\mu \RP  + \LP\frac{\sigma^2}{a} \RP^2}{\LB a +\frac{\sigma^2}{a} \RB^2}\\ 
 &=& \frac{\sigma^2  + \LP\frac{\sigma^2}{a} \RP^2}{\LB a +\frac{\sigma^2}{a} \RB^2}
  \ =\ \frac{\sigma^2 \LP 1 + \frac{\sigma^2}{a^2} \RP}{\LP \sigma^2 + a^2 \RP \LP 1 + \frac{\sigma^2}{a^2} \RP} 
  \ =\ \frac{\sigma^2}{\sigma^2+a^2}\, ,
\end{eqnarray*}
because $\Eb \LP X-\mu \RP^2 = \Var(X) = \sigma^2$ and $\Eb \LP X-\mu \RP = \Eb(X) - \Eb(X) = 0$. 
\end{proof}

Note that the proof Theorem~\ref{thm:KL_convex} (below) by~\cite{ct-91} considers that probability distributions $Q$ and $P$ are discrete, but their argument is straightforwardly generalizable to continuous distributions.
\begin{theorem}{\rm \citep[Theorem~2.7.2]{ct-91}} \label{thm:KL_convex}
The Kullback-Leibler divergence $\KL(Q\|P)$ is convex in the pair $(Q,P)$, \ie, if $(Q_1,P_1)$ and $(Q_2,P_2)$ are two pairs of probability distributions, then
\begin{equation*}
\KL\big(\,\lambda Q_1 + (1\!-\!\lambda) Q_2 \,\big\|\, \lambda P_1 + (1\!-\!\lambda) P_2 \,\big)
\ \leq\   
\lambda\,\KL\big(Q_1\big\|P_1\big)  + (1\!-\!\lambda)\,\KL\big(Q_2\big\|P_2\big)\,, %
\end{equation*}
for all $\lambda\in[0,1]$\,.
\end{theorem}

\begin{corollary}  Both following functions are convex:
\label{cor:bintrin}
\begin{enumerate}
\item
The function $\kl(q\|p)$ of Equation~\eqref{eq:small_kl}, \ie,  the Kullback-Leibler divergence between two Bernoulli distributions;
\item
The function $\kl(q_1,q_2\|p_1,p_2)$ of Equation~\eqref{eq:small_kl2}, \ie, the Kullback-Leibler divergence between two distributions of trivalent random variables.
\end{enumerate}
\end{corollary}
\begin{proof}
Straightforward consequence of Theorem~\ref{thm:KL_convex}.
\end{proof}

\begin{lemma}{\rm \citep{m-04}} \label{lem:maurer}
 Let $X$ be any random variable with values in $[0,1]$ and expectation $\mu = \Eb(X)$. Denote $\Xb$ the vector containing the results  of $n$ independent realizations of~$X$. Then, consider a Bernoulli random variable $X'$ ($\{0,1\}$-valued) of probability of  success $\mu$, \ie, $\pr(X'\!=\!1)=\mu$.  Denote $\Xb'\in\{0,1\}^n$ the vector containing the results of $n$ independent realizations of~$X'$.\\[1mm] 
 If function $f:[0,1]^n \rightarrow \Reals$ is convex, then
 \begin{equation*}
  \Eb \big[ f(\Xb) \big] \ \leq \ \Eb \big[ f(\Xb') \big]\,.
 \end{equation*}
\end{lemma}

The proof of Lemma~\ref{lem:general-maurer} (below) follows the key steps of the proof of Lemma~\ref{lem:maurer} by~\citet{m-04}, but we include a few more mathematical details for completeness.
Interestingly, the proof highlights that one can generalize Maurer's lemma even more, to embrace random variables of any (countable) number of possible outputs. Note that another generalization of Maurer's lemma is given in~\citet{seldin-12} to embrace the case where the random variables $X_1, \ldots, X_n$ are a martingale sequence instead of being independent.
\begin{lemma}[Generalization of Lemma~\ref{lem:maurer}]
 \label{lem:general-maurer}
 Let the tuple $(X,Y)$ be a random variable with values in $[0,1]^2$, such that $X+Y \leq 1$, and with expectation $(\mu_X,\mu_Y) = (\Eb(X), \Eb(Y))$. Given $n$ independent realizations of $(X,Y)$, denote $\Xb = (X_1, \ldots, X_n)$ the vector of corresponding $X$-values and $\Yb = (Y_1, \ldots, Y_n)$ the vector of corresponding $Y$-values. Then, consider a random variable $(X', Y')$ with three possible outcomes, $(1,0)$, $(0,1)$ and $(0,0)$, of expectations $\mu_X$, $\mu_Y$ and $1\!-\!\mu_X\!-\!\mu_Y$, respectively. Denote $\Xb', \Yb' \in \{0,1\}^n$ the vectors of $n$ independent realizations of $(X',Y')$.\\[1mm]
 If a function $f:[0,1]^n\!\times\![0,1]^n\rightarrow \Reals$ is convex, then
 \begin{equation*}
  \Eb \big[ f(\Xb, \Yb) \big] \ \leq \ \Eb \big[ f(\Xb', \Yb') \big]\,.
 \end{equation*}
\end{lemma}
\begin{proof} \renewcommand{\xb}{{\mathbf x}}
 Given two vectors $\xb=(x_1, \ldots, x_n), \yb=(y_1, \ldots, y_n)\in [0,1]^n$, let us define
 \begin{equation*}
  (\xb,\yb) \ \eqdef\  \big( \, (x_1, y_1), (x_2, y_2), \ldots, (x_n, y_n)\,\big) \ \in \ ([0,1]\!\times\![0,1])^{n}\,.
 \end{equation*}
Consider $H=\{(1,0), (0,1), (0,0)\}$. Lemma~\ref{lem:maurer-induction} (below) 
shows that any point $(\xb,\yb)$ can be written as a convex combination of the extreme points $\etab = (\eta_1, \eta_2, \ldots, \eta_n) \in H^n$:
\begin{equation} \label{eq:maurer-any-point}
    (\xb,\yb) \ = \
   \sum_{\etab \, \in H^n} \left[ 
   \left(\prod_{i:\eta_i=(1,0)} \hspace{-2mm} x_i \right) 
   \left(\prod_{i:\eta_i=(0,1)} \hspace{-2mm} y_i  \right) 
   \left(\prod_{i:\eta_i=(0,0)} \hspace{-2mm} 1\!-\!x_i\!-\!y_i \right)    
   \right] \!\cdot \etab\,.
 \end{equation}
 Convexity of function ${f}$ implies
   {\small
  \begin{equation}  \label{eq:maurer-convexity-implies}
    {f} (\xb,\yb) \ \leq \
   \sum_{\etab \, \in H^n} \left[ 
   \left(\prod_{i:\eta_i=(1,0)} \hspace{-2mm} x_i \right) 
   \left(\prod_{i:\eta_i=(0,1)} \hspace{-2mm} y_i  \right) 
   \left(\prod_{i:\eta_i=(0,0)} \hspace{-2mm} 1\!-\!x_i\!-\!y_i \right)    
   \right] \!\cdot {f}(\etab)\,,
 \end{equation}
 }with equality if $(\xb,\yb) \in  H^n =  \{ (1,0), (0,1), (0,0)\}^n$, because the elements of the sum are $0\!\cdot\!{f} (\etab)$ for all $\etab\in H^n \setminus \{(\xb,\yb)\}$ and $1\!\cdot\!{f} (\etab)$ only for $\etab=(\xb,\yb)$.\\[2mm] 
 Given that realizations of random variable $(X, Y)$ are independent and that for a given $\eta_i \in H$, only one of the three products is computed\footnote{The equality between the second and third lines follows from the fact that each expectation inside the sum of Line 2 can be rewritten as the following product of independent random variables:%
 $$\Eb \Big[ \dprod_{\eta_i} g_{\eta_i} (X_i, Y_i) \Big]
 \mbox{ \ \ with \ }
  g_{\eta_i} (X_i, Y_i) \ \eqdef\  
 \mbox{\tiny$\begin{cases} 
 X_i & \mbox{if } \eta_i=(1,0) \\
 Y_i & \mbox{if } \eta_i=(0,1) \\
 1\!-\!X_i\!-\!Y_i & \mbox{otherwise.}
 \end{cases}$}$$}, we get
 {\small
   \begin{eqnarray*}
    \Eb\left[{f} (\Xb,\Yb)\right] &\leq&    
    \Eb\left[
    \sum_{\etab \, \in H^n} \left[ 
   \left(\prod_{i:\eta_i=(1,0)} \hspace{-2mm} X_i \right) 
   \left(\prod_{i:\eta_i=(0,1)} \hspace{-2mm} Y_i  \right) 
   \left(\prod_{i:\eta_i=(0,0)} \hspace{-2mm} 1\!-\!X_i\!-\!Y_i\right)    
   \right] \!\cdot {f}(\etab) \right]\\
   &=&
    \sum_{\etab \, \in H^n} \Eb \left[ 
   \left(\prod_{i:\eta_i=(1,0)} \hspace{-2mm} X_i \right) 
   \left(\prod_{i:\eta_i=(0,1)} \hspace{-2mm} Y_i  \right) 
   \left(\prod_{i:\eta_i=(0,0)} \hspace{-2mm} 1\!-\!X_i\!-\!Y_i\right)    
   \right] \!\cdot {f}(\etab) \\
   &=&
    \sum_{\etab \, \in H^n} \left[ 
   \left(\prod_{i:\eta_i=(1,0)} \hspace{-2mm} \Eb(X_i) \right) 
   \left(\prod_{i:\eta_i=(0,1)} \hspace{-2mm} \Eb(Y_i)  \right) 
   \left(\prod_{i:\eta_i=(0,0)} \hspace{-2mm} 1\!-\!\Eb(X_i)\!-\!\Eb(Y_i)\right)    
   \right] \!\cdot {f}(\etab) \\
   &=&
   \sum_{\etab \, \in H^n} \left[ 
   \left(\prod_{i:\eta_i=(1,0)} \hspace{-2mm} \mu_X \right) 
   \left(\prod_{i:\eta_i=(0,1)} \hspace{-2mm} \mu_Y  \right) 
   \left(\prod_{i:\eta_i=(0,0)} \hspace{-2mm} 1\!-\!\mu_X\!-\!\mu_Y \right)    
   \right] \!\cdot {f}(\etab)\,.
 \end{eqnarray*}
 }This becomes an equality when $(\Xb, \Yb)$  takes values in $H^n$ (as we explain after equation~\ref{eq:maurer-convexity-implies}). We therefore conclude that 
$\Eb \big[ f(\Xb, \Yb) \big] \ \leq \ \Eb \big[ f(\Xb', \Yb') \big]\,.$
\end{proof}

\begin{lemma}[Proof of Equation~\ref{eq:maurer-any-point}]\label{lem:maurer-induction}
\renewcommand{\xb}{{\mathbf x}}
Consider $H=\{(1,0), (0,1), (0,0)\}$ and an integer \mbox{$n > 0$}. Any point $(\xb,\yb) \in \big([0,1]\!\times\![0,1]\big)^n$ can be written as a convex combination of the extreme points $\etab = (\eta_1, \eta_2, \ldots, \eta_n) \in H^n$:
 \begin{equation*}
  (\xb, \yb) \ = \  \sum_{\etab \in H^n} \rho_{\etab}(\xb,\yb)\cdot \etab\,, 
 \end{equation*}
where \\[-5mm]
 \begin{equation*}
  \rho_\etab(\xb, \yb) \ \eqdef \ \left(\prod_{i:\eta_i=(1,0)} \hspace{-2mm} x_i \right) 
   \left(\prod_{i:\eta_i=(0,1)} \hspace{-2mm} y_i  \right) 
   \left(\prod_{i:\eta_i=(0,0)} \hspace{-2mm} 1\!-\!x_i\!-\!y_i \right)\,.
 \end{equation*}
 \end{lemma}

\begin{proof}  \renewcommand{\xb}{{\mathbf x}}
 We prove the result by induction over vector size $n$.

 \noindent\emph{Proof for $n=1$:} 
 \begin{eqnarray*}
   \sum_{\etab \in H} \rho_{\etab}((x_1,y_1))\cdot \etab &=&
   x_1 \cdot ((1,0)) + y_1 \cdot ((0,1)) + (1\!-\!x_1\!-\!y_1)\cdot((0,0)) \\
   &=& ((x_1,y_1))\,.
 \end{eqnarray*}

 \noindent\emph{Proof for $n>1$:}
 We suppose that the result is true for any vector $(\xb,\yb)$ of a particular size $n$ (this is our induction hypothesis) and we prove that it implies
 \begin{samepage}
 \begin{eqnarray*}
  \sum_{(\etab, \eta_{n+1}) \, \in H^{n+1}} \Big[ 
     \rho_{(\etab, \eta_{n+1})}\big((\xb,\yb), (x_{n+1},y_{n+1})\big)
   \Big] \!\cdot (\etab, \eta_{n+1})
   &=&
\big( (\xb,\yb),\, (x_{n+1}, y_{n+1}) \big)\,,
 \end{eqnarray*}
 where $( \mathbf{a},\, b )$ denotes a vector $\mathbf{a}$,  augmented by one element $b$.\\[3mm]
 \end{samepage}
 We have
 \begin{eqnarray*}
   &\qquad& \hspace{-1.5cm}\sum_{(\etab, \eta_{n+1}) \, \in H^{n+1}} \Big[ 
     \rho_{(\etab, \eta_{n+1})}\big((\xb,\yb), (x_{n+1},y_{n+1})\big)
   \Big] \!\cdot (\etab, \eta_{n+1})\\
   &=&
   \sum_{\etab \in H^n} \rho_{\etab}(\xb,\yb) \cdot x_{n+1} \cdot \big(\etab, (1,0)\big)
  +  \sum_{\etab \in H^n} \rho_{\etab}(\xb,\yb) \cdot y_{n+1} \cdot \big(\etab, (0,1)\big) \\
  & & \hspace{3cm}+  \sum_{\etab \in H^n} \rho_{\etab}(\xb,\yb) \cdot (1\!-\!x_{n+1}\!-\!y_{n+1}) \cdot \big(\etab, (0,0)\big) \\
  &=& 
  \Bigg( \sum_{\etab \in H^n} \rho_{\etab}(\xb,\yb) \cdot (x_{n+1}\!+\!y_{n+1}\!+\!1\!-\!x_{n+1}\!-\!y_{n+1})\cdot \etab,\,
  \sum_{\etab \in H^n} \rho_{\etab}(\xb,\yb) \cdot \big(x_{n+1}, y_{n+1}\big) \Bigg) \\
   &=& 
  \Bigg( \sum_{\etab \in H^n} \rho_{\etab}(\xb,\yb)\cdot \etab,\,
  \sum_{\etab \in H^n} \rho_{\etab}(\xb,\yb) \cdot \big(x_{n+1}, y_{n+1}\big) \Bigg) \\
     &=& 
  \Big( (\xb,\yb),\, \big(x_{n+1}, y_{n+1}\big) \Big) \,.
  \end{eqnarray*}

 \noindent
 For the last equality, the $(\xb,\yb)$ term of the vector above is obtained from the induction hypothesis and the last couple is a direct consequence of the following equality:
 \begin{equation*}
   \sum_{\etab \in H^n} \rho_{\etab}(\xb,\yb) \ = \ \prod_{i=1}^{n} \Big(x_i\!+\!y_i\!+\!1\!-\!x_i\!-\!y_i\Big) \ = \ 1\,.
 \end{equation*}
\end{proof}

\begin{proposition}[Concavity of Equation~\ref{eq:fC(d,e)}]\label{prop:CQConvexe}
The function $\fC(d,e)$ is concave.
\end{proposition}
\begin{proof}
We show that the Hessian matrix of $\fC(d,e)$ is a negative semi-definite matrix. In other words, we need to prove that
\begin{equation*}
 \frac{\partial^2 \fC(d,e)}{\partial d^2} \leq 0 \, ; \quad \frac{\partial^2 \fC(d,e)}{\partial e^2} \leq 0 
 \, ; \quad \frac{\partial^2 \fC(d,e)}{\partial d^2}\frac{\partial^2 \fC(d,e)}{\partial e^2} - \left(\frac{\partial^2 \fC(d,e)}{\partial d\partial e}\right)^2 \geq 0\,.
\end{equation*}
Indeed, we have
{\small
\begin{align*}
 \frac{\partial^2 \fC(d,e)}{\partial d^2} &= \frac{2(1-4e)^2}{(2d-1)^3}
  \leq 0 \quad \forall e\in [0,1], d\in\Big[0,\frac{1}{2}\Big]\,,\\
 \frac{\partial^2 \fC(d,e)}{\partial e^2} &= \frac{8}{2d-1}
 \leq 0 \qquad\ \ \forall e\in [0,1], d\in\Big[0,\frac{1}{2}\Big]\,,\\
 \frac{\partial^2 \fC(d,e)}{\partial d^2}\frac{\partial^2 \fC(d,e)}{\partial e^2} - 
 \left(\frac{\partial^2 \fC(d,e)}{\partial d\partial e}\right)^2 &= \frac{2(1-4e)^2}{(2d-1)^3}\cdot\frac{8}{2d-1} - \left(\frac{4 - 16 e}{(1 - 2 d)^2}\right)^2 =  0\,.\\[-13mm]
\end{align*}
}
\end{proof}

\section{A General PAC-Bayesian Theorem for Tuples of Voters and Aligned Posteriors}
\label{section:gen-pac-Bayes-prod-qu}

This section presents a change of measure inequality that generalizes both Lemmas~\ref{lem:change-measure-aligned} and~\ref{lem:change-measure-paired-aligned}, and a PAC-Bayesian theorem that generalizes both Theorems~\ref{thm:gen-pac-Bayes-qu} and~\ref{thm:pac-Bayes-q2-qu}. As these generalizations require more complex notation and ideas, it is provided as an appendix and the simpler versions of the main paper have separate proofs.

\medskip
Let $\Hcal$ be a countable self-complemented set real-valued functions. In the general setting, we recall that $\Hcal$ is self-complemented if there exists a bijection $c:\Hcal\rightarrow \Hcal$ such that $c(f) =-f$ for any $f\in\Hcal$. Moreover, for a distribution $Q$ aligned on a prior distribution $P$ and for any $f\in\Hcal$, we have
\begin{eqnarray*}
Q(f)+Q(c(f)) &=& P(f)+P(c(f))\,.
\end{eqnarray*}
First, we need to define the following notation. Let $\kb$ be a sequence of length $k$, containing numbers representing indices of voters. Let $f_\kb : \Xcal \rightarrow \Yover^k$ be a function that outputs a tuple of votes, such that $f_\kb(\xb) \eqdef \langle f_{\kb_1}(\xb), \ldots, f_{\kb_k}(\xb) \rangle\,.$

Let us recall that $P^k$ and $Q^k$ are Cartesian products of probability distributions $P$ and~$Q$. Thus, the probability of drawing $f_\kb \sim Q^k$ is given by
\begin{equation*}
Q^k(f_\kb) \eqdef Q(f_{\kb_1}) \cdot Q(f_{\kb_2}) \cdot \ldots \cdot Q(f_{\kb_k}) = \prod_{i=1}^{k} Q(f_{\kb_i})\,.
\end{equation*}
Finally, for each $f_\kb$ and each $j \in \{0, \ldots, 2^k\!-\!1\}$, let 
\begin{equation*}
f_\kb^{[j]}(\xb) \eqdef \langle f_{\kb_1}^{(s_1^j)}(\xb), \ldots, f_{\kb_k}^{(s_k^j)}(\xb) \rangle\ ,
\end{equation*}
where $s_1^js_2^j...s_k^j$ is the binary representation of the number $j$, and where $f^{(0)} = f$ and $f^{(1)} = c(f)$. Note that $f_\kb^{[0]} = f_\kb$.

To prove the next PAC-Bayesian theorem, we make use of the following change of measure inequality.
\begin{theorem}\label{thm:change-measure-aligned}
\textbf{\emph{(Change of measure inequality for tuples of voters and aligned posteriors)}} 
For any self-complemented set $\Hcal$, for any distribution $P$ on $\Hcal$, for any distribution $Q$ aligned on $P$, and for any measurable function $\phi:\Hcal^k­\to \Reals$ for which $\phi(f_\kb^{[j]}) = \phi(f_\kb^{[j']})$ for any $j,j' \in \{0, \ldots, 2^k\!-\!1\}$,  we have
\begin{equation*}
\esp{f_\kb \sim Q^k} \phi(f_\kb) \ \leq \ \ln \left( \esp{f_\kb \sim P^k} e^{\phi(f_\kb)} \right) .
\end{equation*}
\end{theorem}
\pagebreak
\begin{proof}
First, note that one can change the expectation over $Q^k$ to an expectation over $P^k$, using the fact that $\phi(f_\kb^{[j]}) = \phi(f_\kb^{[j']})$ for any $j,j' \in \{0, \ldots, 2^k\!-\!1\}$ and that $Q$ is aligned on~$P$.
{ %
\begin{align}
 \nonumber &  2^k\cdot\esp{f_\kb\sim Q^k} \phi(f_\kb) \\[2mm]
 \nonumber   &= \int_{\Hcal^k} df_\kb \ Q^k(f_\kb^{[0]})\, \phi(f_\kb^{[0]})
  + \int_{\Hcal^k} df_\kb \ Q^k(f_\kb^{[1]})\, \phi(f_\kb^{[1]}) + \ldots
   + \int_{\Hcal^k}  df_\kb \ Q^k(f_\kb^{[2^k-1]})\, \phi(f_\kb^{[2^k-1]})\\[2mm]
 \nonumber   &= \int_{\Hcal^k}  df_\kb \ Q^k(f_\kb^{[0]})\, \phi(f_\kb)
 \ +\ \int_{\Hcal^k} df_\kb \ Q^k(f_\kb^{[1]})\, \phi(f_\kb)\ +\ \ldots
   \ +\ \int_{\Hcal^k}  df_\kb \ Q^k(f_\kb^{[2^k-1]})\, \phi(f_\kb)\\[2mm]
 \nonumber   &= \int_{\Hcal^k}df_\kb \ \sum_{j=0}^{2^k-1}\left(Q^k(f_\kb^{[j]})\right)\, \phi(f_\kb)\\[2mm]
 \label{magienoire1} &= \int_{\Hcal^k} df_\kb \ \sum_{j=0}^{2^k-1}\left(\prod_{i=1}^{k}\LB Q(f_{\kb_i}^{(s_i^j)})\RB\right)\, \phi(f_\kb)\\[2mm]
 \label{magienoire2} &= \int_{\Hcal^k}  df_\kb \ \prod_{i=1}^{k}\LB Q(f_{\kb_i}^{(0)}) + Q(f_{\kb_i}^{(1)})\RB\, \phi(f_\kb)\\[2mm]
 \nonumber   &= \int_{\Hcal^k} \hspace{-4mm} df_\kb \ \prod_{i=1}^{k}\LB Q(f_{\kb_i}) + Q(c(f_{\kb_i}))\RB\, \phi(f_\kb)\\[2mm]
 \nonumber   &= \int_{\Hcal^k} df_\kb \ \prod_{i=1}^{k}\LB P(f_{\kb_i}) + P(c(f_{\kb_i}))\RB\, \phi(f_\kb)\\[2mm]
 \nonumber&\ \, \vdots\\
 \nonumber   &=\ 2^k\cdot\esp{f_\kb\sim P^k}\phi(f_\kb)\,,
\end{align}
}where we obtain Line~\eqref{magienoire2} from Line~\eqref{magienoire1} by developing the terms of the product of Line~\eqref{magienoire2}.

The result is obtained by changing the expectation over $Q^k$ to an expectation over $P^k$, and then by applying Jensen's inequality (Lemma~\ref{lem:jensen}, in Appendix~\ref{section:appendix_auxmath}).
\begin{eqnarray*}
\esp{f_\kb \sim Q^k} \phi(f_\kb) 
\ = \ \esp{f_\kb \sim P^k} \phi(f_\kb) 
\ = \ \esp{f_\kb \sim P^k} \ln e^{\phi(f_\kb)} %
&\leq & \ln \left(\esp{f_\kb \sim P^k} e^{\phi(f_\kb)} \right) \,. \\[-8mm]
\end{eqnarray*}
\end{proof}
\pagebreak
\begin{theorem}\label{thm:GPB_Aligned}\textbf{\emph{(General PAC-Bayesian theorem for tuples of voters and aligned posteriors)}}
For any distribution $D$ on $\Xcal\!\times\!\Ycal$, any self-complemented set $\Hcal$ of voters \mbox{$\Xcal \rightarrow \Yover$}, any prior distribution $P$ on $\Hcal$, any integer $k\geq 1$, any convex function 
$\Dcal : [0,1]\times [0,1] \rightarrow \Reals$ and loss function $\loss : \Yover^k \times \Ycal^k \rightarrow [0,1]$ for which  $\Dcal\left(\ElossS(f_\kb^{[j]}),\, \ElossD(f_\kb^{[j]})\right) = \Dcal\left(\ElossS(f_\kb^{[j']}),\, \ElossD(f_\kb^{[j']})\right)\,,$ for any $j, j' \in \{0,\ldots,2^k\!-\!1\}$, for any $m'>0$ and any $\dt\in (0,1]$, we have
\begin{equation*}
\prob{S\sim D^m}\!\!\LP \!\!
\begin{array}{l}
  \mbox{For all posteriors $Q$ aligned on $P$}: \\[1mm]
  \scriptstyle \Dcal\left(\esp{f_\kb \sim Q^k} \!\!\!\ElossS(f_\kb), \esp{f_\kb \sim Q^k}\!\!\! \ElossD(f_\kb)\right)\ \le \
  \dfrac{1}{m'}\!\LB 
  \ln\LP\dfrac{1}{\dt}\esp{S\sim D^m}\esp{f_\kb \sim P^k}\!\!\!e^{\,m'\cdot\Dcal\left(\ElossS(f_\kb),\,\ElossD(f_\kb)\right)}\RP\RB 
\end{array}
\!\!\! \RP \! \ge 1 -\dt\,.
\end{equation*}
\end{theorem}

\bigskip
\begin{proof}
This proof follows most of the steps of Theorem~\ref{thm:gen-pac-Bayes}.~\\
We have that $\esp{f_\kb \sim P^k} e^{m'\cdot\Dcal\left(\ElossS(f_\kb),\,\ElossD(f_\kb)\right)}$ is a non-negative random variable. By Markov's inequality, we have
\begin{equation*}
 \prob{S\sim D^m}\LP \esp{f_\kb\sim P^k}e^{m'\cdot\Dcal\left(\ElossS(f_\kb),\,\ElossD(f_\kb)\right)}\,\le\, 
 \frac{1}{\dt}\esp{S\sim D^m}\esp{f_\kb\sim P^k}e^{m'\cdot\Dcal\left(\ElossS(f_\kb),\,\ElossD(f_\kb)\right)} \RP 
 \ge  1-\dt \ .
\end{equation*}

\medskip
\noindent Hence, by taking the logarithm on each side of the innermost inequality, we obtain
\begin{equation*}
 \prob{S\sim D^m}\LP \ln\LB\esp{f_\kb\sim P^k}e^{m'\cdot\Dcal\left(\ElossS(f_\kb),\,\ElossD(f_\kb)\right)}\RB\,\le\, 
 \ln\LB\frac{1}{\dt}\esp{S\sim D^m}\esp{f_\kb\sim P^k}e^{m'\cdot\Dcal\left(\ElossS(f_\kb),\,\ElossD(f_\kb)\right)}\RB \RP 
 \ge  1-\dt \ .
\end{equation*}

\medskip
\noindent Now, instead of using the change of measure inequality of Lemma~\ref{lem:change-measure}, we use the change of measure inequality of Theorem~\ref{thm:change-measure-aligned} on the left side of innermost inequality, with $\phi(f_\kb) = m'\cdot\Dcal\left(\ElossS(f_\kb),\,\ElossD(f_\kb)\right)$. We then use Jensen's inequality (Lemma~\ref{lem:jensen}, in Appendix~\ref{section:appendix_auxmath}), exploiting the convexity of $\Dcal$.

\begin{eqnarray*}
\forall \,Q \mbox{ aligned on } P :\  \ln\LB\esp{f_\kb \sim P^k}e^{m'\cdot\Dcal(\ElossS(f_\kb),\, \ElossD(f_\kb))} \RB 
 &\ge& \!\! m'\cdot\esp{f_\kb \sim Q^k}\Dcal(\ElossS(f_\kb),\ElossD(f_\kb)) \\
 &\ge& \!\! m'\!\cdot\Dcal(\esp{f_\kb \sim Q^k} \!\!\ElossS(f_\kb), \esp{f_\kb \sim Q^k}\!\! \ElossD(f_\kb))\,.
\end{eqnarray*}
 We therefore have
 {\small
\begin{equation*}
 \prob{S\sim D^m} \! \!\LP \! 
\begin{array}{l}
\mbox{For all posteriors} \ Q\ \mbox{aligned on}\ P: \\
m'\!\cdot\Dcal(\esp{f_\kb \sim Q^k}\!\!\ElossS(f_\kb),\, \esp{f_\kb\sim Q^k}\!\! \ElossD(f_\kb))
  \le
 \ln \! \LB\frac{1}{\dt}\esp{S\sim D^m}\esp{f_\kb \sim P^k}e^{m'\cdot\Dcal(\ElossS(f_\kb),\,\ElossD(f_\kb))}\RB
\end{array}
 \RP 
 \! \ge \!  1-\dt  \, .
\end{equation*}
}The result then follows from easy calculations.
\end{proof}
{ %
\bibliography{mario-bib}
}

\end{document}